\documentclass[12pt,a4paper]{article}
\usepackage[a4paper,margin=3cm]{geometry}
\usepackage{amsmath,amsfonts,amsthm,amssymb,dsfont,lmodern,mathrsfs}
\usepackage{natbib}
\usepackage[utf8]{inputenc}
\usepackage[T1]{fontenc}
\usepackage[shortlabels]{enumitem}
\usepackage[dvipsnames]{xcolor}
\usepackage[colorlinks,linkcolor=blue,citecolor=blue,urlcolor=blue]{hyperref}
\usepackage{doi}
\usepackage{tikz}
\usepackage[margin=1cm]{caption}
\usepackage{cleveref}

\usepackage{parskip}
\begingroup
\makeatletter
\@for\theoremstyle:=definition,remark,plain\do{%
  \expandafter\g@addto@macro\csname th@\theoremstyle\endcsname{%
    \addtolength\thm@preskip\parskip }%
}
\endgroup

\usepackage{titlesec}
\titleformat{\paragraph}[hang]{\bfseries}{}{0mm}{}
\titlespacing{\paragraph}{0mm}{\baselineskip}{0.5ex}

\theoremstyle{plain}
\newtheorem{theorem}{Theorem}[section]
\newtheorem{lemma}[theorem]{Lemma}
\newtheorem{proposition}[theorem]{Proposition}
\newtheorem{conjecture}[theorem]{Conjecture}
\newtheorem{definition}[theorem]{Definition}
\newtheorem{assumption}[theorem]{Assumption}
\theoremstyle{definition}
\newtheorem{example}[theorem]{Example}
\newtheorem{remark}[theorem]{Remark}

\newcommand{\rmd}{\mathrm{d}}
\newcommand{\1}{\mathds{1}}
\renewcommand{\P}{\mathbb{P}}
\newcommand{\E}{\mathbb{E}}
\newcommand{\norm}[2][]{#1\lVert #2 #1\rVert}
\newcommand{\abs}[2][]{#1\lvert #2 #1\rvert}
\newcommand{\tol}{\longrightarrow}
\newcommand{\weakcv}{\overset{w}{\longrightarrow}}
\DeclareMathOperator{\Score}{\Delta}
\DeclareMathOperator*{\argmin}{arg\,min}
\DeclareMathOperator*{\argmax}{arg\,max}
\newcommand{\pderiv}[3][{}]{\frac{\partial^{#1}#3}{\partial{#2}^{#1}}}
\newcommand{\deriv}[2][{}]{\frac{\rmd^{#1}}{\rmd{#2}^{#1}}}
\newcommand{\leb}{\mathrm{Leb}}

\renewcommand{\epsilon}{\varepsilon}
\renewcommand{\rho}{\varrho}
\renewcommand{\phi}{\varphi}
\renewcommand{\emptyset}{\varnothing}

\begin{document}

\title{Infinitesimal gradient boosting}
\author{Clément Dombry\footnote{Universit{\'e} Bourgogne Franche-Comt{\'e}, Laboratoire de Math{\'e}matiques de Besan\c{c}on, UMR CNRS 6623,   F-25000 Besan{\c c}on, France. Email: \href{mailto:clement.dombry@univ-fcomte.fr}{\texttt{clement.dombry@univ-fcomte.fr}},\ \href{mailto:jean-jil.duchamps@univ-fcomte.fr}{\texttt{jean-jil.duchamps@univ-fcomte.fr}}} \ and Jean-Jil Duchamps\footnotemark[1]}
\maketitle

\begin{abstract}
We define infinitesimal gradient boosting as a limit of the popular tree-based gradient boosting algorithm from machine learning. The limit is considered in the vanishing-learning-rate asymptotic, that is when the learning rate tends to zero and the number of gradient trees is rescaled accordingly. For this purpose, we introduce a new class of randomized regression trees bridging totally randomized trees and Extra Trees and using a softmax distribution for  binary splitting. Our main result is the convergence of the associated stochastic algorithm and the characterization of the limiting procedure as the unique solution of a nonlinear ordinary differential equation in a infinite dimensional function space. Infinitesimal gradient boosting defines a smooth path in the space of continuous functions along which the training error decreases, the residuals remain centered  and the total variation is well controlled.  
\end{abstract}

\textbf{Keywords:} gradient boosting, softmax regression tree, vanishing-learning-rate\\asymptotic, convergence of Markov processes.

\textbf{MSC 2020 subject classifications:} primary 60F17; secondary 60J20, 62G05.

\vfill
\pagebreak 

\tableofcontents

\section{Introduction and main results}

\subsection{Background}
Trying to understand how a set of covariates $X=(X^1,\ldots,X^p)$ impacts a quantity of interest $Y$ is a major task of statistical learning and machine learning that is crucial for predicting the response $Y$ when only $X$ can be observed. Among many modern numeric methods such as random forests or neural networks,  gradient boosting is a prominent one, as demonstrated by the incredible success of XGBoost~\citep{CG16} in machine learning contests. This paper proposes a first in-depth mathematical analysis of the dynamics underneath gradient boosting.

In a prediction framework, statistical learning aims at approaching the performance of the Bayes predictor which minimizes the expected prediction loss and is defined by
\begin{equation}\label{eq:optim-theo}
f^\ast=\argmin_{f}\mathbb{E}[L(Y,f(X))].
\end{equation}
Here $f:[0,1]^p\to\mathbb{R}$ is a measurable function called predictor and the function $L$, called loss function, measures the discrepancy between observation $Y$ and  prediction $f(X)$. For instance, the fundamental task of regression uses  the squared error loss $L(y,z)=(y-z)^2$ and the Bayes predictor is the regression function $f^*(x)= \mathbb{E}[Y\mid X=x]$. Because the distribution of $(X,Y)$ is unknown and accessible only through a sample of observations $(x_i,y_i)_{1\leq i\leq n}$, the optimization problem~\eqref{eq:optim-theo} is replaced by its empirical counterpart
\begin{equation}\label{eq:optim-emp}
\argmin_{f}\frac{1}{n}\sum_{i=1}^n L(y_i,f(x_i)).
\end{equation}
Here optimization is not performed over the entire function space because it would yield an overfit. In regression, interpolating functions may achieve zero loss but with poor generalization capacity to new observations. Common strategies are to restrict the problem to smaller parametric classes (e.g.\ linear regression, neural network) and/or to add penalties that impose more regularity of the solution (e.g.\ smoothing splines, generalized additive models). The choices of a suitable parametric class and/or of suitable penalty terms for regularization are critical and related to the generalization capacity of the predictor in relation with under/over-fitting.  When using a parametric model, algorithms from numerical analysis, such as gradient descent and their variants, can be used efficiently to solve the associated finite-dimensional optimization problem. For more general background on statistical learning, the reader should refer to \cite{ESL}.

 With a different strategy, the gradient boosting method, as proposed by \cite{F01}, is an original approach to tackle the optimization problem \eqref{eq:optim-emp} directly in the infinite dimensional function space. Akin to gradient descent, it is a recursive procedure that tries to improve the current predictor by performing small steps in suitable directions. A ``suitable direction'' is obtained by fitting a predictor, called base learner, to the residuals of the current model (i.e.\ the negative gradient of the empirical loss at the observation points). In practice, the most successful base learners are regression trees --- a precise description of regression trees and  gradient boosting  are provided in the next subsection. Efficient implementation in the R package GBM or Python library XGBoost makes tree-based gradient boosting one of the most useful techniques from modern machine learning. 

The development of boosting started with the algorithm AdaBoost for classification by \cite{FS99}, where the idea to combine many weak learners trained sequentially in order to improve the prediction proved successful. \cite{FHT00} were able to see the wider statistical framework of stagewise additive modeling which lead to gradient boosting \citep{F01} and its stochastic version \citep{F02}. In the mathematical analysis of boosting, the main issue discussed in the statistical literature is consistency, i.e.\ the ability of the procedure to achieve the optimal Bayes error rate when the sample size tends to infinity. Such consistency results where proved for AdaBoost \citep{J04}, more general boosting procedures \citep{LV04, BLV04, ZY05} or an infinite-population version of boosting \citep{B04}. These papers mostly use the fact that the Bayes predictor can be approximated by linear combinations taken over a class of simple functions (such as trees) together  with some measure of the complexity of the class (such as VC-dimension). 

The precise dynamics of the boosting procedure are seldom considered, an exception being the study of linear  L2-boosting by \cite{BY03}, where the authors consider a linear base learner and provide an explicit expression of the associated  boosting procedure relying on linear algebra.  Then, the precise knowledge of the behavior of the base learner eigenvalues in the large sample limit allows them to derive consistency. In this linear framework, \cite{DE20} recently introduced the vanishing-learning-rate asymptotic for linear L2-boosting. They proved that, as the learning rate converges to zero and the number of iterations is rescaled accordingly, the boosting procedure converges to the solution of a linear differential equation in a function space. The motivation comes from the fact that small learning rates are known to provide better performances and are commonly used in practice \citep{R07}. This vanishing-learning-rate asymptotic sheds some new light on linear L2-boosting, putting the emphasis on the dynamics of the procedure with finite sample size, rather than on its consistency as the sample size tends to infinity. 

In the present paper, we extend the vanishing-learning-rate asymptotic beyond linear boosting and obtain the existence of a limit for gradient boosting with a general convex loss function and non linear base learner given by regression trees.  Dealing with trees implies new technical issues because the state space of the boosting sequence is now truly infinite dimensional, while the proofs in \cite{DE20} rely on assumptions ensuring that the boosting sequence remains in a finite dimensional space. To tackle this issue, we work in a suitable space that we call the \emph{space of tree functions}. The construction relies on a new encoding of regression trees  by  discrete signed measures, see Section~\ref{sec:T-space}. Another issue with regression trees is their non-linearity and discontinuity with respect to the training sample. Both are due to the splitting procedure that makes use of the response variable in a greedy way, where the $\argmax$ functional underneath the best split selection is not continuous. This has lead us to design a new class of randomized regression trees,  that we call \emph{softmax regression trees}, where  the classical $\argmax$ selection is replaced by a softmax selection. An important feature of this model is that the expected tree is Lipschitz continuous with respect to the training sample, which makes it possible to develop the differential equation approach from \cite{DE20}. Now the dynamics are driven by a nonlinear differential equation and the Lipschitz property ensures the existence and uniqueness of solutions. 

We shortly mention the limitation and perspectives of the present work. Strong regularity properties of the expected base learner are assumed that are tailored for the  two main important tasks of statistical learning that are least squares regression and binary classification. Further statistical tasks such as quantile regression of robust regression are not covered by the theory we develop here and should be the subject of further research. We establish here a probabilistic theory for the vanishing-learning-rate asymptotic of gradient boosting based on a finite and fixed sample. The variability with respect to the sample distribution   and the large sample behavior should be considered in a future work putting the emphasis on statistical issues. Partial result for linear boosting are provided in \cite{DE20} with a bias/variance decomposition of the training and test errors.

The remainder of this section is devoted to a description of our main results with limited technical details. We present our framework for tree-based gradient boosting and our main results concerning the existence of the vanishing-learning-rate asymptotic (Theorem~\ref{thm:cv-lambda-to-0}) and the characterization of its dynamics in terms of a differential equation (Theorem~\ref{thm:EDO}). The technical material required to state and prove our results is developed in the next three sections. Section~\ref{sec:formal-softmax-trees} focuses on the base learner and develops the theory of softmax regression trees. Section~\ref{sec:T-space} is devoted to the construction and study of a new function space tailored to the analysis of tree-based gradient boosting. Section~\ref{sec:inf-gb} proposes a detailed analysis of the vanishing-learning-rate asymptotic for gradient boosting and the associated dynamics. Finally, all the technical proofs are postponed and gathered in Section~\ref{sec:proofs}.

\subsection{Regression trees} \label{sec:intro-reg-trees}

Gradient boosting is usually implemented with regression trees as base learners. Contrary to random forest where fully grown trees are used, gradient boosting usually makes use of shallow trees, typically with depth between $1$ and $5$. We describe below the classical Breiman regression tree and Extra-Tree models from machine learning and also introduce a new model that we call softmax regression tree and that is crucial in our theory. 

\paragraph{Breiman regression trees}
Breiman regression tree \citep{BFOS84} are built with a recursive top-down procedure that uses greedy binary splitting to partition the feature space  into hypercubes called leaves. Starting from the initial feature space $[0,1]^p$, recursive binary splitting produces first $2$ regions, then $4$ regions, and recursively $2^d$ regions that form a partition on $[0,1]^p$. The parameter $d\geq 1$ is called the tree depth.

Greedy binary splitting means that each split is determined so as to minimize the empirical mean squared error (mse). For a region $A\subset [0,1]^p$, we define
\begin{equation}\label{eq:mse}
\mathrm{mse}(A)=\frac 1n
\sum_{i=1}^n(y_i-\bar y(A))^2\1_A(x_i),
\end{equation}
with $\bar y(A)$ the mean of  $\{y_i: x_i\in A\}$. A split consists in dividing the hypercube $A$ into two hypercubes $A_0$ and $A_1$ according to whether some variable $X^j$  is below or above threshold $u$. More precisely, the split encoded by 
$(j,u)\in[\![1,p]\!]\times[0,1]$  yields the partition $A=A_0\cup A_1$ defined by
\begin{align}
  A_0=\{x\in A\ :\ x^j < a^j+u(b^j-a^j)\},\nonumber \\
  A_1=\{x\in A\ :\ x^j\geq a^j+u(b^j-a^j)\},\label{eq:A0-A1}
\end{align}
where $x^{j}$ denotes the $j$th coordinate of $x$, $a^j = \inf_{x\in A}x^j$ and $b^j = \sup_{x\in A}x^j$. The mse decrease -- or simply score -- associated to this split is
\begin{align} 
  \Score(j,u;A) &= \mathrm{mse}(A)-\mathrm{mse}(A_0)-\mathrm{mse}(A_1) \nonumber \\
  &= \frac{n(A_0)}{n}\big(\bar y(A_0)-\bar y(A)\big)^2+\frac{n(A_1)}{n}\big(\bar y(A_1)-\bar y(A)\big)^2,\label{eq:def-score}
\end{align}
with $n(A)$ the number of observations in $A$.  The greedy binary split is encoded by
\[
(j^*,u^*)=\argmax_{(j,u)}\Score(j,u;A),
\]
meaning that Breiman's original algorithm maximizes the mse decrease  over all admissible splits.

\paragraph{Extra-Trees}
For large sample  and/or high dimensional covariate space, the search for the best split can  be computationally expensive and the computational burden can be alleviated by the use of Extra-Trees \citep{GEW06}. In this algorithm the splits are randomized and optimization is not performed over all admissible splits, but only among $K$ random proposals.

At each split, $K$ random proposal $(j_1,u_1),\ldots,(j_K,u_K)$ are drawn uniformly on $[\![1,p]\!]\times[0,1]$ and the split effectively performed is encoded by
\[
(j^*,u^*)=\argmax_{(j,u)\in(j_k,u_k)_{1\leq k\leq K}}\Score(j,u;A).
\]
Extra-Trees are randomized regression tree with two main parameters: the depth $d\geq 1$ controls the numbers of leaves and  the number of proposals at each split $K\geq 1$ controls the degree of randomness.  When $K=1$, the split involves only one random draw and no maximization so that the resulting partition do not depend on the input sample -- the model reduces to the so-called \emph{totally randomized tree}. On the opposite, as $K\to\infty$, Extra-Trees approaches  Breiman's  regression tree.

\paragraph{Softmax regression trees}
Our analysis requires a regularity property of the (expected) tree with respect to its input $(y_i)_{1\leq i\leq n}$ that is not satisfied by the two aforementioned models. This lack of regularity is due to the discontinuity of the $\mathrm{argmax}$ operator. It is indeed well-known that regression trees tend to be unstable since  small modifications in the input may modify the first split and thereby the whole tree structure.  

Our  main idea for regularization is simply to modify  the Extra-Trees construction by replacing the $\mathrm{argmax}$ selection by a smoother $\mathrm{softmax}$ selection. The $\mathrm{softmax}$ function with parameter $\beta\geq 0$ is defined by
\begin{equation}\label{eq:softmax}
\mathrm{softmax}_\beta(z)=\left( \frac{e^{\beta z_k}}{\sum_{l=1}^Ke^{\beta z_l}}\right)_{1\leq k\leq K},\quad z\in\mathbb{R}^K.
\end{equation}
The output $\mathrm{softmax}_\beta(z)$ is interpreted as a probability distribution.  When $\beta=0$, this probability is uniform on $[\![1,K]\!]$. When $\beta\to+\infty$, it concentrates on the subset $\argmax_{k} z_k$.

The \textit{softmax regression trees}  relies on softmax binary splitting defined as follows. As for Extra-Trees, the splits are randomized and involve $K\geq 1$ random proposals $(j_k,u_k)_{1\leq k\leq K}$ uniformly drawn on $[\![1,K]\!]\times[0,1]$. Then the split effectively performed is  randomly  selected with distribution 
\[
(j^*,u^*)\sim \mathrm{softmax}_\beta((\Score(j_k,u_k;A))_{1\leq k\leq K}).
\]
The parameter $\beta\geq 0$ is a further hyperparameter of the model which provides a bridge between totally randomized trees (as $\beta\to 0$) and Extra-Trees (as $\beta\to+\infty$). Furthermore we retrieve  Breiman's regression tree in the limit $K\to+\infty$ and $\beta\to+\infty$. See Remarks~\ref{rk:beta-to-infty} and~\ref{rk:K-to-infty} for a formal justification. A  property crucial to our theory is that the mean softmax regression tree is Lipschitz continuous in its input $(y_i)_{1\leq i\leq n}$, see Proposition~\ref{prop:reg-tree-Lipschitz}.

For future reference and to fix notation, we introduce the following definition. 
\begin{definition}\label{def:softmax-reg-tree} We call softmax regression tree with parameter $(\beta,K,d)$ the  tree function
\[
T(x;(x_i,y_i)_{1\leq i\leq n},\zeta)=\sum_{1\leq k\leq 2^d} \bar y(A_k) \1_{A_k}(x), \quad x\in[0,1]^p,
\]
where $(x_i,y_i)_{1\leq i\leq n}$ denotes the input sample, $\zeta$ the auxiliary randomness used to perform the splits and $(A_k)_{1\leq k\leq 2^d}$ the resulting partition of $[0,1]^p$.
\end{definition}
We describe formally in Section~\ref{sec:formal-softmax-trees} below what is the structure of the auxiliary randomness $\zeta$ and how it determines the partition $(A_k)_{1\leq k\leq 2^d}$.

\paragraph{Illustration}
We illustrate the softmax regression tree model in the simplest case of a $1$-dimensional regression model  with regression trees of depth $d=1$, also called stomps. Let us consider a sample $(x_i,y_i)_{1\leq i\leq n}$  of size $n=100$ from the  regression model $(X,Y)\in[0,1]\times \mathbb{R}$  where
\begin{equation}\label{eq:1d-regression}
X\sim \mathrm{Unif}([0,1]),\quad  Y=\sin\Big(\frac{\pi}{4} + \frac{3\pi}{2}X\Big)+\epsilon
\end{equation}
and the error $\varepsilon$ is a centered Gaussian with standard deviation $\sigma=0.1$ and independent of $X$. Figure~\ref{fig:argmax-softmax} shows the scatter plot associated with the sample  and the argmax / softmax problem associated with the three models (Breiman regression tree, Extra-Tree, softmax regression tree). The tree functions (estimation of the regression function) are displayed in Figure~\ref{fig:example-regression-trees} where one can see that $\beta$ plays the role of a regularization parameter and that the softmax regression trees bridges Breiman regression trees and totally random trees.

\begin{figure}[ht!]
\includegraphics[width=\linewidth]{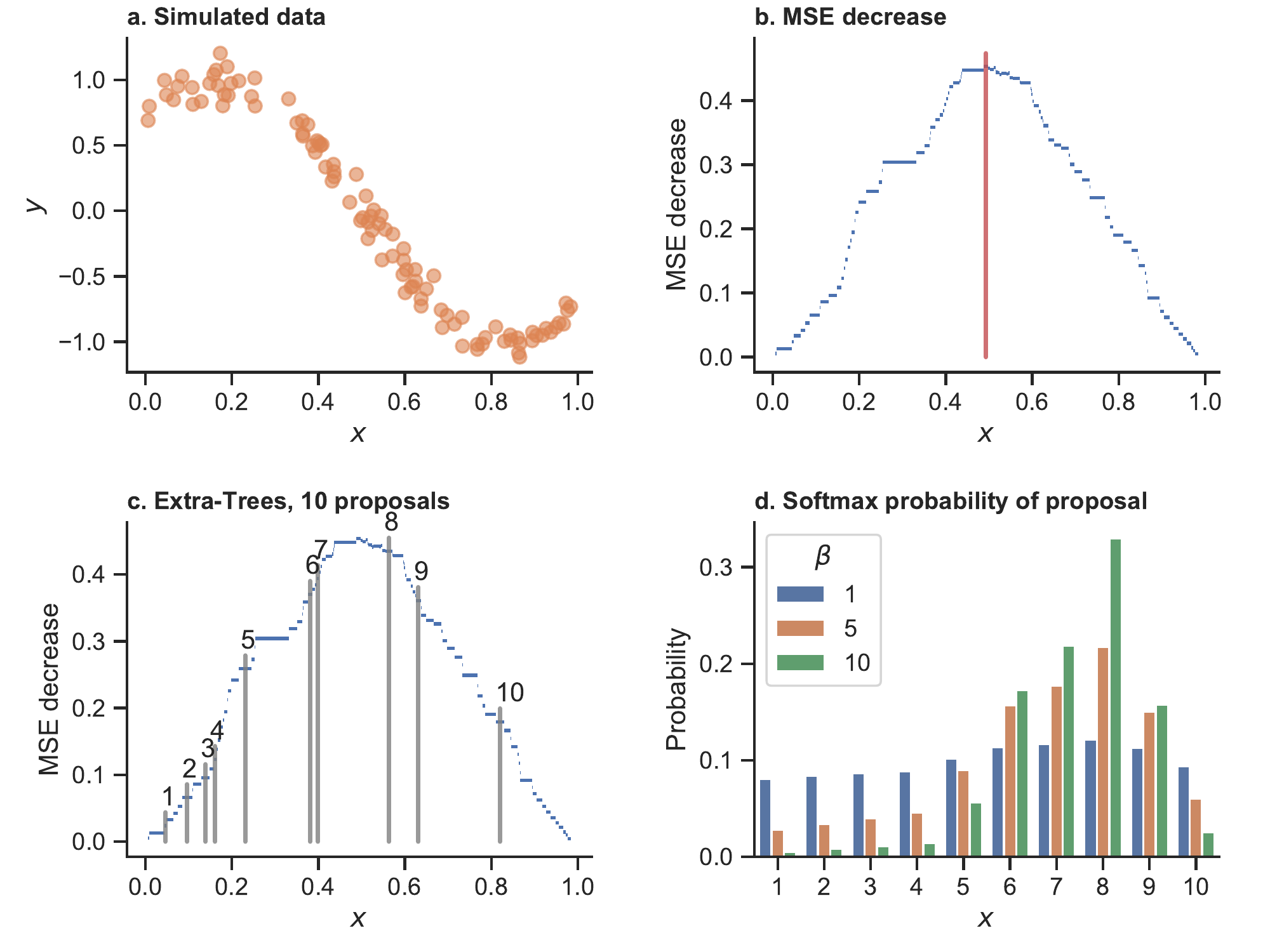}
\caption{A simple $1$-dimensional regression problem and the  maximisation problem associated with depth $1$ trees. Panel a: scatter plot for a sample of size $n=100$ from the regression model \eqref{eq:1d-regression}. Panel b: for Breiman regression tree, the MSE decrease \eqref{eq:def-score} is maximised over all possible thresholds; the vertical red line represents the argmax. Panel c: for Extra-Trees,  the maximisation is restricted over $K=10$ random proposals represented by the vertical gray lines. Panel d: for softmax regression trees, the threshold is chosen randomly among the $K=10$ proposals with the softmax distribution $\mathrm{softmax}_\beta(\mathrm{scores})$  represented by the barplot, for $\beta\in\{1,5,10\}$ here. } \label{fig:argmax-softmax}
\end{figure}

\begin{figure}[ht!]
\includegraphics[width=\linewidth]{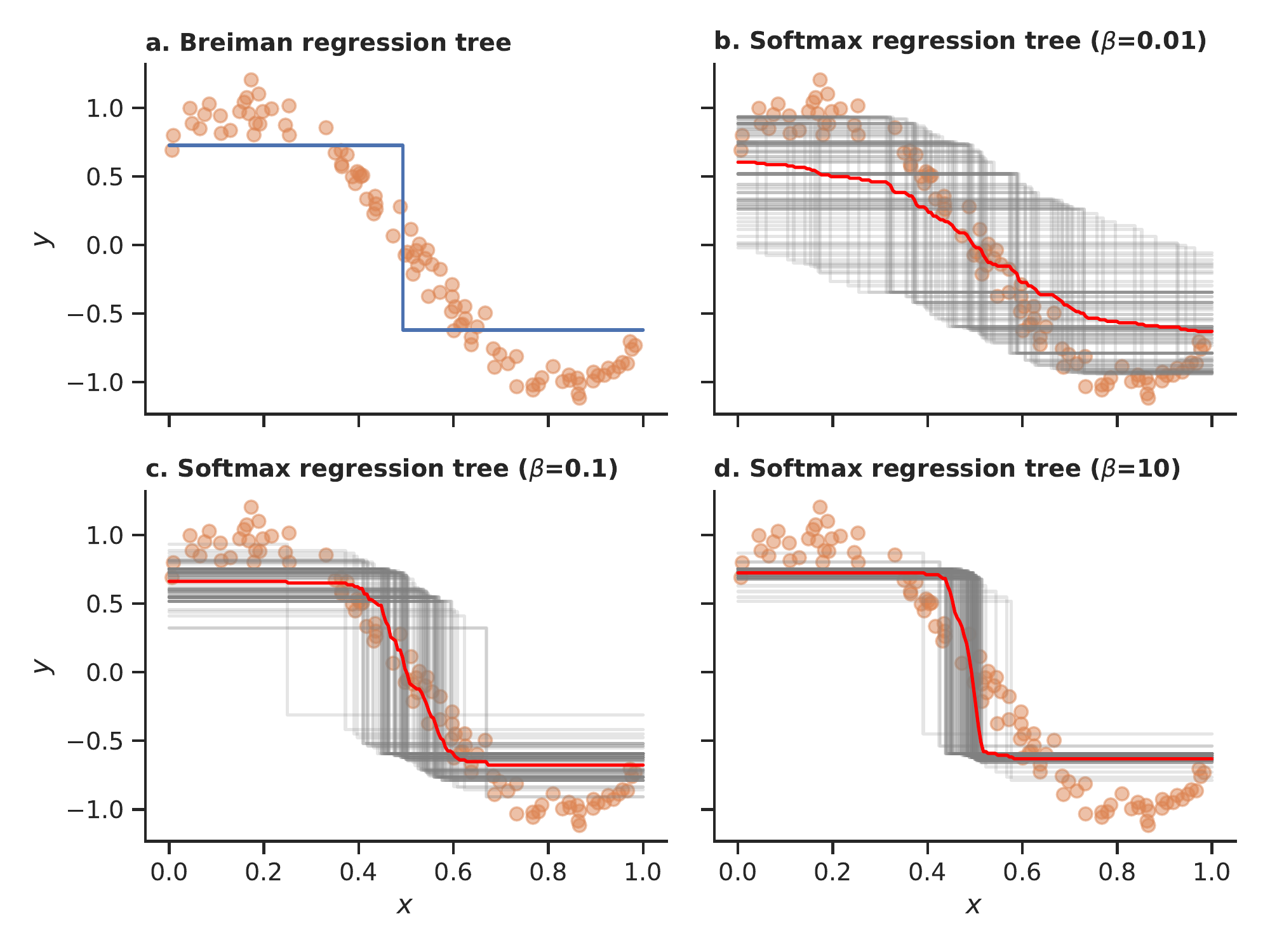}
\caption{Regression tree functions based on the same data as in Figure~\ref{fig:argmax-softmax}. Panel a: Breiman regression tree with depth $1$. Panels b,c,d: 100 random realizations of the softmax regression tree (in gray) and their mean (in red) for the softmax regression tree with $K=20$ and $\beta=0.01$ in panel b (resp. $\beta=0.1$ and $\beta=10$ in panels c and d).}
\label{fig:example-regression-trees}
\end{figure}

\subsection{Tree-based gradient boosting} \label{sec:intro-gradient boosting}

We focus in this paper on gradient boosting as introduced by~\cite{F01}. We also refer to \cite{R07} for practical guidelines on gradient boosting. Recall that the task \eqref{eq:optim-emp} is to minimize the empirical loss. We assume here  that the loss function $L$ is convex and twice differentiable in its second variable. The boosting procedure is an ensemble method that combines many instances of the base learner that are fitted sequentially so as to gradually improve the current model. We use here our  softmax regression tree with parameter $(\beta,K,d)$ as  base learner.  Gradient boosting with learning rate $\lambda>0$ then produces the sequence of predictors $(\hat F_m^\lambda)_{m\geq 0}$ defined as follows:
\begin{enumerate}
\item (Initialization). Set $\hat F_0^\lambda$ equal to the constant predictor: 
\begin{equation}\label{eq:boosting-init}
\hat F_0^\lambda(x)\equiv \argmin_{z\in\mathbb{R}}\frac{1}{n}\sum_{i=1}^n L(y_i,z),\quad x\in[0,1]^p.
\end{equation}
\item (Recursion). At step $m\geq 0$, 
\begin{enumerate}[label={(\roman*)}]
\item Compute the residuals (negative loss gradients):
\begin{equation}\label{eq:residuals}
r_{m,i}^\lambda=-\frac{\partial L}{\partial z}(y_i,\hat F_m^\lambda(x_i)),\quad 1\leq i\leq n;
\end{equation}
\item Fit a softmax regression tree with parameter $(\beta,K,d)$ to the residuals:
\begin{equation}\label{eq:boosting-tree}
T_{m+1}(x;(x_i,r_{m,i}^\lambda)_{1\leq i\leq n},\zeta_{m+1})=\sum_{1\leq k\leq 2^d} \bar r_m^\lambda(A_k) \1_{A_k}(x), 
\end{equation}
where $\zeta_{m+1}$ denotes the auxiliary randomness (independent of the past), $(A_k)_{1\leq k\leq 2^d}$ the resulting partition of $[0,1]^p$ into $2^d$ leaves and $\bar r_m^\lambda(A_k)$ the mean residual in leaf $A_k$;
\item Modify the leaf values according to a line search one-step approximation:
\begin{equation}\label{eq:boosting-modified-tree}
\widetilde{T}_{m+1}(x)=\sum_{1\leq k\leq 2^d} \tilde r_m^\lambda(A_k) \1_{A_k}(x), 
\end{equation}
with
\begin{equation}\label{eq:boosting-line-search-approximation}
\tilde r_m^\lambda(A_k) = -\frac{\sum_{i=1}^n \frac{\partial L}{\partial z}(y_i,\hat F_m^\lambda(x_i))\1_{A_k}(x_i) }{\sum_{i=1}^n \frac{\partial^2 L}{\partial z^2}(y_i,\hat F_m^\lambda(x_i))\1_{A_k}(x_i)}.
\end{equation}
(see Remark~\ref{rk:line-search} below for a justification of this expression).
\item Update the model by adding a shrunken version of the tree:
\begin{equation}\label{eq:boosting-update}
\hat F_{m+1}^\lambda(x)=\hat F_{m}^\lambda (x)+\lambda \widetilde{T}_{m+1}(x).
\end{equation}
\end{enumerate}
\end{enumerate}

We emphasize that at each step, the tree $\widetilde{T}_{m+1}$ is obtained by fitting a softmax regression tree to the residuals in order to compute the partition  (step \textit{iii}) and then modifying the tree values  according to the line search one-step approximation (step~\textit{iv}). This motivates the following definition of \textit{softmax gradient trees}.

\begin{definition}\label{def:softmax-gradient-tree}
Let $(x_i)_{1\leq i\leq n}$ be fixed. For  a bounded function $F:[0,1]^p\to\mathbb{R}$, the residuals at $F$ are given by
\[
r_i=-\frac{\partial L}{\partial z}(y_i, F(x_i)),\quad 1\leq i\leq n.
\]
The softmax gradient tree with parameter $(\beta,K,d)$ is defined as the  tree function
\[
\widetilde{T}(x;F,\zeta)=\sum_{1\leq k\leq 2^d} \tilde r(A_k) \1_{A_k}(x), \quad x\in[0,1]^p,
\]
where $\zeta$ denotes the auxiliary randomness, $(A_k)_{1\leq k\leq 2^d}$ the partition associated with the softmax regression tree $T(\,\cdot\,;(x_i,r_i)_{1\leq i\leq n},\zeta)$ and
\begin{equation}\label{eq:gradient-tree-line-search-approximation}
\tilde r(A_k) = -\frac{\sum_{i=1}^n \frac{\partial L}{\partial z}(y_i, F(x_i))\1_{A_k}(x_i) }{\sum_{i=1}^n \frac{\partial^2 L}{\partial z^2}(y_i,F(x_i))\1_{A_k}(x_i)},\quad 1\leq k\leq 2^d,
\end{equation}
the leaf values.
\end{definition}
With this definition, Equations~\eqref{eq:residuals}-\eqref{eq:boosting-update} take the simple form
\begin{equation}\label{eq:Markov}
\hat F_{m+1}^\lambda=\hat F_m^\lambda+\lambda \widetilde{T}(\,\cdot\,;\hat F_m^\lambda,\zeta_{m+1}),\quad m\geq 0,
\end{equation}
clearly evidencing the Markov structure of the  sequence $(\hat F_m^\lambda)_{m\geq 0}$.

\begin{remark}\label{rk:line-search} In the original definition of gradient boosting by \cite{F01}, stagewise  additive modeling is considered in a greedy way and the modification of the leaf values  in step \textit{(iii)} is given by
\begin{equation}\label{eq:boosting-line-search}
\tilde r(A_k)=\argmin_{z\in\mathbb{R}}\sum_{i=1}^n L(r_i, F(x_i)+z)\1_{A_k}(x_i)
\end{equation}
This corresponds to a line search for optimally updating the current model in an additive way on leaf $A_k$. Quite often, the line search problem \eqref{eq:boosting-line-search}  has no explicit solution and numerical optimization has to be used. To alleviate the computational burden, it is usually replaced by its one-step approximation 
\[
\tilde r(A_k)=\argmin_{z\in\mathbb{R}}\sum_{i=1}^n\Big( L(r_{i},F(x_i))+ \frac{\partial L}{\partial z}L(r_i, F(x_i))z+\frac{1}{2}\frac{\partial^2 L}{\partial z^2} L(r_i,F(x_i))z^2\Big)\1_{A_k}(x_i),
\]
where the function to optimize is replaced by its second order Taylor approximation. Solving for this quadratic problem we retrieve exactly Equation~\eqref{eq:gradient-tree-line-search-approximation}. Because of its constant use in modern implementation of gradient boosting such as XGBoost \citep{CG16}, we directly use this one-step approximation in our definition of gradient boosting and softmax gradient trees.
\end{remark}

\begin{example}\label{ex:regression} (regression with square loss). In the case of square loss $L(y,z)=\frac{1}{2}(y-z)^2$, the residuals are $r_{i}=y_i- F(x_i)$. We recover the usual notion of residual, that is the difference between observation and  predicted value. In this important case, the line search and its one-step approximation  are both equal to the mean residual.  The gradient tree is the same as the regression tree, i.e.\ $T_{m+1}=\widetilde{T}_{m+1}$ in Equations~\eqref{eq:boosting-tree}-\eqref{eq:boosting-modified-tree}. Gradient boosting for $L^2$-regression thus consists in sequentially updating the model in an additive way with a shrunken version of the regression tree fitted to the current residuals.
\end{example}

\begin{example}\label{ex:classification} (binary classification with cross-entropy).
Binary classification aims at predicting a binary response variable $Y\in\{0,1\}$ given $X=x$, the values $0$ and $1$ being often interpreted as failure and success respectively. The goal is to predict the success probability $p^\ast(x)=\mathbb{P}(Y=1\mid X=x)$. The binary cross entropy corresponds to the negative log-likelihood $L(y,z)=-yz+\log(1+e^z)$. The Bayes predictor is then equal to $f^\ast(x)=\log(p^\ast(x)/(1-p^\ast(x))$, which is the logit of the success probability. The loss derivatives are given by
\[
\frac{\partial L}{\partial z}(y,z)=y-p\quad \mbox{and}\quad \frac{\partial^2 L}{\partial z^2}(y,z)=p(1-p)\quad \mbox{with } p=\frac{e^z}{1+e^z}.
\]
\end{example}

\begin{example}\label{ex:AdaBoost} (binary classification with exponential loss).
The algorithm AdaBoost by \cite{FS99} was at the origin of the success of boosting. \cite{F01} showed that it is related to gradient boosting when the loss is exponential. Here $Y\in\{-1,1\}$ and the loss function is $L(y,z)=\exp(-yz)$. The Bayes predictor is $f^\ast(x)=\frac{1}{2}\log(p^\ast(x)/(1-p^\ast(x))$ with $p^\ast(x)=\mathbb{P}(Y=1\mid X=x)$. The loss derivatives are simply
\[
\frac{\partial L}{\partial z}(y,z)=ye^{-yz} \quad \mbox{and}\quad \frac{\partial^2 L}{\partial z^2}(y,z)=e^{-yz}.
\]
\end{example}

\begin{remark} Gradient boosting is a versatile procedure that can handle many different statistical tasks such as quantile regression or robust regression thanks to suitable choices of the loss function. For instance, the least absolute deviation $L(y,z)=|y-z|$ leads to the median regression and the corresponding Bayes predictor is the conditional median of $Y$ given $X=x$. The theory we develop in this paper uses a strong regularity assumption on the loss function and does not cover this example. The same remark applies for the $\tau$th quantile loss used in quantile regression or for the Huber loss used in robust regression. Such examples will be considered in further research. The regularity  assumption we assume in the present paper are tailored for least squares regression and binary classification.
\end{remark}

\paragraph{Illustration}
Using the  simple regression model \eqref{eq:1d-regression} that was used for the illustration of regression trees, we next consider the sequential aggregation of trees provided by gradient boosting. While a single tree provides a very crude estimate of the regression function (see Figure~\ref{fig:example-regression-trees}), the aggregation of many trees can approximate the regression function fairly well. We provide in Figure~\ref{fig:example-gradient-boosting} the output of gradient boosting based on softmax regression trees with depth $d=1$ and with a learning rate $\lambda=0.01$. For $\beta=0$ (in blue), $\beta=0.01$ (in green) or $\beta=1$ (in red), the output after $1000$ iterations provide a reasonable fit of the regression function. We observe that a larger value of $\beta$ implies a faster decrease of the MSE.

\begin{figure}[h!]
\includegraphics[width=\linewidth]{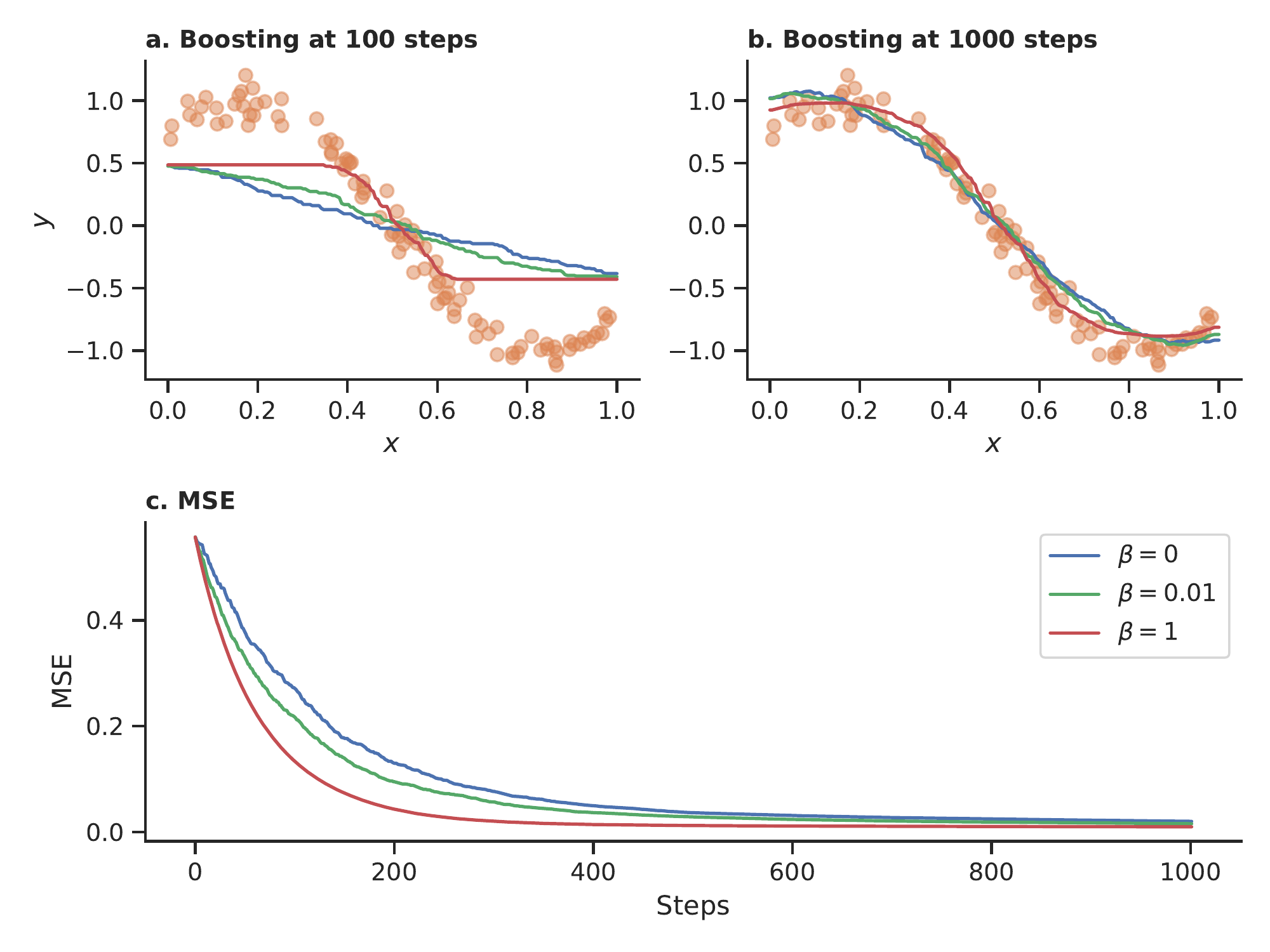}
\caption{Output of the gradient boosting algorithm  based on the same data  as in Figure~\ref{fig:argmax-softmax}. The learning rate is  $\lambda=0.01$, the parameter $\beta$ is either $0$ (blue curves), $0.01$ (green) or $1$ (red), and the number of proposals is $K=20$.  Panel a: gradient boosting output after $n=100$ steps.  Panel b: gradient boosting output after $n=1000$ steps.  Panel c: evolution of the mean squared error.} \label{fig:example-gradient-boosting}
\end{figure}

\subsection{Main results}\label{sec:main-results} 
We present our main results on the existence of the vanishing-learning-rate asymptotic for  gradient boosting based on softmax regression trees and the characterization of the corresponding dynamics. For the sake of simplicity, we state them first in the Banach space $\mathbb{B}=\mathbb{B}([0,1]^p,\mathbb{R})$ of bounded functions endowed with the sup norm. However, because of separability issues, we will develop our theory in the space  $\mathbb{T}=\mathbb{T}([0,1]^p,\mathbb{R})$ of tree functions introduced in the next section. We state
our results under the following assumptions.

\begin{assumption} \label{ass:A} Let the  input $(x_i,y_i)_{1\leq i\leq n}$ be fixed and consider gradient boosting with loss function $L$, softmax regression tree with parameter $(\beta,K,d)$ as base learner   and  learning rate $\lambda>0$. Denote by $(\hat F_m^\lambda)_{m\geq 0}$ the corresponding sequence of predictors defined by Equations~\eqref{eq:boosting-init}-\eqref{eq:boosting-update}.\\
Furthermore, assume the loss function satisfies the following conditions: 
\begin{enumerate}[label={$(A\arabic*)$}]
\item \label{ass:A1} $L:\mathbb{R}^2\to [0,\infty)$ is twice differentiable with respect to the second variable and, for all $y\in\mathbb{R}$,  $z\mapsto \frac{\partial^2 L}{\partial z^2}(y,z)$ is positive and locally Lipschitz.
\item \label{ass:A2} For each $C>0$, 
$\sup_{(y,z):L(y,z)\leq C} \left| \frac{\partial L}{\partial z}(y,z) / \frac{\partial^2 L}{\partial z^2}(y,z)\right|  <\infty$.
\end{enumerate}
\end{assumption}
Note that these assumptions are satisfied in the case of regression and binary classification considered in Examples~\ref{ex:regression},~\ref{ex:classification} and~\ref{ex:AdaBoost} above.

Equation~\eqref{eq:Markov} implies that the sequence  $(\hat F_m^\lambda)_{m\geq 0}$ is a $\mathbb{B}$-valued Markov chain. We consider its asymptotic as the learning rate $\lambda\to 0$ and the number of iterations $m$ is rescaled accordingly. Our first result states that the limit does exist and is deterministic. We call this regime the vanishing-learning-rate asymptotic.
 
\begin{theorem}\label{thm:cv-lambda-to-0}
Under Assumption~\ref{ass:A}, there exists a $\mathbb{B}$-valued  process $(\hat F_t)_{t\geq 0}$ such that, for all $T>0$,
\begin{equation}\label{eq:cv-lambda-to-0}
\sup_{t\in [0,T]} \sup_{x\in [0,1]^p}|\hat F_{[t/\lambda]}^\lambda(x)-\hat F_t(x)| \longrightarrow 0 \quad \mbox{in probability as $\lambda\to 0$.}
\end{equation}
Furthermore, the limit process is continuous as a function of $t$ and deterministic.
\end{theorem}

The proof is postponed to \Cref{sec:proofsCV}, and is deduced from a more specific but technical result (\Cref{thm:cv-L2x}) that we will state in \Cref{sec:CV}.
Note that, due to their path regularities, the processes in Equation~\eqref{eq:cv-lambda-to-0} are separable and we can freely restrict the supremum to rational values so as to avoid measurability issues.

For $\lambda>0$, the rescaled sequence $(\hat F_{[t/\lambda]}^\lambda)_{t\geq 0}$ defines a c\`adl\`ag stochastic process. At times $\lambda m$, $m\geq 1$, jumps occur as randomized trees are added to the model.  Theorem~\ref{thm:cv-lambda-to-0} states that both jumps and randomness disappear in the vanishing-learning-rate asymptotic. We call the limit process $(\hat F_t)_{t\geq 0}$ the  \textit{infinitesimal gradient boosting process}. 
We will additionally argue that the variations from the deterministic limit are normal and of order $\sqrt{\lambda}$ as $\lambda\to 0$.
More precisely, we will prove the following theorem as a by-product of the proof of \Cref{thm:cv-lambda-to-0}, where $L^2_{\mathbf{x}}$ is an adequate function space that will be properly defined in \Cref{sec:T-space}.
\begin{theorem} \label{thm:second-order}
   Under \Cref{ass:A}, for all $T>0$,
  \[
    \sup_{t\in[0,T]}\norm{\hat{F}^{\lambda}_{[t/\lambda]} - \hat{F}_t}_{L^2_{\mathbf{x}}} \;= \; O_{\mathbb{P}}(\sqrt{\lambda}),
  \]
where $O_{\mathbb{P}}$ is the standard notation for  ``big O in probability''.
\end{theorem}
A more precise derivation of the second-order variations around the limit goes beyond the scope of the present paper: we expect it would require stronger assumptions than \Cref{ass:A} and careful arguments to tackle the inevitable technical difficulties related to infinite-dimensional diffusion processes that should arise.
The proof of \Cref{thm:second-order}, and more detailed heuristics for these second-order variations, shall be developed in \Cref{sec:CLT}.

Our second main result is the characterization of the infinitesimal gradient boosting process as the solution of a differential equation in the Banach space $\mathbb{B}=\mathbb{B}([0,1]^p,\mathbb{R}^d)$. We first define the infinitesimal boosting operator that drives the dynamics.
\begin{definition}\label{def:boosting-operator}
Let $(x_i)_{1\leq i\leq n}$ be fixed. The infinitesimal boosting operator $\mathcal{T}:\mathbb{B}\to\mathbb{B}$ is defined by
\[
\mathcal{T}(F)(x)=\mathbb{E}_\zeta[\widetilde{T}(x;F,\zeta)], \quad x\in[0,1]^p, F\in\mathbb{B},
\]
where $\widetilde{T}(\,\cdot\,;F,\zeta)$ denotes the softmax gradient tree from Definition~\ref{def:softmax-gradient-tree}.
\end{definition}
In words, the infinitesimal boosting operator at $F$ is the expectation of the softmax gradient tree used in gradient boosting when updating the predictor $F$. In general, $\mathcal{T}$ is a non linear operator. It implicitly depends on the sample $(x_i,y_i)_{1\leq i\leq n}$, on the loss function $L$ and on the parameter $(\beta,K,d)$ used for the softmax regression trees. 

\begin{theorem}\label{thm:EDO}
Suppose Assumption~\ref{ass:A} is satisfied. In the Banach space $\mathbb{B}$, consider the differential equation 
\begin{equation}\label{eq:ODE}
F'(t)=\mathcal{T}(F(t)),\quad t\geq 0.
\end{equation}
The following properties are satisfied:
\begin{enumerate}[(i)]
\item For all  $F_0\in\mathbb{B}$, Equation~\eqref{eq:ODE} admits a unique solution defined on  $[0,\infty)$ and started at $F_0$;
\item The infinitesimal gradient boosting process $(\hat F_t)_{t\geq 0}$ is  the solution of~\eqref{eq:ODE}  started at $F_0\equiv \argmin_{z\in\mathbb{R}}\frac{1}{n} \sum_i L(y_i,z)$.
\end{enumerate}
\end{theorem}

The proof is postponed to \Cref{sec:proofsCV}.
Theorem~\ref{thm:EDO} reveals the dynamics associated with infinitesimal gradient boosting. In Section~\ref{sec:inf-gb}, we will state stronger version of this result (\Cref{thm:cv-L2x}) in the space of tree functions and study more precisely the properties of the infinitesimal gradient boosting process in terms of regularity and long time behavior.
In particular, it should be noted that the map $\mathcal{T}:\mathbb{B}\to\mathbb{B}$ has some properties akin to that of gradient fields: solutions to \eqref{eq:ODE} have a non-increasing training error $\sum_iL(y_i,F_t(x_i))$ (see \Cref{prop:properties}), and under some technical assumptions this error tends to $0$ as $t\to\infty$ (see \Cref{prop:igb-asymptotic}).

\begin{remark}
Let us stress that this paper focuses on gradient boosting  with a fixed learning rate -- i.e. $\lambda$  does not depend on $m$ in Equation~\eqref{eq:Markov} -- and considers the asymptotic when  $\lambda\to 0$ and time is accelerated  by a factor $1/\lambda$. This regime is very similar to the one used in the discretization of ODEs and Theorems~\ref{thm:cv-lambda-to-0} and~\ref{thm:EDO} can be interpreted as stochastic versions of the convergence of the Euler scheme associated with Equation~\eqref{eq:ODE}. Different regimes are also considered in the literature on gradient boosting and stochastic gradient descent, where the learning rate / step size depends on time, i.e. $\lambda$ is replaced by $\lambda_m\to 0$ in Equation~\eqref{eq:Markov} --  see \cite{BC21} for instance for a convergence result in this different setting.
\end{remark}

\section{Properties of softmax regression trees} \label{sec:formal-softmax-trees}
\subsection{The distribution of softmax regression trees}
We propose now a formal definition  of the softmax regression tree with parameter $(\beta,K,d)$ considered in Definition~\ref{def:softmax-reg-tree} and first set up some notation. 

The binary rooted tree with depth $d\geq 1$ (from graph theory) is defined on the vertex set $\mathscr{T}_d=\cup_{l=0}^d \{0,1\}^l$. A vertex $v\in \{0,1\}^l$ is   seen as a word of size $l$ in the letters $0$ and $1$. The empty word  $v=\emptyset$ corresponds to the tree root. The vertex set is divided into the internal nodes $v\in \mathscr{T}_{d-1}$ and the terminal nodes $v\in\{0,1\}^d$, also called leaves. Each internal node $v$ has two children denoted $v0$ and $v1$ (concatenation of words) while the terminal nodes have no offspring. 

A regression tree is encoded by a \textit{splitting scheme} 
\[
\xi=(j_v,u_v)_{v\in\mathscr{T}_{d-1}}\in ([\![1,p]\!]\times (0,1))^{\mathscr{T}_{d-1}}
\]
giving the splits at each internal node, and its leaf values $(\bar r_v)_{v\in\{0,1\}^d}$. The splitting scheme $\xi$ allows to associate to each vertex $v\in\mathscr{T}_{d}$ a region $A_v=A_v(\xi)$ defined recursively by $A_\emptyset=[0,1]^p$ and, for $v\in\mathscr{T}_{d-1}$, 
\begin{align*}
    &A_{v0} = A_v \cap \{x\in[0,1]^p\ :\ x^{j_{v}} < a_v+u_v(b_v-a_v)\},\\
 &A_{v1} = A_v \cap \{x\in[0,1]^p\ :\ x^{j_{v}} \geq a_v+u_v(b_v-a_v)\},
  \end{align*}
with $a_v=\inf_{x\in A_v} x^{j_v}$ and $b_v=\sup_{x\in A_v} x^{j_v}$. Note that $A_v$ depends only on the splits attached to the ancestors of $v$. For each level $l=0,\ldots,d$,  $(A_v)_{v\in \{0,1\}^l}$ is a partition   of $[0,1]^p$ into $2^l$ hypercubes. The leaf values are then
\[
\bar r_v=\bar r(A_v)=\frac{1}{n(A_v)}\sum_{i:x_i\in A_v} r_i.
\]
Finally, the tree function associated with the sample $(x_i,r_i)_{1\leq i\leq n}$ and  splitting scheme $\xi=(j_v,u_v)_{v\in\mathscr{T}_{d-1}}$ is the piecewise constant function
\begin{equation}\label{eq:reg-tree}
T(x;(x_i,r_i)_{1\leq i\leq n},\xi)=\sum_{v\in\{0,1\}^d}\bar r(A_v)\1_{A_v}(x).
\end{equation}
See Figure~\ref{fig:muT1} for an illustration of a splitting scheme with depth two in one dimension, and the tree function associated with it.
Figure~\ref{fig:muT2} shows a two-dimensional splitting scheme with depth two, and the associated partition of $[0,1]^2$ into four leaves.

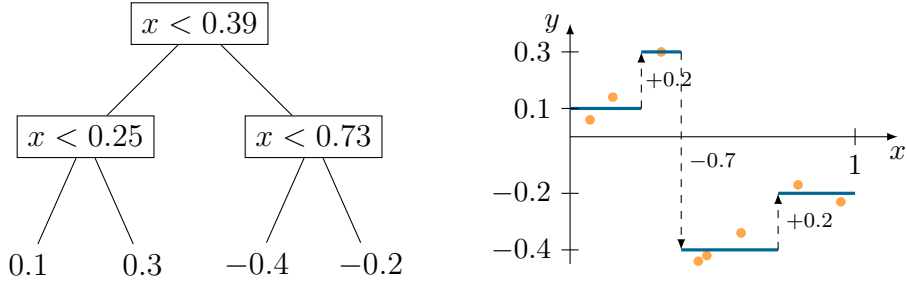
\begin{figure}[ht] \centering
  \begin{tikzpicture}[level 1/.style={sibling distance=2cm},
  level 2/.style={sibling distance=1cm},
  level distance=1cm,scale=1.5]
  \coordinate (root) at (0,0.3);
  \draw (root) node[draw] {$x < 0.39$}
  child { node[draw] {$x < 0.25$}
    child {node[yshift=-2mm,align=center] {$0.1$}  edge from parent}
    child {node[yshift=-2mm,align=center] {$0.3$}}
  } 
  child { node[draw] {$x < 0.73$}
    child {node[yshift=-2mm,align=center] {$-0.4$}}
    child {node[yshift=-2mm,align=center] {$-0.2$} edge from parent}
  };
  \begin{scope}[xshift=3.25cm,yshift=-.7cm,scale=2.5]
  \begin{scope}[color=orange!70]
    \fill (.07,.06) circle (0.17mm);
    \fill (.15,.14) circle (0.17mm);
    \fill (.32,.3) circle (0.17mm);
    \fill (.45,-.44) circle (0.17mm);
    \fill (.48,-.42) circle (0.17mm);
    \fill (.6,-.34) circle (0.17mm);
    \fill (.8,-.17) circle (0.17mm);
    \fill (.95,-.23) circle (0.17mm);
  \end{scope} \small
  \draw[->,>=latex] (0,-.45) -- (0,.4) node[left] {$y$};
  \draw[->,>=latex] (0,0) -- (1.15,0) node[below] {$x$};
  \draw[very thick,MidnightBlue] (0,.1) -- (.25,.1) (.25,.3) -- (.39,.3)
                            (.39,-.4) -- (.73,-.4) (.73,-.2) -- (1,-.2);
  \draw (-.03,.1) node[left] {$0.1$} -- (.03,.1);
  \draw (-.03,.3) node[left] {$0.3$} -- (.03,.3);
  \draw (-.03,-.4) node[left] {$-0.4$} -- (.03,-.4);
  \draw (-.03,-.2) node[left] {$-0.2$} -- (.03,-.2);
  \draw (1,-.03) node[below] {$1$} -- (1,.03);
  {\scriptsize
  \draw[->,>=latex,dashed] (.25,.1) -- (.25,.3) node[midway,right] {\!$+0.2$};
  \draw[->,>=latex,dashed] (.39,.3) -- (.39,-.4) node[pos=.55,right] {$-0.7$};
  \draw[->,>=latex,dashed] (.73,-.4) -- (.73,-.2) node[midway,right] {$+0.2$};
  }
  \end{scope}
  \end{tikzpicture}
\caption{%
A regression tree function $T$ with $d=2$ and $p=1$.
The data points are represented by the orange dots on the right.
The splitting scheme is represented on the left and corresponds to $\xi=(j_v,u_v)_{v\in\mathscr{T}_{d-1}}$ given by: $j_v=1$ for all $v$, and  $u_{\emptyset}=0.39$, $u_{0}=0.25/0.39\approx 0.64$ and $u_{1}=(0.73-0.39)/(1-0.39)\approx 0.56$.
The leaf values are below the corresponding leaves.
The corresponding function $T:[0,1]\to\mathbb{R}$ is represented on the right.}
\label{fig:muT1}
\end{figure}

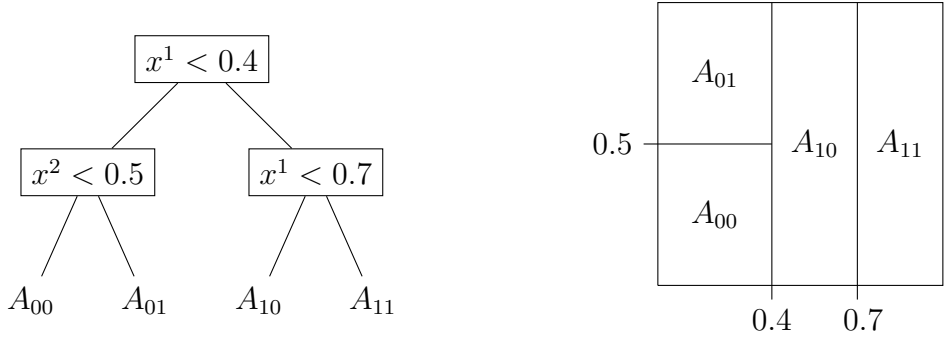
\begin{figure}[ht] \centering
  \begin{tikzpicture}[level 1/.style={sibling distance=2cm},
  level 2/.style={sibling distance=1cm},
  level distance=1cm,scale=1.5]
  \coordinate (root) at (0,0.3);
  \draw (root) node[draw] {$x^{1} < 0.4$}
  child { node[draw] {$x^{2} < 0.5$}
    child {node[yshift=-2mm,align=center] {$A_{00}$}}
    child {node[yshift=-2mm,align=center] {$A_{01}$}}
  } 
  child { node[draw] {$x^{1} < 0.7$}
    child {node[yshift=-2mm,align=center] {$A_{10}$}}
    child {node[yshift=-2mm,align=center] {$A_{11}$}}
  };
  \begin{scope}[xshift=4cm,yshift=-1.7cm,scale=2.5]
  \draw (0,0) -- (0,1) -- (1,1) -- (1,0) -- (0,0);
  \draw (.4,-.05) node[below]{$0.4$} -- (.4,1);
  \draw (.7,-.05) -- (.7,1);
  \draw (-.05,.5) node[left]{$0.5$} -- (.4,.5);
  \node[below] at (.72,-0.05) {$0.7$};
  \node at (0.2,0.25) {$A_{00}$};
  \node at (0.2,0.75) {$A_{01}$};
  \node at (0.55,0.5) {$A_{10}$};
  \node at (0.85,0.5) {$A_{11}$};
  \end{scope}
  \end{tikzpicture}
  \caption{ A tree function $T$ with $d=p=2$. The splitting scheme is represented on the left and corresponds to $\xi=(j_v,u_v)_{v\in\mathscr{T}_{d-1}}$ given by $(j_{\emptyset},u_{\emptyset})=(1,0.4)$, $(j_{0},u_{0})=(2,0.5)$ and $(j_{1},u_{1})=(1,0.5)$.
  The induced partition $(A_v)_{v\in\{0,1\}^d}$ of $[0,1]^2$ is represented on the right.}
  \label{fig:muT2}
\end{figure}

Softmax regression trees are randomized trees and we next describe their distribution. In view of the previous discussion, it is enough to  give the distribution of the associated splitting scheme. Let $P_0$ be the distribution of the splitting scheme of a totally randomized tree, that is the splits $(j_v,u_v)_{v\in\mathscr{T}_{d-1}}$ are  i.i.d.  uniformly distributed on $[\![1,p]\!]\times (0,1)$ under $P_0$. The candidate splits that appear during the procedure can be gathered into candidate splitting schemes  $\xi^1,\ldots,\xi^K$ assumed i.i.d. with distribution $P_0$. A splitting scheme $\xi$ is constructed by selection among the candidate splits, so that  $\xi=(j_v^{\phi(v)},u_v^{\phi(v)})_{v\in\mathscr{T}_{d-1}}=:\xi^\phi$ for some selection map $\phi:\mathscr{T}_{d-1}\to [\![1,K]\!]$ specifying which split is selected at each internal node.  According to the softmax selection rule, the distribution of $\xi$ given $\xi^1,\ldots,\xi^K$ is given by
\begin{align}
&\P(\xi=(j_v^{\phi(v)},u_v^{\phi(v)})_{v\in\mathscr{T}_{d-1}}\mid \xi^1,\ldots,\xi^K)\nonumber\\
&=\prod_{v\in \mathscr{T}_{d-1}} \big(\mathrm{softmax}_\beta( \Delta(s_v^{k},u_v^{k}; A_v(\xi^\phi))_{1\leq k\leq K})\big)_{\phi(v)}.\label{eq:law-of-xi}
\end{align}
In this equation, the factor indexed by the internal node $v$ corresponds to the probability that the region $A_v(\xi)$ is split  according to $(j_v^{\phi(v)},u_v^{\phi(v)})$ among the candidates $(j_v^{k},u_v^k)$, $1\leq k\leq K$. It should be noted that $A_v(\xi)$ only depends on the splits $(j_{v'},u_{v'})$ attached to the ancestors  $v'$ of $v$ which  implies that $(A_v(\xi))_{v\in\mathscr{T}_{d-1}}$ enjoys a branching Markov property which naturally corresponds to the recursive procedure described above.

We next deduce a characterization of the distribution  of the splitting scheme $\xi$, which, in view of Equation~\eqref{eq:reg-tree}, characterizes the distribution of the softmax regression tree. The distribution depends on the parameter $(\beta,K,d)$ and also on the input $(x_i,r_i)_{1\leq i\leq n}$ but  we only write $P_{\beta,K}$ for the sake of readability.

\begin{proposition}\label{prop:RN-derivative}
Let $(x_i,r_i)_{1\leq i\leq n}$ and $(\beta,K,d)$ be fixed. We denote by $P_{\beta,K}$ the distribution  of the splitting scheme $\xi$ associated with the softmax regression tree with parameter $(\beta,K,d)$ grown on the sample $(x_i,r_i)_{1\leq i\leq n}$. Then $P_{\beta,K}$ is absolutely continuous with respect to $P_0 $ with Radon-Nykodym derivative
\begin{equation}\label{eq:RN}
\frac{\rmd P_{\beta,K}}{\rmd P_0}(\xi)= 
\int \prod_{v\in \mathscr{T}_{d-1}}  \frac{\exp(\beta \Score(j_v^1,u_v^1; A_v(\xi)))}{K^{-1}\sum_{k=1}^K \exp(\beta \Score(j_v^{k},u_v^{k}; A_v(\xi)))} P_0(\rmd \xi^2)\cdots P_0(\rmd \xi^K),
\end{equation}
with $\xi^k=(j_v^k,u_v^k)_{v\in\mathscr{T}_{d-1}}$, $2\leq k\leq K$, and for $k=1$, we take $\xi^1=(j_v^1,u_v^1)_{v\in\mathscr{T}_{d-1}}=\xi$.
\end{proposition}
The proof is postponed to \Cref{sec:proofs-reg-trees}.
Note that the Radon-Nykodym derivative~\eqref{eq:RN} is bounded from above by $K^{2^d-1}$.

\begin{remark} \label{rk:zeta-xi}
In Definition~\ref{def:softmax-reg-tree}, the softmax regression tree was introduced in terms of an external randomness $\zeta$. In view of the preceding discussion, we can take $\zeta=(\xi^1,\ldots,\xi^K,\gamma)$ where $\xi^1,\ldots,\xi^K$ denote  $K$ independent splitting schemes with distribution $P_0$ giving the candidate splits  and $\gamma=(\gamma_v)_{v\in\mathscr{T}_{d-1}}$ be independent uniform random variables on $[0,1]$ used to perform the softmax selection at each internal node (using e.g.\ the probability integral transform). In the following, it will be  convenient to write 
\[
T(\,\cdot\,,(x_i,r_i)_{1\leq i\leq n},\zeta)= T(\,\cdot\,,(x_i,r_i)_{1\leq i\leq n},\xi)
\]
where the splitting scheme $\xi$ has distribution $P_{\beta,K}$ and can be seen as a (deterministic) function of $\zeta=(\xi^1,\ldots,\xi^K,\gamma)$ and $(x_i,r_i)_{1\leq i\leq n}$.
\end{remark}

\begin{remark}\label{rk:beta-to-infty}
The softmax distribution converges to the uniform distribution as $\beta\to 0$ and concentrates on the $\argmax$ set as $\beta\to+\infty$. More precisely, 
for $z\in\mathbb{R}^K$ and $I_z=\mathop{\mathrm{\argmax}}_{1\leq k\leq K} z_k$,
\[
\frac{\exp(\beta z_1)}{\sum_{k=1}^K\exp(\beta z^k)}\longrightarrow
\left\{ \begin{array} {ll}
1/K & \mbox{as $\beta \to 0$} \\
\1_{I_z}(1)/\mathrm{card}(I_z)& \mbox{as $\beta \to +\infty$}
\end{array}\right.. 
\]
We deduce readily that  $\rmd P_{\beta,K}/\rmd P_0$ has  limits when $\beta\to 0$ or $\beta\to +\infty$. By Scheff\'e's theorem, this implies convergence  in total variation of the  splitting scheme distribution. These limits   correspond  respectively to totally randomized trees  and  Extra-Trees.
\end{remark}

\begin{remark} \label{rk:K-to-infty} The law of large numbers implies that, for  $(\xi^k)_{k\geq 2}$ i.i.d. with distribution $P_0$ , 
\[
\frac{\exp(\beta \Score(j_v^1,u_v^1; A_v(\xi)))}{K^{-1}\sum_{k=1}^K \exp(\beta \Score(j_v^{k},u_v^{k}; A_v(\xi)))}
\longrightarrow 
\frac{\exp(\beta  \Score(j_v^1,u_v^1; A_v(\xi)))}{\mathbb{E}[\exp(\beta  \Score(j,u; A_v(\xi)))]} \quad \mbox{a.s.},
\]
as $K\to\infty$, with  expectation taken with respect to $(j,u)$ uniformly distributed on $[\![1,p]\!]\times (0,1)$. This limit is interpreted as an exponentially tilted distribution for the law of the split $(j_v^1,u_v^1)$ of the region $A_v(\xi)$. Proposition~\ref{prop:RN-derivative} and dominated convergence then implies, as $K\to\infty$, 
\[
\frac{\rmd P_{\beta,K}}{\rmd P_{0}}(\xi)\longrightarrow \frac{\rmd P_{\beta,\infty}}{\rmd P_{0}}(\xi):=\prod_{v\in \mathscr{T}_{d-1}} \frac{\exp(\beta  \Delta(j_v,u_v; A_v(\xi)))}{\mathbb{E}[\exp(\beta  \Delta(j,u; A_v(\xi))) ]}, 
\]
where $\xi=((j_v,u_v))_{v\in\mathscr{T}_{d-1}}$.
This limit corresponds to the distribution of the splitting scheme of a randomized regression tree  grown thanks to a binary splitting rule based on the  exponentially tilted distribution. Furthermore, letting $\beta\to+\infty$, the exponentially tilted distribution concentrates on the $\argmax$ set so that  Breiman's regression trees are recovered.
\end{remark}

\subsection{A crucial Lipschitz property}
Our motivation for softmax regression trees is the regularity property of the mean tree defined by
\begin{align*}
\bar T_{\beta,K,d}(x; (x_i,r_i)_{1\leq i\leq n})&=\mathbb{E}_\zeta[T(x;(x_i,r_i)_{1\leq i\leq n},\zeta)]\\
&=\int T(x; (x_i,r_i)_{1\leq i\leq n},\xi) P_{\beta,K}(\rmd \xi).
\end{align*}
See Remark~\ref{rk:zeta-xi} for the relationship between $\zeta$ and $\xi$ and  recall that $P_{\beta,K}$ depends on $(x_i,r_i)_{1\leq i\leq n}$. This will be handled with care in the proof with more explicit notation.

\begin{proposition}\label{prop:reg-tree-Lipschitz}
Let $(x_i)_{1\leq i\leq n}\in [0,1]^p$ and $(\beta,K,d)$ be fixed. The mean softmax regression tree is locally Lipschitz in its input $(r_i)_{1\leq i\leq n}$, i.e 
\[
\left\{\begin{array}{lll}
(\mathbb{R}^n,\norm{\cdot}_{\infty})&\longrightarrow&(\mathbb{B},\norm{\cdot}_{\infty})\\
(r_i)_{1\leq i\leq n}&\mapsto &\bar T_{\beta,K,d}(\,\cdot\,;(x_i,r_i)_{1\leq i\leq n})
\end{array}\right.
\]
 is locally Lipschitz.
\end{proposition}
The proof is postponed to \Cref{sec:proofs-reg-trees}.
The same result does not hold when $\beta=\infty$, i.e.\ for Extra-Trees and Breiman's regression trees, because of the lack of continuity of the $\argmax$ selection.

\section{The space of tree functions}\label{sec:T-space}
\subsection{Tree function space and total variation norm}
An original and powerful point of view to study the sequence of functions $(\hat F_m^\lambda)_{m\geq 0}$ produced by gradient boosting is to see it as a Markov chain in a suitable function space that we call the space of tree functions $\mathbb{T}=\mathbb{T}([0,1]^p,\mathbb{R})$. The main idea is that a regression tree is naturally associated with a signed measured. This is especially useful because the space of non-negative measures enjoys nice compactness properties that will ease tightness considerations. We introduce here  the space of tree functions and discuss some useful properties.

We first introduce some notation. Let $\mathcal{M}=\mathcal{M}([0,1]^p)$ (resp. $\mathcal{M}^+=\mathcal{M}^+([0,1]^p))$ denote the space of Borel finite signed measures on $[0,1]^p$ (resp. Borel finite non-negative measures). The total variation of a signed measure $\mu\in\mathcal{M}$ is 
\[
\|\mu\|_{\mathrm{TV}}=\sup\big\{\mu(f)\ :\ f\in \mathbb{B}, \|f\|_\infty\leq 1\big\},
\]
where we let $\mu(f)=\int f \,\rmd \mu$ denote the integral. We note $\mu=\mu^+-\mu^-$ the Jordan decomposition of $\mu$ and $|\mu|=\mu^++\mu^-$ the  variation  of $\mu$ --- see \citet[Section 32]{B95} for background on signed measures. We also consider the subset $\mathcal{M}_0\subset \mathcal{M}$  of signed measures $\mu$ satisfying $\abs{\mu}([0,1]^p\setminus[0,1)^p)=0$.
In the rest of the paper we will often use the notation
\[
	[0,x] := [0,x_1]\times[0,x_2]\times \dots\times[0,x_p],
\]
for any $x\in [0,1]^p$.

\begin{definition}
The space of tree functions $\mathbb{T}=\mathbb{T}([0,1]^p,\mathbb{R})$ is defined as
\[
\mathbb{T}=\big\{T:[0,1]^p\to\mathbb{R}\;:\ T(x)=\mu([0,x])\ \mbox{for some $\mu\in \mathcal{M}_0$ and all $x\in[0,1]^p$}\big\}.
\]
For all $T\in\mathbb{T}$, the measure $\mu\in\mathcal{M}_0$ such that $T(x)=\mu([0,x])$
 is unique and denoted by $\mu_T$. The total variation of $T$ is defined by $\|T\|_{\mathrm{TV}}=\|\mu_T\|_{\mathrm{TV}}$. 
\end{definition}
Clearly, $\norm{\cdot}_{\mathrm{TV}}$ defines a norm on $\mathbb{T}$, and the mapping $\mu\in\mathcal{M}_0\mapsto T\in\mathbb{T}$ is an isomorphism so that the space of tree functions $(\mathbb{T},\norm{\cdot}_{\mathrm{TV}})$ is a Banach space. Furthermore, the inclusion $\mathbb{T}\subset\mathbb{B}$ holds with inequality $\|T\|_{\infty}\leq \|T\|_{\mathrm{TV}}$ so that the injection is continuous.

We now explain the relationship between regression trees and the space $\mathbb{T}$.
Figures~\ref{fig:muT1} and~\ref{fig:muT2} provide an illustration of the fact that regression trees belong to $\mathbb{T}$ in dimension $p=1$ and $p=2$ respectively.
In Figure~\ref{fig:muT1}, we see that $T\in \mathbb{T}$ with $\mu_T=0.1\delta_{0}+0.2\delta_{0.25}-0.7\delta_{0.39}+0.2\delta_{0.73}$, and $\norm{T}_{\mathrm{TV}}=0.1+0.2+0.7+0.2=1.2$.
In Figure~\ref{fig:muT2}, letting $a,b,c$ and $d$ denote the respective leaf values on $A_{00},A_{01},A_{10}$ and $A_{11}$, we can see that $T\in \mathbb{T}$ with $\mu_T=a\delta_{(0,0)}+(b-a)\delta_{(0,0.5)}+(c-a)\delta_{(0.4,0)}+(a-b)\delta_{(0.4,0.5)}+(d-c)\delta_{(0.7,0)}$; therefore we also get $\norm{T}_{\mathrm{TV}}\leq 4\abs{a}+2\abs{b}+2\abs{c}+\abs{d} \leq 4^2\norm{T}_{\infty}$.

\begin{proposition}\label{prop:reg-tree-measure}
Let $\xi\in ([\![1,p]\!]\times (0,1))^{\mathscr{T}_{d-1}}$ be a splitting scheme with depth $d$ and $(\tilde r_v)_{v\in\{0,1\}^d}\in\mathbb{R}^{{\{0,1\}}^d}$ be leaf values. Then the function
\[
T(x)=\sum_{v\in\{0,1\}^d}\tilde r_v \1_{A_v(\xi)}(x),\quad x\in[0,1]^p,
\]
belongs to $\mathbb{T}$ and satisfies 
\[
\|T\|_{\mathrm{TV}}\leq 2^{d+\min(d,p)}\|T\|_{\infty} \leq 4^d \norm{T}_{\infty}.
\]
Furthermore, if $d<p$, the measure $\mu_T$ is supported by the set $\mathcal{F}_d$ of points $x\in[0,1]^p$ with at most $d$ positive components.
\end{proposition}
The proof is postponed to \Cref{sec:proofs-T-space}.

\begin{remark}~ \label{rk:tree-space-props}
In dimension $p=1$, $\mathbb{T}([0,1],\mathbb{R})$ is exactly the set of càdlàg functions with finite total variation that are continuous at $1$. In dimension $p\geq 2$, tree functions $T\in\mathbb{T}$ can be seen as the ``cumulative distribution function'' of a signed measure and therefore have quartant limits and are right-continuous.  For instance, when $p=2$, the four limits $\lim_{y\to x, y\in Q_x^l}T(y)$ exist for all $x\in[0,1]^2$ and $l=1,\ldots,4$, where 
\begin{align*}
Q_x^1&=\{y : y^1\geq x^1,y^2\geq x^2\},\quad &Q_x^2=\{y: y^1\geq x^1,y^2<x^2\},\\ Q_x^3&=\{y: y^1< x^1,y^2\geq x^2\},\quad &Q_x^4=\{y: y^1< x^1,y^2<x^2\}.
\end{align*}
Since $x\in Q_x^1$, $T(x)$ is equal to the limit on $Q_x^1$.
More generally, in dimension $p\geq 1$ there are $2^p$ such quartants. We refer to~\cite{Neu71} for the definition and properties of the multiparameter Skorokhod space  $\mathbb{D}([0,1]^p,\mathbb{R})$. The inclusion $\mathbb{T}([0,1]^p,\mathbb{R})\subset \mathbb{D}([0,1]^p,\mathbb{R})$ holds. For each $T\in\mathbb{T}$, the condition $\mu_T\in \mathcal{M}_0$ implies that $T$ is continuous on $[0,1]^p\setminus[0,1)^p$, i.e.\ at each point having at least one component equal to $1$. 
\end{remark}

\subsection{Decomposition of tree functions and \texorpdfstring{$L^2$}{L2}-norm} \label{sec:L2}
We now consider a specific decomposition of tree functions that will be useful in the sequel. For $J\subset [\![1,p]\!]$ and $\epsilon\in \{0,1\}^{J^c}$, define the $\abs{J}$-dimensional face $\mathcal{F}_{J,\epsilon}$ of $[0,1]^p$ by
\[
  \mathcal{F}_{J,\epsilon} = \{x\in [0,1]^p\ :\ \forall j\in J,\,x^j\in (0,1);\;\forall j\in J^c,\,x^j =\epsilon_j\}.
\]
When $\epsilon$ is identically null on $J^c$, we simply write $\mathcal{F}_J$.

\begin{example}
In dimension $p=2$, the square $[0,1]^2$ is decomposed into corners, edges and interior: the faces of dimension $0$, $1$ and $2$ are respectively given by
\begin{itemize}
\item $\mathcal{F}_{\emptyset}=\{(0,0)\}$, $\mathcal{F}_{\emptyset,(0,1)}=\{(0,1)\}$, $\mathcal{F}_{\emptyset,(1,0)}=\{(1,0)\}$, $\mathcal{F}_{\emptyset,(1,1)}=\{(1,1)\}$;
\item $\mathcal{F}_{\{1\}}=\{0\}\!\times\! (0,1)$, $\mathcal{F}_{\{2\}}=(0,1)\!\times\! \{0\}$, $\mathcal{F}_{\{1\},1}=\{1\}\!\times\! (0,1)$, $\mathcal{F}_{\{2\},1}=(0,1)\!\times\! \{1\}$;
\item $\mathcal{F}_{\{1,2\}}=(0,1)^2$.
\end{itemize}
\end{example}

Each measure $\mu\in\mathcal{M}_0$ is supported by $[0,1)^p$ and can be decomposed as
\[
\mu=\sum_{J\subset [\![1,p]\!]} \mu^J\quad \mbox{with } \mu^J(\cdot)=\mu(\cdot \cap \mathcal{F}_J).
\]
Correspondingly, each tree function $T\in\mathbb{T}$ can be decomposed as
\[
T=\sum_{J\subset [\![1,p]\!]} T^J\quad \mbox{with } T^J(x)=\mu_T^J([0,x]),
\]
where  $T^J(x)$ depends  on $x$ only through the components  $x^j$ for $j\in J$. Alternatively the decomposition is  characterized recursively by  $T^\emptyset=T(0)$ and 
\[
T^J(x)=T(x^J)-\sum_{J'\subsetneq J}T^{J'}(x) ,\quad J\neq  \emptyset,
\]
where $x^J=(x^j\1_{j\in J})_j$ is the vector with components $x^j$ if $j\in J$ and $0$ if $j\notin J$.

\begin{remark}~ \label{rk:interaction}
\begin{enumerate}
\item From a statistical perspective, $T^J(x)$  is interpreted as the interaction effect of the covariates $x^j$, $j\in J$, when the value $0$ is taken as reference. This is more easily explained in the case $p=2$:  $T^\emptyset(x)=T(0,0)$ corresponds to the value of $T$ with both covariates at reference; $T^{\{1\}}(x^1,x^2)=T(x^1,0)-T(0,0)$ corresponds to the additional effect of  $x^1$ with $x^2$ at reference and similarly for $T^{\{2\}}(x^1,x^2)$; finally, $T^{\{1,2\}}(x^1,x^2)=T(x^1,x^2)-T(x^1,0)-T(0,x^2)+T(0,0)$ corresponds to the additional interaction effect between $x^1$ and $x^2$. 
\item According to Proposition~\ref{prop:reg-tree-measure}, the measure $\mu_T$ of a regression tree $T$ with depth $d<p$ is supported by $\mathcal{F}_d=\cup_{\abs{J}\leq d}\mathcal{F}_J$. This implies that $T^J=0$ whenever $\abs{J}>d$, i.e.\ there are no interaction effect involving more than $d$ variables. In particular, when $d=1$,  
$T(x)=T(0) +\sum_{j=1}^p T^{\{j\}}(x)$ where $T^{\{j\}}(x)$ depends on  $x^j$ only, which corresponds to an additive model in the covariates $x^1,\ldots,x^p$.
\end{enumerate}
\end{remark}

Next we consider $\mathbb{T}$ as a subset of a well-chosen $L^2$ space. This is useful because the total variation norm induces a topology that is too strong for our purpose (because the Banach space $(\mathbb{T},\norm{\cdot}_{\mathrm{TV}})$ is non-separable).
Since the space $L^2$ that we will construct is the usual $L^2$ space associated with a certain probability measure, the topology that is induced on $\mathbb{T}$ is separable, allowing us to use techniques for the convergence of stochastic processes that are tailored for such spaces.
For  $J\subset [\![1,p]\!]$ and $\epsilon\in \{0,1\}^{J^c}$, we define the  measure $\leb_{J,\epsilon}$ on $[0,1]^p$  supported by the $\abs{J}$-dimensional face  $\mathcal{F}_{J,\epsilon}$ by
\[
  \leb_{J,\epsilon}(\rmd x) = \prod_{j\in J^c}\delta_{\epsilon_j}(\rmd x^j)\prod_{j\in J}\rmd x^j.
\]
We consider the probability measure $\nu$ on $[0,1]^p$ defined by
\[
  \nu = \frac{1}{3^p}\sum_{J,\epsilon} \leb_{J,\epsilon}
\]
and the Hilbert space  $L^2 = L^2([0,1]^p,\nu)$ with associated norm $\norm{\cdot}_2$. Note that if the tree functions $T_1,T_2 \in\mathbb{T}$ are equal $\nu(\rmd x)$-almost everywhere, then they are equal everywhere because they are right-continuous in the sense of Remark~\ref{rk:tree-space-props}. This implies that $\norm{\cdot}_2$ is  a norm on $\mathbb{T}$. Therefore we can write $\mathbb{T} \subset L^2$, and we have $\norm{T}_2\leq  \norm{T}_{\infty}\leq \norm{T}_{\mathrm{TV}}$ for all $T\in\mathbb{T}$. Note that $\mathbb{T}$ is not closed in $L^2$ since for instance $\1_{[1/n,1]^p}\in \mathbb{T}$ converges in $L^2$ to  $\1_{(0,1]^p}\notin \mathbb{T}$ as $n\to+\infty$.
The proofs of the two following results, relating $L^2$ convergence on $\mathbb{T}$ with other modes of convergence, are postponed to \Cref{sec:proofs-T-space}.

\begin{proposition}~ \label{prop:L2-cv-on-T}
  \begin{enumerate}[(i)]
  	\item $\mathbb{T}$ is dense in $L^2$.
    \item Any subset $K\subset \mathbb{T}$ satisfying $\sup_{T\in K}\norm{T}_{\mathrm{TV}} < \infty$ is relatively compact in $L^2$.
    \item Let $T, T_1,T_2,\dots \in \mathbb{T}$ be such that $(\|T_n\|_{\mathrm{TV}})_{n\geq 1}$ is bounded. Then, $\norm{T_n-T}_2 {\to 0}$ if and only if $\mu_{T_n}^{J} \weakcv \mu_{T}^{J}$ for all $J\subset [\![1,p]\!]$, where $\weakcv$ denotes weak convergence of measures.
  \end{enumerate}
\end{proposition}

To deal with tightness and uniform convergence in the proof of Theorem~\ref{thm:cv-lambda-to-0}, it will be convenient  to consider the subset  $\mathbb{T}^+=\{T\in \mathbb{T}: \mu_T\in\mathcal{M}^+\}$ of tree functions  with positive measure.  We denote by $\overline{\mathbb{T}}^+$ the adherence of $\mathbb{T}^+$ in $L^2$ and equip it with the induced metric. 

\begin{proposition}~ \label{prop:L2-T+}
  \begin{enumerate}[(i)]
    \item $\overline{\mathbb{T}}^+$ is a proper metric space, i.e.\ bounded sets are relatively compact.
        \item  Let $T, T_1,T_2,\dots \in \mathbb{T^+}$ and assume $T$ is continuous. Then, $\norm{T_n-T}_2 \to 0$ implies $\norm{T_n-T}_{\infty}\to 0$.
  \end{enumerate}
\end{proposition}

When dealing with gradient boosting, we need the function $F$ to be well-defined at $(x_i)_{1\leq i\leq n}$ because the algorithm involves the values $F(x_i)$. The following remark provides a slight modification of the previous result in order to deal with this issue and introduces the function space $L^2_\mathbf{x}$ that will be useful in Section~4.
\begin{remark}\label{rk:L2x} With the sample $\mathbf{x}=(x_i)_{i=1}^n$ being fixed, we consider the space $L^2_{\mathbf{x}} = L^2([0,1]^p, \nu_{\mathbf{x}})$, where $\nu_{\mathbf{x}}$ is the probability measure
\[
  \nu_{\mathbf{x}} = \frac{\nu}{2}+\frac{1}{2n}\sum_{i=1}^n\delta_{x_i}.
\]
For $F\in L^2_{\mathbf{x}}$, the values $(F(x_i))_{1\leq i\leq n}$ are well defined and we can consider the gradient tree $\widetilde T(\,\cdot\,;F,\xi)$ for $F\in L^2_{\mathbf{x}}$.

It is readily checked that $\mathbb{T}\subset L^2_{\mathbf{x}}$, with $\norm{T}_{L^2_\mathbf{x}}\leq \norm{T}_{\mathrm{TV}}$ and that the statements and proofs of Propositions~\ref{prop:L2-cv-on-T} and~\ref{prop:L2-T+} still hold with $L^2$ replaced by $L^2_{\mathbf{x}}$, the only modification being that  \textit{(iii)} in Proposition~\ref{prop:L2-cv-on-T} has to be replaced by
\begin{itemize}
  \item[\textit{(iii')}] Let $T, T_1,T_2,\dots \in \mathbb{T}$ be such that $(\|T_n\|_{\mathrm{TV}})_{n\geq 1}$ is bounded. Then, ${\norm{T_n-T}}_{L^2_\mathbf{x}}{\to 0}$ if and only if $\mu_{T_n}^{J} \weakcv \mu_{T}^{J}$ for all $J\subset [\![1,p]\!]$ and $T_n(x_i)\to T(x_i)$ for all $i\in [\![1,n]\!]$.
\end{itemize}
The proof is easily adapted thanks to the fact that $L^2_{\mathbf{x}}$ is isomorphic to $L^2\times \mathbb{R}^n$ via the isomorphism $F\mapsto (F,(F(x_i))_{1\leq i\leq n})$.
\end{remark}

\section{Infinitesimal gradient boosting} \label{sec:inf-gb}
\subsection{Existence and uniqueness of the ODE solution} \label{sec:ode}

For the sake of simplicity, our results were first stated in Section~\ref{sec:main-results} in the Banach space $\mathbb{B}=\mathbb{B}([0,1]^p,\mathbb{R})$ of bounded function. We now state and prove a version of Theorem~\ref{thm:EDO} in the Banach space $\mathbb{T}$ of tree functions endowed with the total variation norm. Existence and uniqueness of the ODE solution follow essentially from the fact that the infinitesimal boosting operator introduced in Definition~\ref{def:boosting-operator} is locally Lipschitz. The proof is similar to the proof of Proposition~\ref{prop:reg-tree-Lipschitz} but slightly more involved because we have to consider softmax gradient trees (Definition~\ref{def:softmax-gradient-tree}), that have a more complex structure than softmax regression trees (Definition~\ref{def:softmax-reg-tree}). 

The results of this section and its subsection are important for the proofs of our main results, Theorems~\ref{thm:EDO} and~\ref{thm:cv-lambda-to-0}.
We defer all proofs to \Cref{sec:proofsMain}.

\begin{proposition}\label{prop:boosting-operator-Lipschitz} Suppose Assumption~\ref{ass:A} is satisfied. Then the infinitesimal boosting operator $\mathcal{T}$ defined in Definition~\ref{def:boosting-operator} satisfies $\mathcal{T}(F)\in\mathbb{T}$, for all  $F\in\mathbb{B}$, and the map $\mathcal{T}:(\mathbb{B},\norm{\cdot}_{\infty})\to (\mathbb{T}, \norm{\cdot}_{\mathrm{TV}})$ is locally Lipschitz.
\end{proposition}

\begin{remark}\label{rk:lipschitz}
  Because $\mathbb{T}\subset\mathbb{B}$ and $\norm{T}_{\infty}\leq \norm{T}_{\mathrm{TV}}$ for all $T\in\mathbb{T}$, the previous proposition implies that $\mathcal{T}$ can be seen as a locally Lipschitz operator on $(\mathbb{B},\norm{\cdot}_{\infty})$ and $(\mathbb{T}, \norm{\cdot}_{\mathrm{TV}})$. This is useful when considering the ODE \eqref{eq:ODE} on both spaces, in Theorem~\ref{thm:EDO} or Theorem~\ref{thm:EDO-T} below respectively.
\end{remark}

\begin{theorem}\label{thm:EDO-T}
Consider the differential equation \eqref{eq:ODE} in the space $(\mathbb{T},\norm{\cdot}_{\mathrm{TV}})$. Under Assumption~\ref{ass:A}, for any initial condition $F_0\in\mathbb{T}$, it has a unique solution defined on $[0,\infty)$ and satisfying $F(0)=F_0$.
\end{theorem}

Since the differential equation \eqref{eq:ODE} is driven by the infinitesimal boosting operator which is locally Lipschitz according to Proposition~\ref{prop:boosting-operator-Lipschitz},  the Cauchy--Lipschitz theorem ensures the existence and uniqueness of local solutions. The proof of Theorem~\ref{thm:EDO-T} consists in verifying that the maximal solution corresponding to the local solution started at $F_0$ at time $t=0$ is defined for all time $t\geq 0$.

\subsection{Convergence of gradient boosting} \label{sec:CV}
In this section, we consider  convergence of the gradient boosting sequence  $(\hat F_m^\lambda)_{m\geq 0}$ in the vanishing learning rate asymptotic $\lambda\to 0$. Theorem~\ref{thm:cv-lambda-to-0} states, for all $T>0$, the convergence in probability 
\[
(\hat F_{[t/\lambda]}^\lambda)_{t\in[0,T]} \stackrel{\P}\longrightarrow (\hat F_t)_{t\in[0,T]}
\]
in the Skohorod space $\mathbb{D}([0,T],\mathbb{B})$ endowed with the norm of uniform convergence. The non-separability of the Banach space $\mathbb{B}$ makes it not suitable for the analysis and we first develop our results in the separable space $L^2_{\mathbf x}$ introduced in Section~\ref{sec:L2}. More precisely, we prove the convergence of  the positive and negative parts of the gradient boosting sequence seen in the space $\overline{\mathbb{T}}^+\subset L^2_\mathbf{x}$. The fact that  $\overline{\mathbb{T}}^+$ is a proper metric space (see Proposition~\ref{prop:L2-T+} $i$) makes tightness issues easily tractable.  Furthermore, the fact that $L^2$-convergence in $\overline{\mathbb{T}}^+$ to a continuous limit implies uniform convergence (see Proposition~\ref{prop:L2-T+} $ii$)  will ultimately be used to derive the uniform convergence stated in Theorem~\ref{thm:cv-lambda-to-0}.

We now consider the decomposition of gradient boosting into positive and negative parts. The Jordan decomposition $\mu=\mu^+-\mu^-$ of the signed measure $\mu\in\mathcal{M}_0$ naturally induces the following decomposition of tree functions: for $T\in\mathbb{T}$, we have
\[
T=T^+-T^-\quad \mbox{with $\mu_{T^+}=(\mu_T)^+$ and $\mu_{T^-}=(\mu_T)^-$}.
\]
The tree functions $T^+,T^-$ are respectively associated with the positive and negative part  of $\mu_T$ and hence belong to $\mathbb{T}^+$. This decomposition can be applied to the softmax gradient tree $\widetilde T(\,\cdot\,;F,\zeta)$ and we can write
\[
  \widetilde T(\,\cdot\,;F,\zeta)=\widetilde T^+(\,\cdot\,;F,\zeta)-\widetilde T^-(\,\cdot\,;F,\zeta).
\]
From now on, for the sake of lighter notation and because we see $\widetilde{T}$ as a (random) operator $\mathbb{T}\to \mathbb{T}$, we will omit the superfluous parts of our notation and simply write $\widetilde{T}(F, \zeta)$ --- and sometimes $\widetilde{T}(F)$ when identifying the randomness is not useful -- as a shorthand notation for $\widetilde{T}(\,\cdot\,;F,\zeta)$.
Therefore we now write the above decomposition simply as
\[
  \widetilde T(F)=\widetilde T^+(F)-\widetilde T^-(F).
\]

The gradient boosting sequence $(\hat F_m^\lambda)_{m\geq 0}$ is then decomposed recursively as $\hat F_0^\lambda= \hat F_0^{\lambda+}-\hat F_0^{\lambda-}$ and
\[
\left\{\begin{array}{ll}
\hat F_{m+1}^{\lambda+}&=\;\hat F_{m}^{\lambda+}+\lambda \widetilde T^+(\hat F_{m}^{\lambda},\zeta_{m+1})\\
\hat F_{m+1}^{\lambda-}&=\;\hat F_{m}^{\lambda-}+\lambda \widetilde T^-(\hat F_{m}^{\lambda},\zeta_{m+1})
\end{array}\right. ,\quad m\geq 0.
\]
Clearly, we have $\hat F_m^\lambda=\hat F_m^{\lambda +}-\hat F_m^{\lambda -}$, and $\hat F_m^{\lambda +},\hat F_m^{\lambda -}\in\mathbb{T}^+$ but it is not true in general that $\hat F_m^{\lambda +}$ and $\hat F_m^{\lambda -}$ are the positive and negative parts of $\hat F_m^{\lambda}$ because the measures associated to the two components may share some common mass (at point $0$) that vanishes when taking the difference.
This slight abuse of notation should however cause no confusion.
Finally, the infinitesimal gradient boosting operator can also be decomposed into a positive and negative part as $\mathcal{T}=\mathcal{T}^+-\mathcal{T}^-$ with
\[
\left\{\begin{array}{ll}
\mathcal{T}^+(F)&=\;\mathbb{E}_\zeta[\widetilde T^+(F,\zeta)]\\
\mathcal{T}^-(F)&=\;\mathbb{E}_\zeta[\widetilde T^-(F,\zeta)]
\end{array}\right. ,\quad F\in L^2_\mathbf{x}.
\]
Note here that the operators are defined on $L^2_\mathbf{x}$ and not $L^2$ --- this is required because we need the values $(F(x_i))_{1\leq i\leq n}$ to be well-defined in order to define the softmax gradient tree $\widetilde T(F)$. The measures associated with $\mathcal{T}^+(F)$ and $\mathcal{T}^-(F)$ may share some common mass so that the decomposition does not coincide with the Jordan decomposition.

The following theorem characterizes the convergence of $(\hat F_{[t/\lambda]}^{\lambda+},\hat F_{[t/\lambda]}^{\lambda-})_{t\geq 0}$ with deterministic limit given as the solution of a differential equation.

\begin{theorem}~ \label{thm:cv-L2x}
 \begin{enumerate}[(i)]
 \item  Given an initial condition $(F^+_0,F^-_0)\in L^2_\mathbf{x}\times L^2_\mathbf{x}$, there is a unique solution $(F^+(t),F^-(t))_{t\geq 0}$ to the ODE system 
  \begin{equation}\label{eq:ODE-pm}
    \left\{\begin{array}{l}
    \deriv{t}F^+(t) = \mathcal{T}^+(F(t))\\
    \deriv{t}F^-(t)= \mathcal{T}^-(F(t))
	\end{array}\right.\quad \mbox{with $F(t)=F^+(t)-F^-(t)$},\quad t\geq 0,
  \end{equation}
  started from $(F^+_0,F^-_0)$. \\
  Furthermore, if $(F^+_0,F^-_0)\in \mathbb{T}^+\times \mathbb{T}^+$, then $(F^+(t),F^-(t))\in \mathbb{T}^+\times \mathbb{T}^+$ for all $t\geq 0$ and $(F(t))_{t\geq 0}$ is equal to the solution of \eqref{eq:ODE} started from $F_0=F^+_0-F^-_0\in\mathbb{T}$.
  \item The convergence in distribution
\[
(\hat F_{[t/\lambda]}^{\lambda+},\hat F_{[t/\lambda]}^{\lambda-})_{t\geq 0}\stackrel{d}\longrightarrow (\hat F_t^+,\hat F_t^-)_{t\geq 0},\quad \text{as }\lambda\to 0,
\] 
holds in the Skorokhod space $\mathbb{D}([0,\infty),\overline{\mathbb{T}}^+\times \overline{\mathbb{T}}^+)$ endowed with the $J_1$-topology where  $(\hat F_t^+,\hat F_t^-)_{t\geq 0}$ is equal to the solution of \eqref{eq:ODE-pm} started from $(\hat F_0^+,\hat F_0^-)$. 
  \end{enumerate}
\end{theorem}
In point \textit{(ii)}, since the limit is time continuous, the convergence also holds with respect to the topology of uniform-in-time convergence on compact sets of $[0, \infty)$.
An important tool in the proof of Theorem~\ref{thm:cv-L2x} is a control of the training error 
\[
L_n(F)= \sum_{i=1}^n L(y_i,F(x_i)),\quad F\in\mathbb{B}.
\]
The next proposition states that for small learning rates, the training error is almost surely non-increasing along the gradient boosting sequence $(\hat F_m^\lambda)_{m\geq 0}$. This in turn yields a control of the increments of the sequence because --- as we will show in the course of the proof of Theorem~\ref{thm:cv-L2x}, see Equation~\eqref{eq:totalVarBounded} --- the total variation of the gradient tree $\norm{\widetilde T(\,\cdot\,;F,\zeta)}_{\mathrm{TV}}$ is uniformly bounded on the level set $\{F\in\mathbb{B}:L_n(F)\leq C\}$.
\begin{proposition}\label{prop:monotonicity}
  For all $T>0$, there exists $\lambda_0>0$ such that $\lambda\in (0,\lambda_0]$ implies that 
  \[
  t\mapsto L_n(\hat F_{[t/\lambda]}^\lambda) \quad \mbox{is a.s. non-increasing on $[0,T]$}.
  \] 
\end{proposition}
Note that in the case of regression $L(y,z)=\frac{1}{2}(y-z)^2$, Proposition~\ref{prop:monotonicity} trivially holds for all $T>0$ and $\lambda>0$  because the tree values are obtained by line search so that the training error cannot increase, see Example~\ref{ex:regression} and Equation~\eqref{eq:boosting-line-search}. In the general case however, the one-step approximation~\eqref{eq:boosting-line-search-approximation} does not ensure such a property but the proposition states that monotonicity still holds for small learning rates.

\subsection{Second-order variations} \label{sec:CLT}

We announced in \Cref{thm:second-order} in the introduction that second-order variations of the stochastic dynamics around the infinitesimal gradient boosting limit are of order $\sqrt{\lambda}$.
Rather than postponing the very short proof of this result, we develop it here as it will enable us to heuristically derive a more precise form of second-order variations.
\begin{proof}[Proof of \Cref{thm:second-order}]
In the course of the proof of \Cref{thm:cv-L2x} (see Equations~\eqref{eq:snd-variation-defG} and~\eqref{eq:G-close-to-M}), we show that one can write (we ignore the decomposition into positive and negative parts in this section)
\begin{equation} \label{eq:martingale-decomp}
\hat{F}^{\lambda}_{[t/\lambda]} - \hat{F}^{\lambda}_0 = M^{\lambda}_{[t/\lambda]}+ \int_0^{t}\mathcal{T}(\hat{F}^{\lambda}_s)\,\rmd s + O_{\mathbb{P}}(\lambda),
\end{equation}
where the $O_{\mathbb{P}}(\lambda)$ is uniform in $t\in[0,T]$ and $M^{\lambda}_{[t/\lambda]}$ is a square-integrable martingale in $L^2_{\mathbf{x}}$ satisfying
\[
\sup_{0\leq t\leq T}\norm{M^{\lambda}_{[t/\lambda]}}^2_{L^2_{\mathbf{x}}} = O_{\mathbb{P}}(\lambda).
\]
In the proof, this is used to show that any limiting process $(\hat{F}_t)_{t\geq 0}$ satisfies
\[
  \hat{F}_t - \hat{F}_0 = \int_0^{t}\mathcal{T}(\hat{F}^{\lambda}_s)\,\rmd s,
\]
and therefore must be the infinitesimal gradient boosting limit, defined as the solution of the ODE $\deriv{t}\hat{F}_t=\mathcal{T}(\hat{F}_t)$.
Subtracting this limit in \eqref{eq:martingale-decomp}, we get
\[
\hat{F}^{\lambda}_{[t/\lambda]} - \hat{F}_t = M^{\lambda}_{[t/\lambda]}+ \int_0^{t}\Big(\mathcal{T}(\hat{F}^{\lambda}_{[s/\lambda]})-\mathcal{T}(\hat{F}_s)\Big)\,\rmd s + O_{\mathbb{P}}(\lambda).
\]
Using the fact that for $\lambda$ small enough, $\hat{F}^\lambda_{[t/\lambda]}$ stays in a domain $E\subset L^2_{\mathbf{x}}$ where $\mathcal{T}$ is $C$-Lipschitz for some $C>0$, we can apply Grönwall's lemma to deduce
\[
  \sup_{t\in [0,T]}\norm{\hat{F}^\lambda_{[t/\lambda]}-\hat{F}_t}_{L^2_{\mathbf{x}}} \leq \Big(\sup_{t\in [0,T]}\norm{M^\lambda_{[t/\lambda]}}_{L^2_{\mathbf{x}}} + O_{\mathbb{P}}(\lambda)\Big)e^{TC} = O_{\mathbb{P}}(\sqrt{\lambda}),
\]
which concludes the proof.
\end{proof}

\paragraph{Heuristics for a CLT}

We have shown that variations around the infinitesimal gradient boosting limit are of order $\sqrt{\lambda}$.
In fact, possibly under stronger assumptions than \Cref{ass:A}, we conjecture that $\mathcal{T}$ is continuously differentiable as an operator $L^2_{\mathbf{x}}\to L^2_{\mathbf{x}}$, and that the following central limit theorem holds.

\begin{conjecture}[Central limit theorem]
  We have the following convergence
  \[
  \frac{1}{\sqrt{\lambda}} (\hat{F}^\lambda_{[t/\lambda]} - \hat{F}_t)_{t\geq 0} \;\stackrel{d}\longrightarrow \; \mathscr{F},
  \]
  where the convergence holds in distribution on $\mathbb{D}([0,\infty),L^2_{\mathbf{x}})$.
  The process $\mathscr{F} = (\mathscr{F}_t)_{t\geq 0}$ is a zero-mean continuous Gaussian process in  $L^2_{\mathbf{x}}$ characterized by the  SDE
  \[
  \mathrm{d}\mathscr{F}_t \;=\; \mathrm{d} \mathscr{G}_t + (\mathrm{d}_{\hat{F}_t}\mathcal{T})(\mathscr{F}_t) \,\mathrm{d}t,
  \]
  where $\mathscr{G}$ is a cylindrical Wiener process on $L^2_{\mathbf{x}}$ with  covariance structure
  \[
  \E \left[\langle\mathscr{G}_t,f\rangle \langle\mathscr{G}_s,g\rangle\right] \;=\;\int_0^{t\wedge s} \E\left[\langle f, \widetilde{T}(\hat{F}_u)-\mathcal{T}(\hat{F}_u)\rangle\langle g, \widetilde{T}(\hat{F}_u)-\mathcal{T}(\hat{F}_u)\rangle\right].
  \]
\end{conjecture}

The reason for this limit is the following.
First let us define
\begin{align*}
\mathscr{F}^\lambda_{t} &= \frac{1}{\sqrt{\lambda}} (\hat{F}^\lambda_{[t/\lambda]} - \hat{F}_t)\\
& = \frac{1}{\sqrt{\lambda}}M^\lambda_{[t/\lambda]} + \frac{1}{\sqrt{\lambda}}\int_0^t\big(\mathcal{T}(\hat{F}^\lambda_{[s/\lambda]})-\mathcal{T}(\hat{F}_s)\big)\,\mathrm{d}s \;+\; O_{\mathbb{P}}(\sqrt{\lambda}).
\end{align*}
Assuming that $\mathcal{T}$ is regular enough, we should have
\[
\frac{1}{\sqrt{\lambda}} \big(\mathcal{T}(\hat{F}^\lambda_{[s/\lambda]})-\mathcal{T}(\hat{F}_s)\big) \;\simeq\; \rmd\mathcal{T}_{\hat{F}_s}(\mathscr{F}^\lambda_s) + O_{\mathbb{P}}(1),
\]
so the conjecture essentially relies on checking tightness and convergence of the sequence of martingales
$M^\lambda/\sqrt{\lambda}$.
Note that
\[
\frac{1}{\sqrt{\lambda}}M^\lambda_{[t/\lambda]} = \sum_{k=0}^{[t/\lambda]} \sqrt{\lambda}\big(\widetilde{T}(\hat{F}^{\lambda}_k) - \mathcal{T}(\hat{F}^{\lambda}_k)\big),
\]
where the $(\widetilde{T}(\hat{F}^{\lambda}_k) - \mathcal{T}(\hat{F}^{\lambda}_k))$ terms are zero-mean and essentially independent with integrable square norm.
Heuristically, $M^\lambda/\sqrt{\lambda}$ should converge to a continuous martingale $M$ with a deterministic ``second Meyer process'' \citep[the infinite-dimensional equivalent of the quadratic variation, see][]{Met82} given by
\[
\langle\!\langle M \rangle\!\rangle_t = \int_0^t\E\left[\big(\widetilde{T}(\hat{F}_{[s/\lambda]})-\mathcal{T}(\hat{F}_s)\big)^{\otimes2}\right] \,\rmd s,
\]
where $h^{\otimes 2}$  denotes the bilinear form $(f,g)\mapsto\langle f, h\rangle\langle g,h\rangle$.
An informal adaptation of the Dubins--Schwarz theorem, i.e.\ that continuous martingales with deterministic quadratic variation are Wiener processes, suggests that any continuous martingale with the above second Meyer process is the Gaussian process $\mathscr{G}$ described in the conjecture.

\subsection{Properties of infinitesimal gradient boosting} \label{sec:prop-inf-gb}
We call infinitesimal gradient boosting the solution $(\hat F_t)_{t\geq 0}$ of the differential Equation~\eqref{eq:ODE} with initial condition~\eqref{eq:boosting-init}. In the following, we consider some properties of infinitesimal gradient boosting, including the behavior of training error and residuals, the asymptotic behavior as $t\to+\infty$ and also the space-time regularity.

\begin{proposition}\label{prop:properties}
Infinitesimal gradient boosting satisfies the following properties:
\begin{enumerate}[(i)]
\item the training error $L_n(\hat F_t)=\sum_{i=1}^n L(y_i,\hat F_t(x_i))$, $t\geq 0$, is non-increasing;
\item the residuals $r_{t,i}=\pderiv{z}L(y_i,\hat F_t(x_i))$, $1\leq i\leq n$, are centered, i.e. 
\[
\bar r_t:=\frac{1}{n}\sum_{i=1}^n r_{t,i}\equiv 0,\quad t\geq 0.
\]
\end{enumerate}
\end{proposition}
The first point is very natural since gradient boosting aims at minimizing the training error. The second point can be interpreted as follows:  it is not possible to reduce the training error by adding a constant term to the model; indeed,
a simple computation shows that $c\mapsto L_n(\hat F_t+c)$ has derivative $n\bar r_t=0$ at $c=0$ and, the function being convex, this corresponds to a minimum. As will be clear from the proof, this property is due to the initialization~\eqref{eq:boosting-init} and to the line search approximation~\eqref{eq:boosting-line-search-approximation}.

We next consider the long time behavior of infinitesimal gradient boosting. A few more assumptions are required for our analysis.

\begin{assumption}~ \label{ass:Abis}
\begin{enumerate}[label={$(A\arabic*)$}] \addtocounter{enumi}{2}
\item \label{ass:A3} for all $y\in\mathbb{R}$, $\inf_z L(y,z)=0$; 
\item \label{ass:A4} for all $y\in\mathbb{R}$ and all $R > 0$, $\sup_{z: |{\pderiv{z}L(y,z)}| \leq R}  \pderiv[2]{z}L(y,z) < \infty$;
\item \label{ass:A5} there is $J\subset [\![1,p]\!]$ with $\abs{J}\leq d$ such that $(x_i^J)_{1\leq i\leq n}$ are pairwise distinct, where $x_i^J=(x_i^j)_{j\in J}\in [0,1]^J$.
\end{enumerate}
\end{assumption}

Assumptions~\ref{ass:A3} and \ref{ass:A4} state conditions on the loss function $L$. It is easily checked that the classical loss functions for regression and classification satisfy those conditions, see Examples \ref{ex:regression}--\ref{ex:AdaBoost}. Furthermore,  \ref{ass:A3} can always be assumed without loss of generality  because shifting the loss function by its infimum does not affect the gradient boosting algorithm. On the other hand, Assumption~\ref{ass:A5} states a condition on the training sample $(x_i)_{1\leq i\leq n}$ and is  discussed in Remark~\ref{rk:A3} below. It is used in the following proposition characterizing the critical points of the ODE~\eqref{eq:ODE}.

\begin{proposition}\label{prop:critical-points}
Under Assumption~\ref{ass:A5},  $\mathcal{T}(F)=0$ if and only if 
\[
\pderiv{z}L(y_i,F(x_i))=0,\quad \mbox{for all $1\leq i\leq n$}.
\]
\end{proposition}

This means that critical points of the ODE~\eqref{eq:ODE} are functions for which the residuals are all null.  

\begin{remark}~\label{rk:A3}
We briefly comment upon Assumption \ref{ass:A5}. It is trivially satisfied in the following two cases:
  \begin{itemize}[-]
    \item the $(x_i)_{1\leq i\leq n}$ are pairwise distinct and $d\geq p$ (consider $J=[\![1,p]\!]$);
    \item there exists $j\in [\![1,p]\!]$ such that the $(x_i^j)_i$ are pairwise distinct (consider $J=\{j\}$).
  \end{itemize}
Next, we discuss a simple example where \ref{ass:A5} does not hold.   Consider the problem of regression with $p=2, d=1$ and $(x_i, y_i)_{1\leq i\leq 4}$ given by
  \[
    \begin{cases}
    x_1 = (\frac 13,\frac 13), \; y_1 = 1, \quad & x_2 = (\frac 23,\frac 13), \; y_2 = -1, \\
    x_3 = (\frac 13,\frac 23), \;y_3 = -1, \quad & x_4 = (\frac 23,\frac 23), \;y_4 = 1.
    \end{cases}
  \]
  For this configuration, starting from $F \equiv 0$, any randomized tree with depth $d=1$ is null because, for any vertical of horizontal split, the sum of residuals  compensate in each regions.   Therefore,  $\mathcal{T}(F)=0$ but the residuals are nonzero.   Similar examples can be built in higher dimension as long as $d < p$.
\end{remark}

The analysis of critical point suggests that, in the long-time asymptotic,   infinitesimal gradient boosting should converge to a critical point with null residuals. This is the subject of our next result.

\begin{proposition} ~ \label{prop:igb-asymptotic}
Under Assumption~\ref{ass:Abis}, we have
\begin{gather*}
\pderiv{z}L(y_i,\hat F_t(x_i))\longrightarrow 0,\quad 1\leq i\leq n,\\
\text{and}\quad L_n(\hat{F}_t)\longrightarrow 0
\end{gather*}
as $t\to\infty$.
\end{proposition}
In the specific case of regression or classification,  we get under Assumption \ref{ass:Abis}:
  \begin{itemize}
    \item   $\hat{F}_t(x_i) \to y_i$ in the case of regression (Example~\ref{ex:regression});
    \item   $\hat{F}_t(x_i) \to +\infty$ if $y_i=1$ and $\hat{F}_t(x_i) \to - \infty$ otherwise in the case of classification (Examples~\ref{ex:classification} and~\ref{ex:AdaBoost}).
  \end{itemize}
This shows  that infinitesimal gradient boosting is prone to overfitting as $t\to+\infty$ as it tries to match the training sample perfectly. This is a well-known feature of gradient boosting and early stopping is usually used to avoid overfitting and obtain good generalization properties, see \cite{ZY05}.

Finally, we study the regularity of infinitesimal gradient boosting with respect to space and time. To this aim, we define a function space $\mathbb{W}$ and first a reference measure $\pi_0\in\mathcal{M}_0$ as follows. For a splitting scheme $\xi$, we consider $\pi_\xi=\sum\delta_x\in\mathcal{M}_0$ the  point measure with a Dirac mass at each vertex in $[0,1)^p$ of the  partition $(A_v(\xi))_{v\in\{0,1\}^d}$ into hypercubes; the measure $\pi_0$ is then defined as the intensity measure of the point process $\pi_\xi$ under $P_0(\rmd \xi)$, that is
\[
\pi_0=\int \pi_\xi \, P_0(\rmd \xi)\in\mathcal{M}_0.
\]
The function space $\mathbb{W}=\mathbb{W}([0,1]^p,\mathbb{R})$ is defined by
\[
\mathbb{W}=\Big\{T\in \mathbb{T}\ :\ \mu_T\ll \pi_0 \text{ and } \frac{\rmd \mu_T}{\rmd \pi_0}\in L^\infty([0,1]^p,\pi_0)\Big\}
\]
and is endowed with the norm 
\[
\norm{T}_{W}=\norm[\big]{\frac{\rmd \mu_T}{\rmd \pi_0}}_\infty,\quad T\in\mathbb{W}.
\]
Clearly,  $\mathbb{W}\subset\mathbb{T}$ with $\norm{T}_{TV}\leq  \norm{T}_{W}$, for all $T\in\mathbb{W}$. Furthermore, $(\mathbb{W},\norm{\cdot}_W)$ is a Banach space isomorphic to $L^\infty([0,1]^p,\pi_0)$ via the map $T\mapsto \rmd \mu_T/\rmd\pi_0$. 

\begin{proposition}\label{prop:regularity}
Infinitesimal gradient boosting $(\hat F_t)_{t\geq 0}$ defines a smooth space in $\mathbb{W}$, i.e. $\hat F_t\in\mathbb{W}$ for all $t\geq 0$ and the mapping $t\mapsto  \hat F_t$ is  continuously differentiable. Furthermore, the same result holds for the positive and negative parts $(\hat F_t^+)_{t\geq 0}$ and $(\hat F_t^-)_{t\geq 0}$ defined jointly as the solution of~\eqref{eq:ODE-pm}.
\end{proposition}
As a consequence, there exists a continuously differentiable mapping $t\mapsto \phi_t\in L^\infty([0,1]^p,\pi_0)$ such that
\[
\hat F_t(x)= \int_{[0,x]} \phi_t(y) \,\pi_0(\rmd y),\quad x\in [0,1]^p, t\geq 0.
\]
Note that $\pi_0$ is absolutely continuous with respect to $\nu$, because all possible vertices of a partition corresponding to a splitting scheme $\xi$ with distribution $P_0$ have coordinates that are either null, or of the form $\prod_{i=1}^N U_i$, where $N$ is random and the $U_i$ are i.i.d.\ uniform in $[0,1]$. As a consequence, Proposition~\ref{prop:regularity} implies that  $(t,x)\mapsto \hat F_t(x)$ is jointly continuous on $[0,\infty)\times [0,1]^p$ and the same holds for $\hat F_t^+(x)$ and $\hat F_t^-(x)$.

Surprisingly, infinitesimal gradient boosting is regular not only in time but also in space. Recall that $(\hat F_{t})_{t\geq 0}$ appears as the limit of gradient boosting $(\hat F_{[t/\lambda]}^\lambda)_{t\geq 0}$ which is highly discontinuous. The limit $\lambda\to 0$ as a regularizing effect with respect to time and space. The temporal regularizing effect is quite natural since the jump size $\lambda$ vanishes in the limit. The spatial regularizing effect is due to the randomization of softmax gradient trees: due to split randomization, the discontinuities in the different gradient trees occur  at different places and an averaging effect provides a regular function in the limit. This heuristic reasoning is made rigorous in Lemma~\ref{lem:regularity}  where we state that $\mathcal{T}(F)\in \mathbb{W}$ for all $F\in\mathbb{B}$, evidencing the spatial regularization effect of the infinitesimal gradient boosting operator.

\section{Proofs} \label{sec:proofs}
\subsection{Proofs related to Section~\ref{sec:formal-softmax-trees}} \label{sec:proofs-reg-trees}
\begin{proof}[Proof of Proposition~\ref{prop:RN-derivative}]Let $A\subset ([\![1,p]\!]\times (0,1))^{\mathscr{T}_{d-1}}$ be measurable and consider the event $\{\xi\in A\}$. Modulo null sets, it can be decomposed into the disjoint union
\[
\{\xi\in A\}=\bigcup_{\phi}\{\xi\in A,\; \xi=(j_v^{\phi(v)},u_v^{\phi(v)})_{v\in\mathscr{T}_{d-1}}\}
\]
with the union taken over all selection maps  $\phi:\mathscr{T}_{d-1}\to [\![1,K]\!]$. We deduce
\[
\mathbb{P}(\xi\in A)=\sum_{\phi}\int\mathbb{P}(\xi\in A,\; \xi=(j_v^{\phi(v)},u_v^{\phi(v)})_{v\in\mathscr{T}_{d-1}} \mid \xi^1,\ldots,\xi^K)\otimes_{k=1}^K P_0(\rmd \xi^k).
\]
By invariance of the product measure, all the terms in the sum are equal to the term corresponding to $\phi\equiv 1$ and there are $K^{2^d-1}$ such terms, with $2^d-1$ the cardinal of $\mathscr{T}_{d-1}$. We deduce
\begin{align*}
&\mathbb{P}(\xi\in A)\\
&=K^{2^d-1} \int \mathbb{P}(\xi\in A,\; \xi=\xi^1 \mid \xi^1,\ldots,\xi^K) \otimes_{k=1}^K P_0(\rmd \xi^k)\\
&= K^{2^d-1}\int \1_{\{\xi^1\in A\}} \prod_{v\in \mathscr{T}_{d-1}} \big(\mathrm{softmax}_\beta( \Delta(j_v^{k},u_v^{k}; A_v(\xi^1))_{1\leq k\leq K})\big)_{1} \otimes_{k=1}^K P_0(\rmd \xi^k)\\
&= \int_A\Big(\int \prod_{v\in \mathscr{T}_{d-1}}  \frac{\exp(\beta \Score(j_v^1,u_v^1; A_v(\xi^1)))}{K^{-1}\sum_{k=1}^K \exp(\beta \Score(j_v^{k},u_v^{k}; A_v(\xi^1)))} \otimes_{k=2}^K P_0(\rmd \xi^k)\Big)P_0(\rmd \xi^1),
\end{align*}
where the second equality relies on  Equation~\eqref{eq:law-of-xi} and the third on the definition of the softmax function. This characterizes the distribution $P_{\beta,K}$ of $\xi$ and proves that the Radon-Nykodym derivative is given by Equation~\eqref{eq:RN}.
\end{proof}

\begin{proof}[Proof of Proposition~\ref{prop:reg-tree-Lipschitz}]
Consider input $\mathbf{r}=(r_i)_{1\leq i\leq n}$ and $\mathbf{r}'=(r_i')_{1\leq i\leq n}$ and let $(x_i)_{1\leq i\leq n}$ be fixed. We may replace without risk of confusion $(x_i,r_i)_{1\leq i\leq n}$ and  $(x_i,r_i')_{1\leq i\leq n}$ by $\mathbf{r}$ and $\mathbf{r}'$ respectively in the notation below. For a fixed (deterministic) splitting scheme $\xi=(j_v,u_v)_{v\in\mathscr{T}_{d-1}}$, we compare $T(\;\cdot\;;\mathbf{r},\xi)$ and $T(\;\cdot\;;\mathbf{r}',\xi)$ --- we will actually remove the ``$\,\cdot\,$'' parts of these expressions for the rest of the proof. Denoting by $(A_v)_{v\in\{0,1\}^d}$ the partition associated to $\xi$, we have 
\[
T(\;\cdot\;;\mathbf{r},\xi) = \sum_{v\in \{0,1\}^d }r(A_v)\1_{A_v}\quad\mbox{and}\quad  T(\;\cdot\;;\mathbf{r}',\xi)=\sum_{v\in \{0,1\}^d }r'(A_v)\1_{A_v}, 
\]
with $r(A_v)$ (resp. $r'(A_v)$) the mean of the values $r_i$ (resp. $r'_i$) with $x_i\in A_v$. Since  
 \[
|r(A_v)-r'(A_v)|\leq \|\mathbf{r}-\mathbf{r}'\|_{\infty},
\]
we deduce
\begin{equation}\label{eq:reg-tree-lipschitz-1}
\| T(\;\cdot\;;\mathbf{r},\xi)-T(\;\cdot\;;\mathbf{r}',\xi)\|_\infty \leq \|\mathbf{r}-\mathbf{r}'\|_{\infty}.
\end{equation} 

Next, we consider a random splitting scheme $\xi$. It should be stressed that the splitting scheme distribution depends on the input (respectively $\mathbf{r}$ and $\mathbf{r}'$) and we denote by $P_{\beta,K}$ and $P_{\beta,K}'$ the respective splitting scheme distributions. Proposition~\ref{prop:RN-derivative} implies
\[
\bar T_{\beta,K,d}(x;\mathbf{r})=  \int T(x;\mathbf{r},\xi)\frac{\rmd P_{\beta,K}}{\rmd P_{0}}(\xi)P_{0}(\rmd \xi)
\]
and similarly for $\bar T_{\beta,K,d}(x;\mathbf{r}')$ with $\mathbf{r}$ and  $\rmd P_{\beta,K}$ replaced by $\mathbf{r}'$ and $\rmd P_{\beta,K}'$ respectively. We deduce 
\begin{align}
\big|\bar T_{\beta,K,d}(x;\mathbf{r})-\bar T_{\beta,K,d}(x;\mathbf{r}')\big|&\leq \int \big|T(x;\mathbf{r},\xi)-T(x;\mathbf{r}',\xi)\big|\; \frac{\rmd P_{\beta,K}}{\rmd P_{0}}(\xi)P_{0}(\rmd \xi)\nonumber\\
&\quad + \int |T(x;\mathbf{r}',\xi)| \; \Big|\frac{\rmd P_{\beta,K}}{\rmd P_{0}}(\xi)-\frac{\rmd P_{\beta,K}'}{\rmd P_{0}}(\xi)\Big| P_{0}(\rmd \xi).\label{eq:reg-tree-lipschitz-2}
\end{align}
The first term is bounded from above by $\|\mathbf{r}-\mathbf{r}'\|_{\infty}$ thanks to Equation~\eqref{eq:reg-tree-lipschitz-1}. For the second term, we use  $|T(x;\mathbf{r}',\xi)|\leq \|\mathbf{r}'\|_\infty $ and the following Lipschitz property of the Radon-Nykodym derivative, justified below,
\begin{equation}
\sup_{\xi}\big|\frac{\rmd P_{\beta,K}}{\rmd P_{0}}(\xi)-\frac{\rmd P_{\beta,K}'}{\rmd P_{0}}(\xi)\big|
\leq CR \|\mathbf{r}-\mathbf{r}'\|_{\infty}\label{eq:reg-tree-lipschitz-3}
\end{equation}
for $\max(\|\mathbf{r}\|_\infty,\|\mathbf{r}'\|_\infty)\leq R$ and  $C= \beta (2^d-1)K^{2^d-1}$. With these bound, Equation~\eqref{eq:reg-tree-lipschitz-2} implies 
\[
\big\|\bar T_{\beta,K,d}(\;\cdot\;;\mathbf{r})-\bar T_{\beta,K,d}(\;\cdot\;;\mathbf{r}')\big\|_\infty\leq (1+CR^2)\|\mathbf{r}-\mathbf{r}'\|_\infty,
\]
proving that $\mathbf{r}\mapsto \bar T_{\beta,K,d}(\;\cdot\;;\mathbf{r})$ is locally Lipschitz.

We finally prove Equation~\eqref{eq:reg-tree-lipschitz-3} using the explicit formula~\eqref{eq:RN} for the Radon-Nykodym derivative. In definition~\eqref{eq:def-score}, the score $\Delta(j,u;A)$ implicitly depends on the input $\mathbf{r}$ so for clarity we write below $\Delta(j,u;A,\mathbf{r})$. Simple computations yield
\[
|\Delta(j,u;A,\mathbf{r})-\Delta(j,u;A,\mathbf{r}')|\leq 2R \|\mathbf{r}-\mathbf{r}'\|_\infty
\]
for $\max(\|\mathbf{r}\|_\infty,\|\mathbf{r}'\|_\infty)\leq R$ and all $(j,u)$ and $A$. Furthermore, it is elementary to see that the softmax function $z\mapsto e^{z_1}/\sum_{k=1}^K e^{z_k}$ is $1/2$-Lipschitz for the uniform norm on $\mathbb{R}^K$. Therefore, for fixed $\xi,\xi_1,\xi_2,\ldots,\xi_K$ and $v\in\mathscr{T}_{d-1}$, the different factors in Equation~\eqref{eq:RN} satisfy
\[
\Big|\frac{\exp(\beta \Score(j_v^1,u_v^1; A_v(\xi),\mathbf{r}))}{\sum_{k=1}^K \exp(\beta \Score(j_v^{k},u_v^{k}; A_v(\xi),\mathbf{r}))} -\frac{\exp(\beta \Score(j_v^1,u_v^1; A_v(\xi),\mathbf{r}'))}{\sum_{k=1}^K \exp(\beta \Score(j_v^{k},u_v^{k}; A_v(\xi),\mathbf{r}'))}\Big|\\
\leq  \beta R\; \|\mathbf{r}-\mathbf{r}'\|_\infty.
\]
Using the inequality $\big|\Pi_{i\in I} a_i-\Pi_{i\in I} b_i\big| \leq \sum_{i\in I}|a_i-b_i|$ for finite families $(a_i),(b_i)\in  [-1,1]^I$ with $I=\mathscr{T}_{d-1}$ of cardinal $2^d-1$ and integrating with respect to $\delta_\xi(\rmd \xi^1)\otimes_{k=2}^K P_0(\rmd \xi^k)$, we deduce Equation~\eqref{eq:reg-tree-lipschitz-3}. 
\end{proof}

\subsection{Proofs related to Section~\ref{sec:T-space}} \label{sec:proofs-T-space}
\begin{proof}[Proof of Proposition~\ref{prop:reg-tree-measure}]
 Consider $a,b\in[0,1]^p$ with $0\leq a^i<b^i\leq 1$ for all $1\leq i\leq p$. Using the notation of \cite{Neu71}, we consider the map $S=\1_{[a,b\rangle}$, where
  \[
    [a,b\rangle := I_1 \times I_2 \times \dots \times I_p, \qquad \text{with}\qquad I_i=
    \begin{cases}
      [a^i,b^i) & \text{if } b^i < 1\\
      [a^i,b^i] & \text{if } b^i = 1.
    \end{cases}
  \]
  Let us show that  $S\in \mathbb{T}$ and characterize the measure $\mu_S$. For each vertex $\epsilon=(\epsilon_1,\dots,\epsilon_p)\in \{0,1\}^p$, we denote by $|\epsilon|=\sum_{i=1}^p \epsilon_i$ the number of positive components. The vertices of the hypercube $[a,b]$ are $a+\epsilon\cdot(b-a)$   where $\cdot$ denotes coordinate-wise multiplication. Then the signed measure $\mu_S$ 
  \begin{equation} \label{eq:mu-atom-decomposition}
    \mu_S = \sum_{\epsilon\in\{0,1\}^p} (-1)^{|\epsilon|}\1_{[0,1)^p}(a+\epsilon\cdot(b-a)) \delta_{a+\epsilon\cdot(b-a)}
  \end{equation}
belongs to $\mathcal{M}_0$ and is such that $\mu_S([0,x])=S(x)$ for all $x\in[0,1]^p$. This proves  that $\1_{[a,b\rangle}=S\in\mathbb{T}$. Furthermore, we see from \eqref{eq:mu-atom-decomposition} that $\norm{\1_{[a,b\rangle}}_{\mathrm{TV}}=2^{q}\leq 2^p$, with $q$ the number of coordinates $i$ such that $b^i<1$.
  
  Now recall the recursive construction of the regions $(A_v(\xi))_{v\in \mathscr{T}_d}$ from the splitting scheme $\xi$, described in Section~\ref{sec:formal-softmax-trees}.
  For each leaf $v\in\{0,1\}^d$, the region $A_v(\xi)$ is defined as the intersection of  $d$ subsets of $[0,1]^p$ of the form $[a,b\rangle$, where there is a unique $i\in [\![1,p]\!]$ such that $a^i >0$ or $b^i <1$.   Therefore $A_v$ is necessarily of the form $[a,b\rangle$, where there are  $q\leq \min(p,d)$ coordinates $i\in [\![1,p]\!]$ such that $a^i >0$ or $b^i <1$.   From the discussion above, we deduce the  bound
  \[
  \norm{T}_{\mathrm{TV}} \;\leq\; \sum_{v \in\{0,1\}^d}\abs{\tilde r_v}\,\norm{\1_{A_v(\xi)}}_{\mathrm{TV}} \;\leq\; 2^{d+\min(p,d)}\norm{T}_{\infty}.
  \]
In the case $d<p$, each region $A_v(\xi)$ is of the form $[a,b\rangle$ where the set $J$ of coordinates $j\in [\![1,p]\!]$ such that $a^j >0$ or $b^j <1$ has cardinal $q\leq d<p$. If $\epsilon\in\{0,1\}^p$ is such that $a+\epsilon \cdot(b-a)\in[0,1)^p$ and $j\notin J$, then $(a+\epsilon\cdot(b-a))^j = a^j = 0$.   This shows that the signed measure $\mu_S$ associated with $S=\1_{[a,b\rangle}$ and defined in~\eqref{eq:mu-atom-decomposition} has its support in
  \[
    \mathcal{F}_{J} = \{x\in [0,1)^p\ :\ x^j=0\text{ for all }j\notin J\} \subset \mathcal{F}_d.
  \]
  Since $T$ is a linear combination of such functions, the proof is complete.
\end{proof}

\begin{proof}[Proof of Proposition~\ref{prop:L2-cv-on-T}]
\textit{(i)} The density results from standard approximation arguments, for instance, of continuous functions, which are dense in $L^2$, by step functions.

\textit{(ii)} Since $L^2$ is a metric space, it is sufficient to consider a sequence $T_1,T_2,\ldots\in \mathbb{T}$ that is bounded in total variation, and show there exists a convergent subsequence.
Because $(\mu_{T_n})_{n\geq 1}$ is a sequence of signed measures  that is supported on the compact space $[0,1]^p$ and bounded in total variation, we can extract a weakly convergent subsequence.
In fact, we can apply this argument to the sequences $(\mu^J_{T_n})_{n\geq 1}$ and assume that 
$\mu_{T_n}^J \,\weakcv \,\mu^{J}$ along some subsequence, jointly for all $J\subset[\![1,p]\!]$, where the $\mu^{J}$ are finite signed measures supported on $\overline{\mathcal{F}}_{J}\subset [0,1]^p$ --- note that considering closures is necessary.

We aim at proving $\nu$-a.e.\ convergence of the sequence of functions $(T_n)_{n\geq 1}$ --- this implies $L^2$ convergence by the dominated convergence theorem because $\nu$ is finite and the $T_n$ are bounded. More precisely, we prove that
\begin{equation} \label{eq:limit-not-in-T}
  T_n(x)\,\underset{n\to\infty}{\tol}\, T(x) := \sum_{J\subset [\![1,p]\!]}\1_{\mathcal{F}_{\geq J}}(x) \mu^J([0,x])\quad\mbox{$\nu$-a.e.},
\end{equation}
where $\mathcal{F}_{\geq J} = \{x\in[0,1]^p\ :\ \forall j\in J, \,x^j > 0\}$. Note that of course the limit is not necessarily in $\mathbb{T}$.

First consider a  point $a\in [0,1]^p$ and some $J\subset [\![1,p]\!]$.
Viewing $\mu_{T_n}^J$ and $\mu^{J}$ as measures on the space $\overline{\mathcal{F}}_J$, by the  Portmanteau theorem, the weak convergence $\mu_{T_n}^J \,\weakcv \,\mu^{J}$ implies
\[
\mu_{T_n}^J([0,a]) \,\underset{n\to\infty}{\tol} \, \mu^{J}([0,a])
\]
if $\mu^J(\partial([0,a^J])) = 0$, with $a^J=(a^j\1_{j\in J})_j$ and where the boundary in the previous expression is the topological boundary within the space $\overline{\mathcal{F}}_{J}$ --- \emph{not within $[0,1]^p$}.
Notice that
\begin{itemize}
  \item if $a\notin \mathcal{F}_{\geq J}$, that is if there exists $j\in J$ such that $a^j=0$, then $[0,a]\cap \mathcal{F}_{J}=\emptyset$, so that $\mu_{T_n}^J([0,a])=0$ for all $n\geq 1$;
  \item we have the inclusion
  \[
  \partial([0,a^J]) \subset \bigcup_{\substack{j\in J\\0\leq a^j<1}} \{x\in[0,1]^p\ :\ x^j = a^j\}.
  \]
\end{itemize}
Now let us define
\[
A = \{0,1\}\cup\Big\{t\in[0,1]\ :\ \sum_{J\subset[\![1,p]\!]}\sum_{j=1}^p\mu^J(\{x\in[0,1]^p\ :\ x^j = t\}) = 0\Big\}.
\]
The set $A$ has a  countable complement in $[0,1]$ and therefore has  Lebesgue measure~$1$.
Since $0,1\in A$, we have $\nu(A^p)=1$ as well.
Finally consider $a\in A^p$ and let us show that the sequence $(T_n(a))_{n\geq 1}$ converges to $T(a)$ defined in \eqref{eq:limit-not-in-T}.
It is sufficient to show that $\mu_{T_n}^J([0,a])$ converges to $\1_{\mathcal{F}_{\geq J}}(a)\mu^J([0,a])$ for any $J\subset [\![1,p]\!]$. From the discussion above, this is obvious whenever $a\notin \mathcal{F}_{\geq J}$.
If this is not the case, we have
\[
  \partial([0,a^J]) \subset \bigcup_{\substack{j\in J\\0 <a^j<1}} \{x\in[0,1]^p\ :\ x^j = a^j\},
\]
which is a null $\mu^{J}$-measure set by construction of $A$, implying
\[
\mu_{T_n}^J([0,a]) \,\tol \, \mu^J([0,a]).
\]
As this is true for all $J\subset[\![1,p]\!]$, we conclude that
\[
T_n(a) \,\underset{n\to\infty}{\tol} \, T(a),\qquad \mbox{for all $a\in A^p$},
\]
therefore $T_n\to T$ in $L^2$, completing the proof.

\textit{(iii)} The sequence $(T_n)_{n\geq 1}$ being bounded in total variation, it is relatively compact in $L^2$ by \textit{(ii)}.
Using the identification of the possible adherence points~\eqref{eq:limit-not-in-T} of this sequence in the proof of the previous point, \textit{(iii)} is easily deduced.
\end{proof}

\begin{proof}[Proof of Proposition~\ref{prop:L2-T+}]
\textit{(i)}
To show that $\overline{\mathbb{T}}^+$ is a proper metric space, it is sufficient to show that any bounded sequence $T_1,T_2,\dots \in \mathbb{T}^+$ is relatively compact in $L^2$.

The fact that the sequence $(T_n)_{n\geq 1}$ is bounded in $L^2$ implies that the real sequence $(T_n(1))_{\geq 1}$ is bounded because the reference measure $\nu$ includes a Dirac mass at $1$. Since  $T_n\in\mathbb{T}^+$, the equality $T_n(1) = \norm{T_n}_{\mathrm{TV}}$ holds. Therefore the sequence $(T_n)$ is bounded in total variation so that relative compactness is deduced from point \textit{(ii)} in Proposition~\ref{prop:L2-cv-on-T}.

\textit{(ii)} This is a straightforward consequence of the  multidimensional version of Pólya's uniform convergence theorem for distribution functions --- for a general form of this convergence theorem, see \citet[Theorem~2]{BT67}.
\end{proof}

\subsection{Proofs related to Section~\ref{sec:inf-gb}} \label{sec:proofsMain}
\subsubsection{Proofs related to Section~\ref{sec:ode}} \label{sec:proofsODE}

\begin{proof}[Proof of Proposition~\ref{prop:boosting-operator-Lipschitz}]
From Definition~\ref{ass:A}, the function $\mathcal{T}(F)$ is defined pointwise by  $\mathcal{T}(F)(x)=\mathbb{E}_\zeta[\widetilde T(x;F,\zeta)]$ for all $x\in[0,1]^p$. By Proposition~\ref{prop:reg-tree-measure}, $\widetilde T(\,\cdot\,;F,\zeta)\in\mathbb{T}$ for all $\zeta$. Note that it is not straightforward to define $\mathcal{T}(F)$ as an expectation in the Banach space $(\mathbb{T},\norm{\cdot}_{TV})$ because of separability and measurability issues. Alternatively, we can take the expectation of the random measure $\mu_{\widetilde T(\,\cdot\,;F,\zeta)}$ and define $\mu = \mathbb{E}_\zeta[ \mu_{\widetilde T(\,\cdot\,;F,\zeta)}]\in\mathcal{M}_0$. It is then straightforward to see that $\mathcal{T}(F)(x)=\mu([0,x])$ for all $x\in[0,1]^p$ so that $\mathcal{T}(F)\in\mathbb{T}$ with 
\begin{equation}\label{eq:gradient-tree-lipschitz-1}
\mu_{\mathcal{T}(F)}=\mathbb{E}_\zeta[ \mu_{\widetilde T(\,\cdot\,;F,\zeta)}].
\end{equation}

We next prove that $\mathcal{T}:\mathbb{B}\to\mathbb{T}$ is locally Lipschitz. Assumption~\ref{ass:A1} plays here a crucial role  because it implies the following property: writing $\mathbf{z}=(F(x_i))_{1\leq i\leq n}\in\mathbb{R}^n$, $\mathbf{r}=(r_i)_{1\leq i\leq n}$ with $r_i=\pderiv{z}{L}(x_i,F(x_i))$ and
\[
\tilde r(A)=-\frac{\sum_{i=1}^n \frac{\partial L}{\partial z}(y_i,F(x_i))\1_{\{x_i\in A\}} }{\sum_{i=1}^n \frac{\partial^2 L}{\partial z^2}(y_i,F(x_i))\1_{\{x_i\in A\}}},\quad A\subset [0,1]^p,
\]
the maps $\mathbf{z}\mapsto \mathbf{r}$ and $\mathbf{z}\mapsto \tilde r(A)$ are locally Lipschitz  for the uniform norm on $\mathbb{R}^n$.
Indeed, by Assumption~\ref{ass:A1}, $L$ is $C^2$ with positive and Lipschitz-continuous second derivative.
Furthermore, because only finitely many different functions $\mathbf{z}\mapsto \tilde r(A)$ can be generated for different regions $A$, the Lipschitz constant can be assumed independent of $A$. Finally, since $F\mapsto \mathbf{z}$ is linear with operator norm $1$, the maps  $F\mapsto \mathbf{r}$ and  $F\mapsto \tilde r(A)$ are locally Lipschitz functions on  $\mathbb{B}$.

As in Remark~\ref{rk:zeta-xi}, we shall consider that the splitting scheme $\xi$ depends on the auxiliary randomness $\zeta$ and let $\widetilde T(\,\cdot\,;F,\xi)$ denote the gradient tree with splitting scheme $\xi$.  Let us recall that the distribution  $P_{\beta,K}$ of $\xi$ depends (implicitly) on on the vector of residuals $\mathbf{r}$ and that the splitting scheme $\xi$ induces a  partition $(A_v)_{v\in\{0,1\}^d}$. Then the leaf values of the gradient trees are given by $(\tilde r(A_v))_{v\in\{-1,1\}^d}$ ---  not by $(r(A_v)_{v\in\{-1,1\}^d})$ as in the case of regression trees. The gradient tree can thus be written
\[
  \widetilde{T}(\;\cdot\;;F,\xi) = \sum_{v\in \{0,1\}^d}\tilde{r}(A_v)\1_{A_v}.
\]
As in the proof of Proposition~\ref{prop:reg-tree-Lipschitz}, we first analyze the case of  a fixed splitting scheme $\xi$ inducing a fixed partition $(A_v)_{v\in\{0,1\}^d}$. Since the maps  $F\mapsto \tilde{r}(A_v)$, $v\in\{0,1\}^d$, are locally Lipschitz,  there exists, for any $R\geq 0$, a constant $C\geq 0$  that does not depend on the splitting scheme $\xi$ and such that
\[
  \norm{\widetilde{T}(\,\cdot\,;F,\xi) - \widetilde{T}(\,\cdot\,;F',\xi)}_{\infty} \,\leq\, C\norm{F-F'}_{\infty},
\]
 for all $F,F'\in \mathbb{B}$ with $\max(\norm{F}_{\infty},\norm{F'}_{\infty})\leq R$.
Since $\xi$ is induced by a $d$-depth splitting scheme,  Proposition~\ref{prop:reg-tree-measure} implies 
\begin{align}
  \norm{\mu_{\widetilde{T}(\,\cdot\,;F,\xi)} - \mu_{\widetilde{T}(\,\cdot\,;F',\xi)}}_{\mathrm{TV}} &\leq 4^d\norm{\widetilde{T}(\,\cdot\,;F,\xi) - \widetilde{T}(\,\cdot\,;F',\xi)}_{\infty}\nonumber\\
  &\leq  4^d C\norm{F-F'}_{\infty}\label{eq:gradient-tree-lipschtiz-2}.
\end{align}

Next, we compare $\mathcal{T}(F)$ and $\mathcal{T}(F')$  for $F,F'\in\mathbb{B}$. With similar notation  as in the proof of  Proposition~\ref{prop:reg-tree-Lipschitz}, we denote by $\mathbf{r}=(r_i)_{1\leq i\leq n}$ (resp. $\mathbf{r}'=(r'_i)_{1\leq i\leq n}$) the residuals  and by $P_{\beta,K}$ (resp. $P_{\beta,K}'$) the splitting scheme distribution associated with $F$ (resp. $F'$).  Equation~\eqref{eq:gradient-tree-lipschitz-1} can be rewritten as
\[
\mu_{\mathcal{T}(F)}=\int \mu_{\widetilde T(\,\cdot\,;F,\xi)} \frac{\rmd P_{\beta,K}}{\rmd P_{0}}\rmd P_{0}.
\]
The same equation holds for $\mu_{\mathcal{T}(F')}$  with $F$ replaced by $F'$ and $P_{\beta,K}$ by $P_{\beta,K}'$. We deduce, as in Equation~\eqref{eq:reg-tree-lipschitz-2},
\begin{align*}
\norm{\mathcal{T}(F)-\mathcal{T}(F')}_{\mathrm{TV}}&= \norm{\mu_{\mathcal{T}(F)}-\mu_{\mathcal{T}(F')}}_{\mathrm{TV}}\\
&\leq \int  \norm{\mu_{\widetilde{T}(\,\cdot\,;F,\xi)} - \mu_{\widetilde{T}(\,\cdot\,;F',\xi)}}_{\mathrm{TV}}\; \frac{\rmd P_{\beta,K}}{\rmd P_{0}}(\xi)P_{0}(\rmd \xi)\nonumber\\
&\quad + \int  \norm{\mu_{\widetilde{T}(\,\cdot\,;F',\xi)}}_{\mathrm{TV}} \; \Big|\frac{\rmd P_{\beta,K}}{\rmd P_{0}}(\xi)-\frac{\rmd P_{\beta,K}'}{\rmd P_{0}}(\xi)\Big| P_{0}(\rmd \xi).
\end{align*}
The end of the proof is then similar to the end of the proof of Proposition~\ref{prop:reg-tree-Lipschitz}: the first term is bounded from above by Equation~\eqref{eq:gradient-tree-lipschtiz-2} and the second one by Equation~\eqref{eq:reg-tree-lipschitz-3} together with the fact that the vector of residuals $\mathbf{r}$ is locally Lipschitz in $F$.
\end{proof}

The following lemma is a simple consequence of Assumption~\ref{ass:A2} and is crucial for controlling the total variation of gradient trees, which is needed in the proof of Theorem~\ref{thm:EDO-T}.
\begin{lemma}\label{lem:control-increment}
  Under Assumption~\ref{ass:A}, for $C>0$, define
  \[
  M(C)=\sup_{(y,z):L(y,z)\leq C} \left| \frac{\partial L}{\partial z}(y,z)\;\big/\; \frac{\partial^2 L}{\partial z^2}(y,z)\right| <\infty.
  \]
  Let $F\in\mathbb{T}$ be such that $L_n(F)=\sum_{i=1}^n L(y_i, F(x_i))\leq C$. Then
  \begin{equation} \label{eq:lem-bound}
  \|\widetilde T(\,\cdot\,;F,\zeta)\|_{\infty}\leq M(C),\qquad \|\widetilde T(\,\cdot\,;F,\zeta)\|_{\mathrm{TV}}\leq 4^d M(C),
  \end{equation}
  and consequently $\norm{\mathcal{T}(F)}_{\mathrm{TV}}\leq 4^dM(C)$.
\end{lemma}

\begin{proof}
  Because the loss function is non negative, the inequality $L_n(F)\leq C$ yields $L(y_i,F(x_i))\leq C$ for all $i=1,\ldots,n$ and Assumption~\ref{ass:A2} implies
  \[
  \abs{R_i}\,:=\,\left| \frac{\pderiv{z}L(y_i,F(x_i))}{ \pderiv[2]{z}L(y_i,F(x_i))}\right|\leq M(C),\quad i=1,\ldots,n.
  \]
  Let $(A_v)_{v\in\{0,1\}^d}$ be the leaves of $\widetilde T(\,\cdot\,;F,\xi)$ so that
  \[
  \widetilde T(\;\cdot\;;F,\xi)=\sum_{v\in\{0,1\}^d} \tilde r(A_v)\mathds{1}_{A_j}
  \]
  with 
  \[
  \tilde r(A_j)=-\frac{\sum_{i:x_i\in A_j}\pderiv{z}L(y_i,F(x_i))}{\sum_{i:x_i\in A_j}\pderiv[2]{z}L(y_i,F(x_i))}.
  \]
  We finally observe that
  \[
  |\tilde r(A_v)|\leq \frac{\sum_{i:x_i\in A_v}\pderiv[2]{z} L(y_i,F(x_i))\left|R_i\right|}{\sum_{i:x_i\in A_v}\pderiv[2]{z} L(y_i,F(x_i))}\leq M(C)
  \]
  because we recognize a weighted average of the $\abs{R_i}$ terms, which are bounded from above by $M(C)$. This implies the bound for the uniform norm in Equation~\eqref{eq:lem-bound}, and the bound in total variation follows from Proposition~\ref{prop:reg-tree-measure}.
  The last bound is deduced from \eqref{eq:lem-bound} and the fact that we have $\norm{\mathcal{T}(F)}_{\mathrm{TV}} \leq \E[\norm{\widetilde{T}(\,\cdot\,;F,\zeta)}_{\mathrm{TV}}]$.
\end{proof}

\begin{proof}[Proof of Theorem~\ref{thm:EDO-T}]
Because the operator $\mathcal{T}$ is Lipschitz as proven in Proposition~\ref{prop:boosting-operator-Lipschitz}, the Cauchy--Lipschitz theorem  implies the \textit{local} existence and uniqueness of solutions for Equation~\eqref{eq:ODE}. It follows that, with the initial condition $F(0)=F_0$, there exists  a unique local solution of \eqref{eq:ODE} and that this solution can be extended uniquely into a maximal solution, still noted $F$, defined on a maximal interval $[0,t^\ast)$.  We need to prove that $t^\ast=+\infty$. For the sake of clarity, we write $F_t$ for the solution at time $t$ and $F_t(x)$ for its evaluation at $x$. We first prove that the function 
\[
t\mapsto L_n(F_t)=\sum_{i=1}^n L(y_i, F_t(x_i))
\]
is non-increasing on $[0, t^\ast)$. Indeed, its derivative is equal to
\begin{align}
\deriv{t}L_n(F_t)&=\sum_{i=1}^n \big(\deriv{t}F_t(x_i)\big)\,\frac{\partial L}{\partial z}(y_i, F_t(x_i))\nonumber\\
&= \sum_{i=1}^n \mathcal{T}(F_t)(x_i)\, \frac{\partial L}{\partial z}(y_i, F_t(x_i))\nonumber\\
&= - \mathbb{E}_\zeta\left[ \sum_{v\in\{0,1\}^d}\frac{\left(\sum_{i=1}^n \frac{\partial L}{\partial z}(y_i, F_t(x_i))\1_{A_v}(x_i)\right)^2 }{\sum_{i=1}^n \frac{\partial^2 L}{\partial z^2}(y_i,F_t(x_i))\1_{A_v}(x_i)}\right] \label{eq:negative-derivative}\\
&\leq 0. \nonumber
\end{align}
The second equality uses the fact that $(F_t)_{t\in [0,t^\ast)}$ is a solution of \eqref{eq:ODE}. The third equality uses Definitions~\ref{def:softmax-gradient-tree} and~\ref{def:boosting-operator} where the partition $(A_v)_{v\in\{0,1\}^d}$ depends on the external randomness $\zeta$. The derivative being non-positive, we have proved that $t\mapsto L_n(F_t)$ is non-increasing on $[0,t^\ast)$. 

Consider the level set 
\[
\Lambda:=\{G\in\mathbb{T}:L_n(G)\leq L_n(F_0)\}.
\]
We have just proved that $t\mapsto L_n(F_t)$ is non-increasing  so that  $F_t\in\Lambda$ for all $t\in[0,t^\ast)$.
Therefore, by Lemma~\ref{lem:control-increment}, there exists a constant $C>0$ such that $\norm{\mathcal{T}(F_t)}_{\mathrm{TV}} \leq C$ for all $t\in[0,t^\ast)$.
We deduce that for all $0\leq t\leq s < t^\ast$,
\[
\|F_s-F_t\|_{\mathrm{TV}}=\left\|\int_t^s \mathcal{T}(F_u)\,\rmd u\right\|_{\mathrm{TV}} \leq C(s-t).
\]
If $t^\ast$ were finite, this would imply that $(F_t)$ is Cauchy as $t\to t^\ast$, contradicting the maximality assumption of the solution $(F(t))_{t\in[0,t^\ast)}$.
We deduce that $t^\ast=+\infty$ and that the total variation norm of $F_t$ increases at most linearly.
\end{proof}

\subsubsection{Proofs related to Section~\ref{sec:CV} and proofs of Theorem~\ref{thm:cv-lambda-to-0} and~\ref{thm:EDO}} \label{sec:proofsCV}

The following Lemma is crucial for the proof of Proposition~\ref{prop:monotonicity} below.
\begin{lemma}\label{lem:lambda-K}
  Let $L$ satisfy Assumption~\ref{ass:A} and consider $K_1$ (resp.\ $K_2$) a finite (resp.\ compact) subset of $\mathbb{R}$.
  Then there exists $\lambda_0>0$ such that for all $\lambda\in (0,\lambda_0]$, $k\geq 1$, $y_1,\dots,y_k\in K_1$ and $z_1,\dots,z_k\in K_2$, the following inequality holds
  \[
    \sum_{i=1}^k L(y_i, z_i+\lambda z) \;\leq\; \sum_{i=1}^k L(y_i, z_i),\quad \mbox{with } z = -\displaystyle\frac{\sum_{i=1}^k \pderiv{z}L(y_i,z_i)}{\sum_{i=1}^k \pderiv[2]{z}L(y_i,z_i)}.
  \]
\end{lemma}

\begin{proof}
  Fix  $y_1,\dots,y_k\in K_1$ and $z_1,\dots,z_k\in K_2$ and define
  \[
    h(u) = \sum_{i=1}^k L(y_i, z_i+ u),\quad u\in \mathbb{R}.
  \]
  By Taylor's approximation, we have
  \[
    h(u)-h(0) = uh'(0) + \frac{u^2}{2}h''(\theta u),
  \]
  for some $\theta\in [0,1]$.
  Therefore, for $z=-h'(0)/h''(0)$, 
  \begin{align*}
    h(\lambda z)-h(0) &= -\lambda \frac{h'(0)^2}{h''(0)} + \frac{\lambda^2}{2} \frac{h'(0)^2}{ h''(0)^2}h''(\theta\lambda z)\\
    &= -\lambda \frac{h'(0)^2}{h''(0)} \Big(1-\frac{\lambda}{2}\frac{h''(\theta\lambda z)}{h''(0)}\Big),
  \end{align*}
  for some $\theta\in [0,1]$.
  We need to prove that uniformly in $k\geq 1$,  $y_1,\dots,y_k\in K_1$, $z_1,\dots,z_k\in K_2$  and $\theta\in [0,1]$,  the last expression is non-positive for  $\lambda$ small enough,  which amounts to showing that
  \begin{equation}\label{eq:proof-lemma-compactness}
    \lambda \leq 2 \frac{h''(0)}{h''(\theta\lambda z)} = 2\frac{\sum_{i=1}^k \pderiv[2]{z}L(y_i,z_i)}{\sum_{i=1}^k \pderiv[2]{z}L\Big(y_i,z_i-\theta\lambda\frac{\sum_i \pderiv{z}L(y_i,z_i)}{\sum_i \pderiv[2]{z}L(y_i,z_i)}\Big)}.
  \end{equation}
  Using the fact that $L(y,z)$ is $C^2$ in $z$ with $\pderiv[2]{z}L(y,z)> 0$, first let
  \[
    M := \max_{y\in K_1,z\in K_2}\abs[\bigg]{\frac{\pderiv{z}L(y,z)}{\pderiv[2]{z}L(y,z)}} < \infty
  \]
  and
  \[
    \lambda_0 := \min_{y\in K_1,z\in K_2,\theta\in[-1,1]} 2 \frac{\pderiv[2]{z}L(y,z)}{\pderiv[2]{z}L(y,z+\theta M)} > 0.
  \]
  Clearly, for all $k\geq 1$, $y_1,\dots,y_k\in K_1$, $z_1,\dots,z_k\in K_2$and $\lambda \leq \lambda_0\wedge 1$, the inequality~\eqref{eq:proof-lemma-compactness} is satisfied, concluding the proof.
\end{proof}

\begin{proof}[Proof of Proposition~\ref{prop:monotonicity}]
The proof relies on Lemma~\ref{lem:lambda-K}. Let $K$ be a compact set containing the interval $[F_0-\delta,F_0+\delta]$ with $F_0$ the initial value of the boosting procedure and $\delta>0$. Let $\lambda_0$ as in Lemma~\ref{lem:lambda-K} and $\lambda\in (0,\lambda_0]$. We prove that  $\|\hat F_m^\lambda-F_0\|_\infty \leq \delta$ implies $L_n(\hat F_{m+1}^\lambda)\leq L_n(\hat F_{m}^\lambda)$. Consider a leaf of the tree $T(\,\cdot\,;\hat F_m^\lambda,\zeta_{m+1})$ that contains the values $(x_i)_{i\in I}$ for some $I\subset [\![1,n]\!]$. Applying Lemma~\ref{lem:lambda-K} with $k=\mathrm{card}(I)$ and $z_i=\hat F_m^\lambda(x_i)\in K$ for $i\in I$, we obtain  
\[
\sum_{i\in I}L(y_i,\hat F_{m+1}^\lambda(x_i))\leq \sum_{i\in I}L(y_i,\hat F_{m}^\lambda(x_i)).
\]
In words, the updated model reduces the error within each leaf. Summing over all leaves, we get $L_n(\hat F_{m+1}^\lambda)\leq L_n(\hat F_{m}^\lambda)$ so that the updated model reduces the training error. This shows that the training error is non-increasing as long as $(\hat F_m^\lambda)_{m\geq 0}$ remains in the ball $\Delta=\{F:\|F-F_0\|_\infty\leq \delta\}$.

We next evaluate the time needed to exit $\Delta$.  Let  $C=L_n(F_0)$ be the initial value of the training error. By Lemma~\ref{lem:control-increment}, $L_n(\hat F_m^\lambda)\leq C$ implies
\[
\|T(\,\cdot\,;\hat F_m^\lambda,\zeta_{m+1}) \|_\infty\leq  M:=M(C)\quad \mbox{a.s.},
\]
whence we deduce
\[
\|\hat F_{m+1}^\lambda-\hat F_{m}^\lambda \|_\infty=\lambda \|T(\,\cdot\,;\hat F_m^\lambda,\zeta_{m+1}) \|_\infty\leq \lambda M\quad \mbox{a.s.}
\]
As long as $(\hat F_m^\lambda)_{m\geq 0}$ remains in $\Delta$, the training error is non-increasing and hence lower than its initial value  $L_n(F_0)=C$ and this implies $\|F_{m+1}^\lambda-F_{m}^\lambda \|_\infty\leq \lambda M$. To exit $\Delta$, the chain must travel the distance $\delta$ at speed less than $\lambda M$ so that at least $\delta/(\lambda M)$ iterations are needed. Renormalizing time, we get that $t\mapsto L_n(\hat F_{[t/\lambda]}^\lambda)$ does not exit $\Delta$  and is non-increasing on $[0,T]$ with $T=\delta/M$. Since $\delta>0$ is arbitrary and $M$ depends only on $C=L_n(F_0)$, we conclude the proof by taking $\delta=MT$ and choosing $\lambda_0$ accordingly. 
\end{proof}

\begin{proof}[Proof of Theorem~\ref{thm:cv-L2x} \textit{(i)}]~ \\
The proof of Proposition~\ref{prop:boosting-operator-Lipschitz} is easily adapted to prove that $\mathcal{T}^+$ and $\mathcal{T}^-$ are locally Lipschitz from $(\mathbb{B},\norm{\cdot}_\infty)$ into $(\mathbb{T},\norm{\cdot}_{\mathrm{TV}})$. Indeed, with exactly the same proof, Equation~\eqref{eq:gradient-tree-lipschtiz-2} can be modified into
\begin{align*}
\norm{\widetilde{T}^+(\,\cdot\,;F,\xi)-\widetilde{T}^+(\,\cdot\,;F',\xi)}_\mathrm{TV} &= \norm{\mu^+_{\widetilde{T}(\,\cdot\,;F,\xi)}-\mu^+_{\widetilde{T}(\,\cdot\,;F',\xi)}}_\mathrm{TV}\\
&\leq 4^d C\norm{F-F'}_\infty,
\end{align*}
and similarly for the negative part $\widetilde{T}^-(\,\cdot\,;F,\xi)-\widetilde{T}^-(\,\cdot\,;F',\xi)$. 

Because a functions  $F\in L^2_\mathbf{x}$ is well defined at $(x_i)_{1\leq i\leq n}$ and because the softmax gradient trees $\widetilde{T}^\pm(\,\cdot\,;F,\xi)$ depends on $F$ only through the values $(F(x_i))_{1\leq i\leq n}$, the operators $\mathcal{T}^+,\mathcal{T}^+:L^2_\mathbf{x}\rightarrow L^2_\mathbf{x} $ are well defined and also locally Lipschitz as
\begin{align*}
\norm{\mathcal{T}^\pm(F)-\mathcal{T}^\pm(F)}_{L^2_\mathbf{x}}&\leq \norm{\mathcal{T}^\pm(F)-\mathcal{T}^\pm(F)}_{\mathrm{TV}} \\
&\leq C \max_{1\leq i\leq n} |F(x_i)-F'(x_i)|\\
&\leq nC\norm{F-F'}_{L^2_\mathbf{x}}.
\end{align*}
The fact that $\mathcal{T}^\pm$ are locally Lipschitz operators implies the existence and uniqueness of \textit{local solutions} to the ODE~\eqref{eq:ODE-pm}. Proving that the local solution started at $(F_0^+,F_0^-)$ at time $t=0$ can be extended uniquely into a maximal solution on $[0,\infty)$ is done exactly as in the proof of Theorem~\ref{thm:EDO-T} and the existence and uniqueness of \textit{global solutions} follows.

When $(F_0^+,F_0^-)\in\mathbb{T}^+\times\mathbb{T}^+ $, the integral representation
\[
F^\pm(t)=F_0^\pm+\int_0^t \mathcal{T}^\pm(F(s))\rmd s,\quad t\geq 0, 
\] 
implies that $F^\pm(t)\in\mathbb{T}^+$ because $ \mathcal{T}^\pm(F(s))\in \mathbb{T}^+$ for all $s\in[0,t]$. Considering the difference between the positive and negative components, we get
\[
F(t)=F(0)+\int_0^t \mathcal{T}(F(s))\rmd s,\quad t\geq 0,
\]
proving that $(F(t))_{t\geq 0}$ is the solution to ODE \eqref{eq:ODE} started at $F_0$.
\end{proof}

\begin{proof}[Proof of Theorem~\ref{thm:cv-L2x} \textit{(ii)}] Our strategy is the following:
\begin{itemize}
\item[$\bullet$] Step 1 --- tightness. We consider a sequence $\lambda=\lambda(N)\to 0$ and prove that the sequence of processes $(\hat F_{[t/\lambda]}^{\lambda+},\hat F_{[t/\lambda]}^{\lambda-})_{t\geq 0}$ is tight in $\mathbb{D}([0,\infty),L^2_\mathbf{x}\times L^2_\mathbf{x})$. We use here the fact that the processes take their values in $\overline{\mathbb{T}}^+\subset L^2_\mathbf{x}$ which is a proper metric space, i.e.\ in which bounded sets are relatively compact.
\item[$\bullet$] Step 2 --- identification of the limit. We show that the weak convergence
\begin{equation}\label{eq:hyp-cv}
(\hat F_{[t/\lambda]}^{\lambda+},\hat F_{[t/\lambda]}^{\lambda-})_{t\geq 0}\stackrel{d}\longrightarrow (\tilde F_t^+,\tilde F_t^-)_{t\geq 0} \quad \mbox{in $\mathbb{D}([0,\infty),L^2_\mathbf{x}\times L^2_\mathbf{x})$}
\end{equation}
along a subsequence $\lambda(N)\to 0$ implies that the limit $(\tilde F_t^+,\tilde F_t^-)$ is the unique solution of the ODE in $L^2_\mathbf{x}\times L^2_\mathbf{x}$ considered in Theorem~\ref{thm:cv-L2x} \textit{(i)}. We use here the convergence of martingales associated with the Markov chain.
\end{itemize}
Steps 1 and 2 together imply the announced  convergence in $\mathbb{D}([0,\infty),L^2_\mathbf{x}\times L^2_\mathbf{x})$.

\bigskip
$\bullet$ Proof of Step 1 --- tightness.\\
 Fix $T>0$. By Proposition~\ref{prop:monotonicity}, for $\lambda\in (0,\lambda_0]$, where $\lambda_0$ may depend on $T$, the training error $t\mapsto L_n(\hat F_{[t/\lambda]}^\lambda)$ is a.s. non-increasing on $[0,T]$ and hence bounded from above by its initial value $C=L_n(F_0)$.
 Then Lemma~\ref{lem:control-increment} implies the following control of the increments: for $0\leq t_1\leq t_2\leq T$,
\begin{equation}\label{eq:increments-TV}
\left\|\hat F_{[t_2/\lambda]}^{\lambda\pm}-\hat F_{[t_1/\lambda]}^{\lambda\pm}\right\|_{\mathrm{TV}}\leq \lambda\left( [t_2/\lambda]-[t_1/\lambda]\right)4^dM.
\end{equation}
Indeed, there are $[t_2/\lambda]-[t_1/\lambda]$ increments, each with norm less than $\lambda 4^dM$. This implies the following global control in total variation norm: 
\begin{equation} \label{eq:totalVarBounded}
\|\hat F_{[t/\lambda]}^{\lambda\pm}\|_{\mathrm{TV}}\leq \|F_0\|_{\mathrm{TV}}+4^dMT,\quad t\in [0,T].
\end{equation}
We deduce that the Markov chain remains up to time $T$ in the set
\[
\mathcal{K}:=\{F\in\overline{\mathbb{T}}^+: \norm{F}_{L^2_\mathbf{x}}\leq \|F_0\|_{\mathrm{TV}}+4^dMT\},
\]
which is a compact subset of $\overline{\mathbb{T}}^+$ according to Proposition~\ref{prop:L2-T+}. This proves a compact containment condition: for all $\lambda\in (0,\lambda_0]$,
\[
\mathbb{P}((\hat F_{[t/\lambda]}^{\lambda+},\hat F_{[t/\lambda]}^{\lambda-})\in \mathcal{K}\times \mathcal{K},\ t\in[0,T])=1.
\]

Furthermore, Inequality~\eqref{eq:increments-TV} allows to control the modulus of continuity of the path $t\mapsto F_{[t/\lambda]}^{\lambda\pm}$. We define, for $\delta>0$,
\[
\omega_{\hat F^{\lambda\pm}}(\delta)=\sup_{\substack{0\leq t_1\leq t_2\leq T\\|t_1-t_2|<\delta}} \left\|\hat F_{[t_2/\lambda]}^{\lambda\pm}-\hat F_{[t_1/\lambda]}^{\lambda\pm}\right\|_{\mathrm{TV}}.
\]
Inequality~\eqref{eq:increments-TV} implies 
\[
\omega_{\hat F^{\lambda\pm}}(\delta)\leq (\delta+\lambda)4^d M \quad \mbox{a.s.}
\]
Taking $\delta < \eta/(2\cdot 4^dM)$ and $\lambda \leq \inf(\delta,\lambda_0)$, we have
\[
\mathbb{P}\left( \omega_{\hat F^{\lambda\pm}}(\delta)\geq \eta \right)=0.
\]
According to \citet[Chapter 3, Corollary 7.4]{EK86}, we obtain the tightness in $\mathbb{D}([0,\infty),L^2_\mathbf{x}\times L^2_\mathbf{x})$ of the sequence of processes $(\hat F_{[t/\lambda]}^{\lambda+},\hat F_{[t/\lambda]}^{\lambda-})_{t\geq 0}$ , where $\lambda=\lambda(N)\to 0$ is an arbitrary sequence.

$\bullet$ Proof of Step 2 --- identification of the limit.\\
 Assume the weak convergence~\eqref{eq:hyp-cv} along a subsequence $\lambda=\lambda(N)\to 0$. Consider the processes $(G_t^{\lambda \pm})_{t\geq 0}$ in $\mathbb{D}([0,\infty),L^2_\mathbf{x}\times L^2_\mathbf{x})$ defined by
\[
G_t^{\lambda\pm}=\hat F_{[t/\lambda]}^{\lambda\pm}-F_0-\int_0^t \mathcal{T}^\pm(\hat F^\lambda_{[s/\lambda]})\,\rmd s,\quad t\geq 0.
\]
We take here the (Bochner-)integral of an $L^2_\mathbf{x}$-valued function and refer to \cite{DU77} for integration in Banach spaces. We will only integrate locally bounded $ L^2_\mathbf{x}$-valued functions, which are Bochner-integrable simply because $L^2_\mathbf{x}$ is separable. In this context the integral behaves mostly as usual: in particular, for a continuous integrand $t\mapsto\phi(t)$, the function $t\mapsto\int_0^t\phi(s)\,\rmd s$ is $C^1$ with derivative $\phi$.

Consider the map $\Phi:\mathbb{D}([0,\infty)L^2_\mathbf{x}\times L^2_\mathbf{x})\to  \mathbb{D}([0,\infty),L^2_\mathbf{x}\times L^2_\mathbf{x})$ defined by $\Phi(F^\pm)=G^\pm$ with
\begin{equation} \label{eq:snd-variation-defG}
G^\pm_t=F_t^\pm-F_0^\pm-\int_0^t \mathcal{T}^\pm(F^+_s-F^-_s)\rmd s,\quad t\geq 0.
\end{equation}
Because $\Phi$ is continuous and $(\hat F_{[t/\lambda]}^{\lambda\pm})_{t\geq 0}\stackrel{d}\longrightarrow (\tilde F_t^\pm)_{t\geq 0}$, the continuous mapping theorem \citep[Theorem 2.7]{B99} implies the convergence
\begin{equation}\label{eq:cv1}
(G^{\lambda\pm}_t)_{t\geq 0}=\Phi((\hat F_{[t/\lambda]}^{\lambda\pm})_{t\geq 0})\stackrel{d}\longrightarrow \Phi((\tilde F_t^\pm)_{t\geq 0}).
\end{equation}
On the other hand, consider the discrete time processes
\begin{equation}\label{eq:def-martingale}
M_m^{\lambda\pm}=\hat F_{m}^{\lambda\pm}-F_0^\pm-\lambda \sum_{k=1}^{m} \mathcal{T}^\pm(\hat F^\lambda_{k})\rmd s,\quad m\geq 0.
\end{equation}
Its increments are given by
\begin{align*}
M_{m+1}^{\lambda \pm}- M_{m}^{\lambda \pm}&=\hat F_{m+1}^{\lambda\pm}-\hat F_{m}^{\lambda\pm}-\lambda\mathcal{T}^\pm(\hat F^\lambda_{m})\\
&=\lambda \widetilde{T}^\pm(\hat F_{m}^{\lambda},\zeta_{m+1})-\lambda \mathbb{E} \big[\widetilde{T}^\pm(\hat F_{m}^{\lambda},\zeta_{m+1}) \;\big | \; \hat F_{m}^{\lambda}\big],\quad m\geq 0.
\end{align*}
This shows that $(M_m^{\lambda\pm})_{m\geq 0}$ are martingales
with respect to the filtration $\mathcal{F}_m=\sigma(\zeta_k, 1\leq k\leq m)$, $m\geq 0$, with values in $L^2_\mathbf{x}$.
More precisely,  $(M_m^{\lambda\pm})_{m\geq 0}$ are square-integrable martingales with values in the separable Hilbert space $L^2_\mathbf{x}$.
For the existence of conditional expectation and properties of square-integrable martingales in Hilbert spaces, we refer to \cite{Met82}.
Here we will use only basic properties that are satisfied exactly as in the scalar case, namely the fact that
\begin{equation}\label{eq:sq-norm-martingale}
  \E\big[\norm{M_m^{\lambda\pm}}_{L^2_\mathbf{x}}^2\big] = \sum_{k=1}^m\E\big[\norm{M_{k}^{\lambda\pm}-M_{k-1}^{\lambda\pm}}_{L^2_\mathbf{x}}^2 \big],
\end{equation}
and Doob's inequality
\begin{equation}\label{eq:doob}
  \mathbb{E}\Big[\sup_{0\leq k\leq m}\norm{M_k^{\lambda\pm}}_{L^2_\mathbf{x}}^2\Big] \leq 4  \mathbb{E}\big[\norm{M_m^{\lambda\pm}}_{L^2_\mathbf{x}}^2\big].
\end{equation}
We prove below that, as $N\to\infty$,
\begin{equation}\label{eq:cv2}
\sup_{0\leq t\leq T}\left\|G_t^{\lambda\pm}-M_{[t/\lambda]}^{\lambda\pm} \right\|_{L^2_\mathbf{x}}\longrightarrow 0 \quad \text{a.s.},
\end{equation}
and
\begin{equation}\label{eq:cv3}
\sup_{0\leq t\leq T}\left\|M_{[t/\lambda]}^{\lambda\pm} \right\|_{L^2_\mathbf{x}}\longrightarrow 0 \quad \mbox{in probability}.
\end{equation}
Equation~\eqref{eq:cv2} is rather straightforward and consists in handling a boundary term. Indeed, writing the sum in Equation~\eqref{eq:def-martingale} as an integral, we get
\[
M_{[t/\lambda]}^{\lambda\pm}=\hat F_{[t/\lambda]}^{\lambda\pm}-F_0^\pm-\int_0^{\lambda[t/\lambda]}\mathcal{T}^\pm(\hat F^\lambda_{[s/\lambda]})\,\rmd s
\]
and
\[
G_t^{\lambda\pm}-M_{[t/\lambda]}^{\lambda\pm}=\int_{\lambda[t/\lambda]}^t \mathcal{T}^\pm(\hat F^\lambda_{[s/\lambda]})\,\rmd s.
\]
By the triangle inequality, we deduce 
\begin{equation} \label{eq:G-close-to-M}
\sup_{0\leq t\leq T}\left\|G_t^{\lambda\pm}-M_{[t/\lambda]}^{\lambda\pm} \right\|_{L^2_\mathbf{x}}\leq \lambda \sup_{0\leq s\leq T}\left\| \mathcal{T}^\pm(\hat F^\lambda_{[s/\lambda]})\right\|_{L^2_\mathbf{x}}.
\end{equation}
Since $\norm{\cdot}_{L^2_\mathbf{x}}\leq \norm{\cdot}_{\mathrm{TV}}$, by Proposition~\ref{prop:monotonicity} and the proof of tightness above, the supremum remains bounded for $\lambda\in (0,\lambda_0]$ whence we deduce the almost sure convergence to $0$ as $\lambda(N)\to 0$.

For the proof of \eqref{eq:cv3}, consider \eqref{eq:sq-norm-martingale} with $m = [T/\lambda]$ and $\lambda \in (0,\lambda_0]$.
We get
\begin{align*}
  \E\big[\norm{M_m^{\lambda\pm}}_{L^2_\mathbf{x}}^2\big] &= \sum_{k=1}^m\E\big[\norm{M_{k}^{\lambda\pm}-M_{k-1}^{\lambda\pm}}_{L^2_\mathbf{x}}^2 \big],\\
  &= \lambda^2\sum_{k=1}^m\E\big[\norm{\widetilde{T}^\pm(\,\cdot\,;\hat F_{m}^{\lambda},\zeta_{m+1})-\mathcal{T}^\pm(\hat F^\lambda_{m})}_{L^2_\mathbf{x}}^2 \big].
\end{align*}
Again because $\norm{\cdot}_{L^2_\mathbf{x}}\leq \norm{\cdot}_{\mathrm{TV}}$ and since $\widetilde{T}^\pm(\,\cdot\,;\hat F_{m}^{\lambda},\zeta_{m+1})$ and $\mathcal{T}^\pm(\hat F^\lambda_{m})$ are almost surely bounded in total variation by a deterministic constant, we have $\E\big[\norm{M_m^{\lambda\pm}}_{L^2_\mathbf{x}}^2\big]=O(\lambda) \to 0$.
An application of Doob's inequality~\eqref{eq:doob} is enough to conclude.

Putting \eqref{eq:cv2} and \eqref{eq:cv3} together, we get 
\[
(G^{\lambda\pm}_t)_{t\geq 0} \stackrel{d}\longrightarrow 0 \quad \mbox{in $\mathbb{D}([0,\infty),L^2_\mathbf{x}\times L^2_\mathbf{x})$},
\]
so that \eqref{eq:cv1} becomes
\[
  \tilde{F}_t^\pm = \tilde{F}_0^\pm + \int_0^t \mathcal{T}^\pm(\tilde{F}^+_s - \tilde{F}^-_s)\,\rmd s, \qquad t\geq 0.
\]
In other words, $(\tilde{F}^\pm)_{t\geq 0}$ is the solution to the ODE~\eqref{eq:ODE-pm}.
\end{proof}

\begin{proof}[Proof of Theorem~\ref{thm:cv-lambda-to-0}] We assume here that, for all $t\in[0,T]$, $\hat F_t^+$ and $\hat F_t^-$ are continuous on $[0,1]^p$ and that $t\mapsto \hat F_t^\pm$ are continuous for the uniform norm --- this is a consequence of Proposition \ref{prop:regularity}. Let us fix $T>0$, $\epsilon>0$ and $\delta>0$ such that 
\begin{equation}\label{eq:continuity-uniform}
  \sup_{\substack{t,s\in [0,T]\\\abs{t-s}\leq \delta}}\norm{\hat{F}^\pm_{t} - \hat{F}^\pm_{s}}_{\infty} \;\leq\; \epsilon.
\end{equation}
Let us also fix a subdivision $0=t_0<t_1<\dots<t_k=T$ with $t_i-t_{i-1}\leq \delta$.\\
Using Skorokhod's representation theorem in the separable space  $\mathbb{D}([0,\infty),L^2_\mathbf{x}\times L^2_\mathbf{x} )$, let us assume that $(\hat{F}^{\lambda\pm}_{[t/\lambda]})_{t\geq 0} \to (\hat{F}^{\pm}_t)_{t\geq 0}$ almost surely.
Since the map $t\mapsto \hat{F}^{\pm}_t$ is continuous, we therefore have the uniform convergence $\sup_{t\in[0,T]}\norm{\hat{F}^{\lambda\pm}_{[t/\lambda]}-\hat{F}^{\pm}_{t}}_{L^2_{\mathbf{x}}}\to 0$ almost surely.
Now since for all $i$, the functions $F^+_{t_i},F^-_{t_i}\in \mathbb{T}^+$ are continuous, we can apply Proposition~\ref{prop:L2-T+} \textit{(ii)} to obtain
\begin{equation}\label{eq:uniform-cv-subdiv}
\norm{\hat{F}^{\lambda\pm}_{[t_i/\lambda]} - \hat{F}^\pm_{t_i}}_{\infty} \;\tol\; 0, \qquad 0\leq i\leq k
\end{equation}
almost surely as $\lambda \to 0$. By construction $\hat{F}^{\lambda\pm}_{[t/\lambda]}(x)$ is non-decreasing in $t$ for all $x$, so that we have
\[
  \hat{F}^{\lambda\pm}_{[t_{i-1}/\lambda]} \leq \hat{F}^{\lambda\pm}_{[t/\lambda]}\leq \hat{F}^{\lambda\pm}_{[t_i/\lambda]}, \quad \mbox{for $t\in [t_{i-1},t_i]$}.
\]
By the construction of the subdivision $(t_i)_{0\leq i\leq k}$ based on Equation~\eqref{eq:continuity-uniform}, we deduce
\begin{align*}
  \norm{\hat{F}^{\lambda\pm}_{[t/\lambda]} - \hat{F}^{\pm}_{t}}_{\infty} &\leq \max\big(\norm{\hat{F}^{\lambda\pm}_{[t_{i-1}/\lambda]} - \hat{F}^{\pm}_{t}}_{\infty}, \; \norm{\hat{F}^{\lambda\pm}_{[t_{i}/\lambda]} - \hat{F}^{\pm}_{t}}_{\infty} \big)\\
  &\leq \max\big(\norm{\hat{F}^{\lambda\pm}_{[t_{i-1}/\lambda]} - \hat{F}^{\pm}_{t_{i-1}}}_{\infty}, \; \norm{\hat{F}^{\lambda\pm}_{[t_{i}/\lambda]} - \hat{F}^{\pm}_{t_i}}_{\infty} \big)+2\epsilon
\end{align*}
Therefore Equation \eqref{eq:uniform-cv-subdiv} implies $\limsup_{\lambda\to 0} \norm{\hat{F}^{\lambda\pm}_{[t/\lambda]} - \hat{F}^{\pm}_{t}}_{\infty} \leq \epsilon$.
Since $\epsilon>0$ is arbitrary, this concludes the proof of Theorem~\ref{thm:cv-lambda-to-0}.
\end{proof}

\begin{proof}[Proof of Theorem~\ref{thm:EDO}]
  Using the fact that $\mathcal{T}$ is a locally Lipschitz operator of $(\mathbb{B},\norm{\cdot}_{\infty})$, as mentioned in Remark~\ref{rk:lipschitz}, the proof of \textit{(i)} is the same as
for the proof of \Cref{thm:EDO-T}.

By \Cref{thm:cv-L2x}, we have for all $t\in [0,t^*)$, where $t^*$ is the maximal definition time of the ODE,
\[
	\deriv{t}\hat{F}_t = \mathcal{T}^+(\hat{F}^+_t-\hat{F}^-_t) - \mathcal{T}^-(\hat{F}^+_t-\hat{F}^-_t) = \mathcal{T}(\hat{F}_t)
\]
in $L^2_{\mathbf{x}}$.
By \Cref{prop:L2-T+} \textit{(ii)}, we know that the convergence
\[
	\lim_{s\to t} \frac{\hat{F}^\pm_t-\hat{F}^\pm_s}{t-s} = \mathcal{T}^\pm(\hat{F}_t)
\]
holds also in $\mathbb{B}$, which shows that $\deriv{t}\hat{F}_t = \mathcal{T}(\hat{F}_t)$ holds also in $\mathbb{B}$, concluding the proof of \textit{(ii)}.
\end{proof}

\subsubsection{Proofs related to Section~\ref{sec:prop-inf-gb}} \label{sec:proofs-prop-IGB}
\begin{proof}[Proof of Proposition~\ref{prop:properties}]~
\begin{enumerate}[label={\itshape(\roman*)}]
\item This monotonicity property is established in the proof of Theorem~\ref{thm:EDO-T}, see Equation~\eqref{eq:negative-derivative}.
\item Because $L$ is convex and differentiable with respect to the second variable, the  initialization~\eqref{eq:boosting-init} implies that $\bar r_t=0$ at time $t=0$. Similarly as in Equation~\eqref{eq:negative-derivative}, the derivative  with respect to $t\geq 0$ is given by 
\begin{align}
\frac{\rmd }{\rmd t}\bar r_t&= \frac{1}{n}\sum_{i=1}^n \mathcal{T}(F_t)(x_i)\, \frac{\partial^2 L}{\partial z^2}(y_i, \hat F_t(x_i))\nonumber \\
&= - \frac{1}{n}\mathbb{E}_\zeta\left[ \sum_{v\in\{0,1\}^d}\sum_{i=1}^n \frac{\partial L}{\partial z}(y_i, \hat{F}_t(x_i))\1_{A_v}(x_i) \right]\nonumber\\
&= -\bar r_t.\label{eq:equa-diff-rt}
\end{align}
In the second line, for  each term indexed by $v$, a factor $\sum_{i=1}^n \frac{\partial L^2}{\partial z^2}(y_i, \hat F_t(x_i))\1_{A_v}(x_i)$ is canceling out. The third line follows because the sum of all terms is equal to $\bar r_t$ and does not depend on the randomness $\zeta$. 

Equation~\eqref{eq:equa-diff-rt} implies that $\bar r_t=\bar r_0 e^{-t}$. This is equal to $0$ for initialization~\eqref{eq:boosting-init}. Interestingly, a different initialization would provide an exponentially fast decay of the residuals mean. \qedhere
\end{enumerate}
\end{proof}

The following Lemma will be used in the proof of Proposition~\ref{prop:critical-points}.
\begin{lemma} \label{lem:bound-p-region}
  Let $(x_i,r_i)_{1\leq i\leq n}$ and $(\beta,K,d)$ be fixed.
  Consider $J\subset [\![1,p]\!]$ with $\abs{J}\leq d$ and define 
  \[
    A(h) = \{x\in [0,1]^p\ :\ \forall j\in J, \,x^j < h^j\},\quad h\in [0,1]^J.
  \]
  Then for any $h_0\in[0,1]^p$ and $\rho>0$, there exists $q>0$ such that
  \begin{equation*} \label{eq:bound-p-region}
    P_{\beta,K}\big(A_{0}(\xi)=A(h) \mbox{ for some $h$ such that $\norm{h-h_0}_{\infty} \leq \rho$} \big) \geq q,
  \end{equation*}
  where $P_{\beta,K}$ is the splitting scheme distribution in Equation~\eqref{eq:RN} and $A_{0}(\xi)$ is the region containing $0$ in the partition associated to $\xi$. Furthermore, the constant $q$ depends only on $p,\beta, K, d, \rho$ and $R := \max_i\abs{r_i}$.
\end{lemma}

\begin{proof}[Proof of Lemma~\ref{lem:bound-p-region}]
We first consider the case  $\beta=0$, i.e. $P_{\beta,K}=P_0$, and  prove that  the event
  \[
E_{h_0,\rho}=\left\{    A_{0}(\xi)=A(h) \mbox{ for some $h$ such that $\norm{h-h_0}_{\infty} \leq \rho$} \right\}
  \]
satisfies  $P_0(E_{h_0,\rho}) \geq q$ with $q>0$  depending  only on $p$, $d$ and $\rho$. Possibly shifting $h_0$ by a small distance and reducing $\rho$, we may assume without loss of generality that $h_0 \in [\rho,1-\rho]^J$. For simplicity, assume first that $\abs{J}=d$, and let $ j_1 <  \dots < j_d$  be the increasing enumeration of $J$.
 We recall the notation $\xi=(j_v,u_v)_{v\in\mathscr{T}_{d-1}}$ and observe that the region $A_0(\xi)$ containing $0$ depends only of the values $(j_v,u_v)$ for the   $d$ leftmost nodes in $\mathscr{T}_{d-1}$ that we denote by  $v_1, v_2, \dots, v_{d}$. We then have
  \begin{align*}
  P_0(E_{h_0,\rho}) & \;\geq\; P_0\big(\forall k\in [\![1,d]\!],\;j_{v_k}=j_k\text{ and } u_{v_k}\in[h_0^{j_k}- \rho,h_0^{j_k}+ \rho]\big)\\
  &\;=\; \Big(\frac{2\rho}{p}\Big)^d \;=:\; q.
  \end{align*}
  This simple argument is adapted to the case when $\abs{J}=r<d$ by considering the sequence $j_1<\cdots<j_r=\ldots=j_d$ and noting that $\xi\in E_{h_0,\rho}$ if and only if   
\[
j_{v_k}=j_k \quad \mbox{and}\quad u_{v_k}\in [h_0^{j_k}- \rho,h_0^{j_k}+ \rho],\quad \mbox{for $k=1,\ldots,r-1$,}
\]
and
\[
j_{v_r}=\cdots=j_{v_d}=j_r,\quad \prod_{k=r}^r u_{v_d}\in [h_0^{j_r}- \rho,h_0^{j_r}+ \rho].
\]
Further details are left to the reader.

Next we consider the case $\beta>0$. Since $P_{\beta,K}$ is absolutely continuous with respect to $P_0$ according to Proposition~\ref{prop:RN-derivative}, we can write
\[
P_{\beta,K}(E_{h_0,\rho})=\int_{E_{h_0,\rho}} \frac{\rmd P_{\beta,K}}{\rmd P_0}(\xi)\; P_0(\rmd \xi).
\]
We prove that the Radon-Nykodym derivative~\eqref{eq:RN} is bounded from below by a positive constant.  By definition each of the
  $\Delta(s^k_v,u^k_v; A_v(\xi))$ is a nonnegative polynomial with degree $2$ in the $(r_i)$ and with coefficients bounded by $2$. 
  Therefore there is a uniform bound
  \[
    0\leq \Delta(s^k_v,u^k_v; A_v(\xi)) \leq C,
  \]
  where $C$ depends only on $R:= \max_i\abs{r_i}$.
  This implies that each factor in~\eqref{eq:RN} satisfies
  \[
    \frac{\exp(\beta \Score(s_v^1,u_v^1; A_v(\xi)))}{\sum_{k=1}^K \exp(\beta \Score(s_v^{k},u_v^{k}; A_v(\xi)))} \geq \frac{1}{1+(K-1)\mathrm{e}^{\beta C}}.
  \]
 We deduce that the Radon-Nykodym derivative~\eqref{eq:RN} is bounded from below by a positive constant depending only on $\beta,K,d$ and $R$. Since $P_0(E_{h_0,\rho})\geq q$, this yields a lower bound for $P_{\beta,K}(E_{h_0,\rho})$ and concludes the proof.
\end{proof}

\begin{proof}[Proof of Proposition~\ref{prop:critical-points}]
 It is trivially checked that, for $F\in \mathbb{B}$,  the nullity of the residuals $r_i = \pderiv{z}L(y_i,F(x_i))$ implies $\mathcal{T}(F)=0$. We prove here the converse implication under assumption \ref{ass:A5}. We consider  $F\in \mathbb{B}$  such that $(r_i)_{1\leq i\leq n}\neq 0$, and our goal is to show that $\mathcal{T}(F)\neq 0$.
  
  Recall from \eqref{eq:negative-derivative} that we can write
  \[
    \sum_{i=1}^n \mathcal{T}(F)(x_i)r_i = -\E_{\xi}\left[\sum_{v\in \{0,1\}^d}\frac{\big(\sum_{i=1}^nr_i\1_{A_v}(x_i)\big)^2}{\sum_{i=1}^n\pderiv[2]{z}L(y_i,F(x_i))\1_{A_v}(x_i)}\right].
  \]
 It is therefore sufficient to show that the sum under the expectation is not null with positive probability. All the terms being nonnegative, we show that the term  corresponding to the region $A_0$ is positive with positive probability.   Consider the set of indices $i\in [\![1,n]\!]$ such that $r_i\neq 0$, and among those, fix $i_0$ such that $x_{i_0}$ is minimal for the lexicographic order on $[0,1]^p$. Under assumption \ref{ass:A5}, there is $J$ such that the $(x_i^J)_i$ are pairwise  distinct, which ensures the positivity of $\rho=\min_{1\leq i_1<i_2\leq n} \norm{x_{i_1}^J-x_{i_2}^J}_{\infty}>0$. By definition of  $i_0$ and $\rho$, for $h\in (x_{i_0}^J,x_{i_0}^J+\rho)$  (addition by $\rho$ meant componentwise), the region $A(h)$ contains only one point $x_i$ corresponding to a non-null residual and this point is $x_{i_0}$.  Lemma~\ref{lem:bound-p-region} shows that 
  \[
    P_{\beta,K}\big(A_{0}(\xi)=A(h) \text{ for some }h\in (x_{i_0}^J,x_{i_0}^J+\rho) \big) > 0.
  \]
On this event, we have $(\sum_{i=1}^nr_i\1_{A_0}(x_i))^2 = r_{i_0}^2>0$ and the sum is positive. This implies that $\sum_{i=1}^n \mathcal{T}(F)(x_i)r_i\neq 0$ and hence $\mathcal{T}(F)\neq 0$.
\end{proof}

\begin{proof}[Proof of Proposition~\ref{prop:igb-asymptotic}]
According to Proposition~\ref{prop:properties}, the non-negative function $t\mapsto L_n(\hat F_t)$ is non-increasing. We denote by $\ell\geq 0$ its limit as $t\to +\infty$ and we prove that $\ell=0$. We proceed by contradiction and assume that $\ell>0$. Then  we will prove below that 
\[
\sup_{t\geq 0}\frac{\mathrm{d}}{\mathrm{d}t}L_n(\hat{F}_t) \leq -\eta \qquad \text{for some }\eta >0,
\]
whence $L_n(\hat F_t)\leq L_n(\hat F_0)-\eta t$, yielding a contradiction because the function is non-negative. 

We will use the following properties of $L_n$:
  \begin{enumerate}[(i)]
  \item for fixed  $(y_i)_{1\leq i\leq n}$ and $\ell_1>0$, there is a constant $R>0$  such that $L_n(F)\leq \ell_1$ implies 
  \[
  \max_{1\leq i\leq n}  \abs[\big]{\pderiv{z}L(y_i,F(x_i))}\leq R;
  \]
  \item for fixed  $(y_i)_{1\leq i\leq n}$ and $\ell_0>0$, there is a constant $\delta>0$ such that $L_n(F)\geq \ell_0$ implies 
  \[
  \max_{1\leq i\leq n}  \abs[\big]{\pderiv{z}L(y_i,F(x_i))}\geq \delta.
  \]
  \end{enumerate}
  The first point relies on the fact that $z\mapsto L(y,z)$ is convex and non-negative for all~$y$ (Assumption~\ref{ass:A}). The second point uses the fact that $z\mapsto L(y,z)$ has infimum~$0$ (Assumption~\ref{ass:A3}); then $L_n(F)\geq \ell_0$ implies $L(y_i,F(x_i))\geq \ell_0/n$ for some $i$, so that $F(x_i)$ is far from the minimizer of the convex function  $z\mapsto L(y_i,z)$ and the (non-increasing) derivative must be bounded away from~$0$.

We let $\ell_0=\ell>0$ and $\ell_1=L_n(\hat F_0)$. Note that $\ell_0\leq L_n(\hat F_t)\leq\ell_1$ for all $t\geq 0$ so that we can make use of the bound \textit{(i)} and \textit{(ii)} above. By~\eqref{eq:negative-derivative}, $\frac{\mathrm{d}}{\mathrm{d}t}L_n(\hat{F}_t)=\sum_i\mathcal{T}(\hat F_t)(x_i)r_i$ where
\[
\sum_{i=1}^n\mathcal{T}(\hat F_t)(x_i)r_i= -\E_{\xi}\left[\sum_{v\in \{0,1\}^d}\frac{\big(\sum_{i=1}^nr_i\1_{A_v}(x_i)\big)^2}{\sum_{i=1}^n\pderiv[2]{z}L(y_i,F(x_i))\1_{A_v}(x_i)}\right],
\]
where we use again, for conciseness, the notation $r_i = \pderiv{z}L(y_i,\hat F_t(x_i))$. We need to provide a lower bound for $\sum_{i=1}^n\mathcal{T}(F)(x_i)r_i$ when $\ell_0\leq L_n(F)\leq \ell_1$ and use a similar argument as  in the proof of Proposition~\ref{prop:critical-points}. Let $J$ be the set of coordinates appearing in \ref{ass:A5} and set $\rho=\frac{1}{2}\min_{i_1\neq i_2}\norm{x_{i_1}^J-x_{i_2}^J}_\infty>0$.   Consider the sets
  \[
    A(h) = \{x\in [0,1]^p\ :\ x^j < h^j \text{ for all }j\in J\},\quad h\in [0,1]^J.
  \]
 By \textit{(ii)} above, $L_n(F) \geq \ell_0$ implies that there is an index $i_0$ such that $\abs{r_{i_0}}\leq \delta$. Introducing the function   
 \[
    S(h) = \sum_{i=1}^nr_i\1_{A(h)}(x_i),
 \]
we have
  \begin{equation} \label{eq:ri-equality}
    r_{i_0} = \sum_{v\in \{-1,1\}^J} \sigma_v S(x_i+ \rho v)\quad \text{with} \quad \sigma_v = (-1)^{\sum_{k}\1_{v_k=-1}}.
  \end{equation}
This is due to the fact that $S$ is a tree function on $[0,1]^J$ with $\mu_S=\sum_{i=1}^n r_i\delta_{x_i^J}$ and that $x_{i_0}^J$ is the only atom in the hypercube $[  x_{i_0}^J-\rho, x_{i_0}^J+\rho]$ --- addition by $\rho$ meant componentwise. Then the bound $r_{i_0}\geq \delta$ together with Equation~\eqref{eq:ri-equality} imply that 
\[
\abs{S(x_{i_0}+\rho v)} \geq 2^{-\abs{J}}\abs{r_{i_0}} \geq 2^{-d}\delta
\]
for some $v\in \{-1,1\}^J$ . Furthermore, by construction we have $S(h) = S(x_i+\rho v)$ for all $h$ such that $\norm{h-x_{i_0}^J-\rho v}_\infty < \rho$.  Now we use Lemma~\ref{lem:bound-p-region} to bound from below the probability of the event 
\[
E=\{ A_0(\xi)=A(h)\;\mbox{for some $h\in[x_{i_0}^J-\rho,x_{i_0}^J+\rho]$}\}
\]
More precisely, we have $P_{\beta,K}(E)\geq q>0$, where $q$ depends only on $p,\beta,K,d,\rho$ and $R\geq \max_i \abs{r_i}$. By  point \textit{(i)} above, $R$ depends on $\ell_1$ but not on $F$ satisfying $L_n(F)\leq \ell_1$.   On the event $E$,  we have $\big(\sum_{i=1}^n r_i\1_{A(h)}\big)^2 \geq 4^{-d}\delta^2$ and, by \ref{ass:A4},
  \[
    \sum_{i=1}^n\pderiv[2]{z}L(y_i,F(x_i))\1_{A(h)} \leq n C,
  \]
  where $C$  depends only on $R$ and the $(y_i)$. 
  Therefore on the event $E$, we have
  \[
    \frac{\big(\sum_{i=1}^n r_i\1_{A(h)}\big)^2}{\sum_{i=1}^n\pderiv[2]{z}L(y_i,F(x_i))\1_{A(h)}} \geq \frac{4^{-d}\delta^2}{nC},
  \]
whence we deduce
  \[
    \E_{\xi}\left[\sum_{v\in \{0,1\}^d}\frac{\big(\sum_{i=1}^nr_i\1_{A_v}(x_i)\big)^2}{\sum_{i=1}^n\pderiv[2]{z}L(y_i,F(x_i))\1_{A_v}(x_i)}\right] \geq q\frac{4^{-d}\delta^2}{nC}  > 0.
  \]
  This proves that $L_n(\hat F_t)\to 0$.
  It follows that $L(y_i, \hat F_t(x_i))\to 0$ for all $i$.
  Since $0$ is the infimum of the convex function $z\mapsto L(y_i,z)$, the convergence $\pderiv{z}L(y_i, \hat F_t(x_i))\to 0$ holds for the derivative as well, ending the proof.
\end{proof}

Our proof of Proposition~\ref{prop:regularity} relies on the following lemma.
\begin{lemma}\label{lem:regularity}
For all $F\in\mathbb{B}$,  $\mathcal{T}(F)\in\mathbb{W}$ and the mapping  $\mathcal{T}:(\mathbb{B},\norm{\cdot}_\infty)\rightarrow (\mathbb{W},\norm{\cdot}_W)$ is locally Lipschitz. Furthermore, the same is true for  $\mathcal{T}^+$ and $\mathcal{T}^-$.
\end{lemma}
Similarly as in Remark~\ref{rk:lipschitz}, the inclusion $\mathbb{W}\subset \mathbb{B}$ with norm inequality $\norm{F}_{\infty}\leq \pi_0([0,1]^p)\norm{F}_{\mathbb{W}}$ for all $F\in\mathbb{W}$ implies that $\mathcal{T}$ can be seen as a locally Lipschitz operator on $(\mathbb{W},\norm{\cdot}_{\mathbb{W}})$.
\begin{proof}[Proof of Lemma~\ref{lem:regularity}]
We consider only the proof for $\mathcal{T}$ which is easily adapted for $\mathcal{T}^+$ and $\mathcal{T}^-$ by considering the positive and negative part of the gradient trees. It is enough to prove the following properties: 
\begin{enumerate}[(i)]
\item for all $F\in\mathbb{B}$, there is $M\geq 0$ such that 
\[
|\mu_{\mathcal{T}(F)}(A)|\leq M \pi_0(A),\quad \mbox{for all borel set $A\subset [0,1]^p$};
\] 
\item for all $R>0$, there exists $M\geq 0$ such that for all $F$ and $F'\in\mathbb{B}$ satisfying $\max(\norm{F}_\infty,\norm{F'}_\infty)\leq R$,
\[
|\mu_{\mathcal{T}(F)}(A)-\mu_{\mathcal{T}(F')}(A)|\leq M\norm{F-F'}_\infty \pi_0(A),\quad \mbox{for all borel set $A\subset [0,1]^p$}.
\] 
\end{enumerate}
Proof of \textit{(i)}. According to Proposition~\ref{prop:reg-tree-measure} and its proof, for each splitting scheme $\xi$, the measure associated to the gradient tree $\widetilde T(\,\cdot\,;F,\xi)$ is a point measure with atoms at the vertices in $[0,1)^p$ of the partition $(A_v(\xi))_{v\in\{0,1\}^d}$  and each atom has a mass with absolute value less than $\|\widetilde T(\,\cdot\,;F,\xi)\|_{\mathrm{TV}}$. By the definition of $\pi_\xi$, this implies
\[
\mu_{\widetilde T(\,\cdot\,;F,\xi)}\ll \pi_\xi\quad \mbox{with } \Big|\frac{\rmd \mu_{\widetilde T(\,\cdot\,;F,\xi)}}{\rmd \pi_\xi}(x)\Big| \leq \|\widetilde T(\,\cdot\,;F,\xi)\|_{\mathrm{TV}}\leq C,
\]
for some $C>0$ that depends only on $\norm{F}_{\infty}$ --- this last inequality is a consequence of Equations~\eqref{eq:gradient-tree-lipschtiz-2} with $F'=0$. We deduce  
\begin{align*}
|\mu_{\mathcal{T}(F)}(A)|&=\Big|\int \mu_{\widetilde T(\,\cdot\,;F,\xi)}(A) \, P_{\beta,K}(\rmd \xi) \Big|\\
&\leq  \int \Big(\int_A \Big|\frac{\rmd \mu_{\widetilde T(\,\cdot\,;F,\xi)}}{\rmd \pi_\xi}(x)\Big| \,\pi_\xi(\rmd x) \Big)\, \frac{\rmd P_{\beta,K}}{\rmd P_{0}}(\xi)\,P_0(\rmd \xi) \\
&\leq \int C\,\pi_\xi(A)\, \frac{\rmd P_{\beta,K}}{\rmd P_{0}}(\xi)\,P_0(\rmd \xi).
\end{align*}
Since  $\frac{\rmd P_{\beta,K}}{\rmd P_{0}}(\xi)\leq K^{2^d-1}$  according to Equation~\eqref{eq:RN}, we get 
\begin{align*}
|\mu_{\mathcal{T}(F)}(A)|
&\leq  C K^{2^d-1}\int \pi_\xi(A)\,P_0(\rmd \xi)\\
&=C K^{2^d-1}\pi_0(A).
\end{align*}
We deduce $\mu_{\mathcal{T}(F)}\ll \pi_0$ with  Radon-Nykodym derivative bounded by $C K^{2^d-1}$, so that $\mathcal{T}(F)\in\mathbb{W}$.

Proof of \textit{(ii)}. With similar notation as in the proof of Proposition~\ref{prop:boosting-operator-Lipschitz}, we compute
\begin{align*}
&|\mu_{\mathcal{T}(F)}(A)-\mu_{\mathcal{T}(F')}(A)|\\
\leq & \int \Big(\int_A \Big|\frac{\rmd \mu_{\widetilde T(\,\cdot\,;F,\xi)}}{\rmd \pi_\xi}(x)-\frac{\rmd \mu_{\widetilde T(\,\cdot\,;F',\xi)}}{\rmd \pi_\xi}(x)\Big| \pi_\xi(\rmd x)\Big)\, \frac{\rmd P_{\beta,K}}{\rmd P_{0}}(\xi)P_0(\rmd \xi) \\
&\quad + \int \Big(\int_A \Big|\frac{\rmd \mu_{\widetilde T(\,\cdot\,;F',\xi)}}{\rmd \pi_\xi}(x)\Big| \pi_\xi(\rmd x)\Big)\, \Big|\frac{\rmd P_{\beta,K}}{\rmd P_{0}}(\xi)P_0(\rmd \xi) -\frac{\rmd P_{\beta,K}'}{\rmd P_{0}}(\xi)\Big|P_0(\rmd \xi)\\
\leq & \int \Big(\int_A \norm{\widetilde T(\,\cdot\,;F,\xi)-\widetilde T(\,\cdot\,;F',\xi)}_{\mathrm{TV}}\,\pi_\xi(\rmd x)\Big)\, \frac{\rmd P_{\beta,K}}{\rmd P_{0}}(\xi)P_0(\rmd \xi) \\
&\quad + \int C\pi_\xi(A)\, \Big|\frac{\rmd P_{\beta,K}}{\rmd P_{0}}(\xi)P_0(\rmd \xi) -\frac{\rmd P_{\beta,K}'}{\rmd P_{0}}(\xi)\Big|P_0(\rmd \xi)\\
\leq &\; (4^d C K^{2^d-1}+CR) \|F-F'\|_\infty \pi_0(A),
\end{align*}
where the constant $C$ depends on $R$ but not on $F$ or $F'$.
The last inequality relies on  Equations~\eqref{eq:gradient-tree-lipschtiz-2} and~\eqref{eq:reg-tree-lipschitz-3}.
\end{proof}

\begin{proof}[Proof of Proposition~\ref{prop:regularity}]
We consider the proof for $(\hat F_t)_{t\geq 0}$ only since it can be easily adapted to $(\hat F_t^+)_{t\geq 0}$ and  $(\hat F_t^-)_{t\geq 0}$. By Theorem~\ref{thm:EDO-T}, infinitesimal gradient boosting is a solution of the ODE~\eqref{eq:ODE} in the space $(\mathbb{T},\norm{\cdot}_{\mathrm{TV}})$ so that
\[
\hat F_t=\hat F_0+\int_0^t \mathcal{T}(\hat F_s)\,\rmd s,\quad t\geq 0.
\]
By Lemma~\ref{lem:regularity}, $\mathcal{T}$ takes its value in $\mathbb{W}\subset \mathbb{T}$ and, since $\hat F_0\in\mathbb{W}$, it implies that $\hat F_t\in\mathbb{W}$ for all $t\geq 0$. Furthermore, $\mathcal{T}:\mathbb{W}\to\mathbb{W}$ is locally Lipschitz so that $(\hat F_t)_{t\geq 0}$ is the unique solution of the ODE~\eqref{eq:ODE} in  $(\mathbb{W},\norm{\cdot}_{\mathbb{W}})$. Therefore infinitesimal gradient boosting defines a smooth path in $\mathbb{W}$. 
\end{proof}

\phantomsection

\end{document}